\documentclass{article}
\usepackage[utf8]{inputenc} %
\usepackage[T1]{fontenc}    %
\usepackage{fullpage}

\usepackage[round]{natbib}
\renewcommand{\bibname}{References}
\renewcommand{\bibsection}{\subsubsection*{\bibname}}

\usepackage[margin = 1in]{geometry}

\usepackage{url}            %
\usepackage{booktabs}       %
\usepackage{amsfonts}       %
\usepackage{wrapfig}
\usepackage{caption}
\usepackage[mathscr]{euscript}
\usepackage{float}
\usepackage[nodisplayskipstretch]{setspace}

\usepackage{algorithm}
\usepackage[noend]{algpseudocode}

\newcounter{subroutine}
\makeatletter
\newenvironment{subroutine}[1][htb]{%
  \let\c@algorithm\c@subroutine
  \renewcommand{\ALG@name}{Subroutine}%
  \begin{algorithm}[#1]%
  }{\end{algorithm}
}
\makeatother

\usepackage{times}
\usepackage{adjustbox}
\usepackage{amssymb}
\usepackage{amsthm}
\usepackage{graphicx}
\usepackage{latexsym}
\usepackage{mathtools}
\usepackage{multirow}
\usepackage{paralist}
\usepackage{minitoc} %
\usepackage{xspace}
\usepackage{enumitem}
\usepackage{thm-restate}
\usepackage{soul}

\usepackage{tikz}
\usetikzlibrary{shapes}
\usetikzlibrary{positioning}
\usetikzlibrary{plotmarks}
\usetikzlibrary{patterns}
\usetikzlibrary{intersections,shapes.arrows}
\usetikzlibrary{fit,calc}

\definecolor{darkpink}{rgb}{0.91, 0.33, 0.5}

\definecolor{puorange}{rgb}{0.80,0.20,0}
\definecolor{bluegray}{rgb}{0.04,0,0.7}
\definecolor{greengray}{rgb}{0.05,0.50,0.15}
\definecolor{darkbrown}{rgb}{0.40,0.2,0.05}
\definecolor{darkcyan}{rgb}{0,0.4,1}
\definecolor{black}{rgb}{0,0,0}
\definecolor{grey}{rgb}{0.93,0.93,0.93}
\definecolor{royalazure}{rgb}{0.0, 0.22, 0.66}
\definecolor{blueviolet}{RGB}{138,43,226}

\usepackage[colorlinks,citecolor=blue,linkcolor=blue,urlcolor=blue,breaklinks]{hyperref}

\usepackage{cleveref}

\crefname{section}{Sec.}{Sections}
\crefname{appendix}{Appx.}{Appxs}
\crefname{theorem}{Thm.}{Thms.}
\crefname{lemma}{Lem.}{Lems.}
\crefname{lem}{Lem.}{Lems.}
\crefname{corollary}{Cor.}{Cors.}
\crefname{proposition}{Prop.}{Props.}
\crefname{prop}{Prop.}{Props.}
\crefname{assumption}{Asm.}{Asms.}
\crefname{asm}{Asm.}{Asms.}
\crefname{algorithm}{Alg.}{Algs.}
\Crefname{algorithm}{Algorithm}{Algorithms}
\Crefname{subroutine}{Subroutine}{Subroutines}
\crefname{figure}{Fig.}{Figs.}
\crefname{table}{Tab.}{Tabs.}

\newcommand\numberthis{\addtocounter{equation}{1}\tag{\theequation}}

\usepackage{color}
\definecolor{lightblue}{RGB}{0,160,200}
\definecolor{gray}{rgb}{0.4,0.4,0.4}
\usepackage{listings}
\newcommand{\red}[1]{{\color{purple}#1}}
\newcommand{\green}[1]{{\color{greengray}#1}}
\newcommand{\blue}[1]{{\color{royalazure}#1}}
\newcommand{\yellow}[1]{{\color{orange}#1}}
\newcommand{\purple}[1]{{\color{violet}#1}}

\let\originalparagraph\paragraph
\renewcommand{\paragraph}[1]{\originalparagraph{#1.}}

\usepackage{bm,bbm, mathtools}

\newcommand{\myparagraph}[1]{\vspace{-8pt}\paragraph{#1}\hspace{-0.8em}}  %

\newcommand \reals {\mathbb{R}}

\newcommand \inv {^{-1}} %
\newcommand \T {^{\top}}	%
\newcommand{\ones}{\mathbf{1}}
\renewcommand{\epsilon}{\varepsilon}
\newcommand \eps \epsilon
\newcommand{\pow}[1]{^{(#1)}}

\newcommand \expect {\mathbb{E}}

\DeclarePairedDelimiterX{\inp}[2]{\langle}{\rangle}{#1, #2} 
\newcommand{\norm}[1]{\left\Vert #1 \right\Vert} 
\newcommand{\normsq}[1]{\left\Vert #1 \right\Vert^2} 

\newcommand{\abs}[1]{\left\lvert #1 \right\rvert}

\newcommand \argmin {\operatorname*{arg\,min}} %
\newcommand \argmax {\operatorname*{arg\,max}} %
\newcommand \prox {\mathop{\mathrm{prox}}\nolimits}  %
\newcommand \grad {\nabla}

\newcommand \Tcal {\mathcal T}

\newcommand \Mcal {\mathcal M}

\newcommand \Pcal {\mathcal P}

\definecolor{cadmiumgreen}{rgb}{0.0, 0.42, 0.24}
\definecolor{harvardcrimson}{rgb}{0.79, 0.0, 0.09}

\newtheorem{theorem}{Theorem}

\newtheorem{proposition}[theorem]{Proposition}
\newtheorem{prop}[theorem]{Proposition}
\newtheorem{corollary}[theorem]{Corollary}
\newtheorem{lemma}[theorem]{Lemma}

\theoremstyle{definition}

\newcommand{\R}{\mathbb{R}}

\newcommand{\p}[1]{\ensuremath{\left(#1\right)}}
\newcommand{\br}[1]{\ensuremath{\left\{#1\right\}}}
\newcommand{\sbr}[1]{\ensuremath{\left[#1\right]}}
\newcommand{\sse}{\subseteq}
\renewcommand{\P}[2]{\ensuremath{{\mathbb P}_{#1}\left[#2\right]}}
\newcommand{\E}[2]{\ensuremath{{\mathbb E}_{#1}\left[#2\right]}}

\newcommand{\mc}[1]{\ensuremath{\mathcal{#1}}}

\newcommand{\msc}[1]{\ensuremath{\mathscr{#1}}}

\newcommand{\ip}[1]{\ensuremath{\left\langle #1 \right\rangle}}

\DeclareMathOperator*{\Unif}{Unif}

\newcommand{\argsort}{\operatorname{argsort}}
\newcommand{\convert}{\operatorname{Convert}}
\newcommand{\sol}{\operatorname{Sol}}
\newcommand{\pool}{\operatorname{Pool}}
\newcommand{\rank}{\operatorname{rank}}
\newcommand{\bubble}{\operatorname{Bubble}}

\newcommand{\lsvrg}{LSVRG\xspace}
\newcommand{\lsaga}{Prospect\xspace}
\newcommand{\saddlesaga}{SaddleSAGA\xspace}
\newcommand{\lsagaprox}{Prospect-Moreau\xspace}

\newcommand{\pav}{\operatorname{PAV}}

\newcommand{\yacht}{{\tt yacht}\xspace}
\newcommand{\energy}{{\tt energy}\xspace}
\newcommand{\concrete}{{\tt concrete}\xspace}
\newcommand{\kinnm}{{\tt kin8nm}\xspace}
\newcommand{\power}{{\tt power}\xspace}
\newcommand{\acsincome}{{\tt acsincome}\xspace}
\newcommand{\diabetes}{{\tt diabetes}\xspace}
\newcommand{\amazon}{{\tt amazon}\xspace}
\newcommand{\iwildcam}{{\tt iwildcam}\xspace}

\newcommand{\primobj}{F_\sigma}
\newcommand{\q}{q^{\text{opt}}}
\newcommand{\copt}{c^{\text{opt}}}
\renewcommand{\c}{c}
\newcommand{\f}{\ell}

\definecolor{problem}{cmyk}{0, 0.7808, 0.4429, 0.1412}

\newcommand{\risk}{\msc{R}}

\newcommand{\fchi}{\ensuremath{f_{\chi^2}}}
\newcommand{\fkl}{\ensuremath{f_{\text{KL}}}}

\newcommand{\Ex}{\mathbb{E}}

\newcommand{\env}{\operatorname{env}}

\doparttoc
\faketableofcontents

\title{Distributionally Robust Optimization \\with Bias and Variance Reduction}
\author{Ronak Mehta$^1$ \qquad Vincent Roulet$^2$ \qquad Krishna Pillutla$^{3}$ \qquad Zaid Harchaoui$^1$ \vspace{0.3cm} \\
{
\small $^1$University of Washington 
\qquad
$^2$Google DeepMind
\qquad
$^3$Google Research
}
}
\date{\vspace{-1em}}

\begin{document}

\maketitle

\begin{abstract}
We consider the distributionally robust optimization (DRO) problem with spectral risk-based uncertainty set and $f$-divergence penalty. This formulation includes common risk-sensitive learning objectives such as regularized condition value-at-risk (CVaR) and average top-$k$ loss. We present Prospect, a stochastic gradient-based algorithm that only requires tuning a single learning rate hyperparameter, and prove that it enjoys linear convergence for smooth regularized losses. This contrasts with previous algorithms that either require tuning multiple hyperparameters or potentially fail to converge due to biased gradient estimates or inadequate regularization. Empirically, we show that Prospect can converge 2-3$\times$ faster than baselines such as stochastic gradient and stochastic saddle-point methods on distribution shift and fairness benchmarks spanning tabular, vision, and language domains.

\end{abstract}

\section{Introduction}\label{sec:intro}
The ingredients of empirical risk minimization (ERM) are generally considered to be: a model with parameters $w \in \R^d$ (e.g. a neural network), a loss $\ell: \R^d \rightarrow \R^n$ where $\ell_i(w)$ is the error of $w$ on training example $i$, 
and an optimizer that returns a sequence $(w\pow{t})_{t \geq 1}$ converging to the solution of
\begin{align}
    \min_{w \in \R^d} \frac{1}{n} \sum_{i=1}^n \ell_i(w).
    \label{eq:erm}
\end{align}
The fourth ingredient--often taken for granted--is the choice of \emph{risk functional}, which aggregates individual training losses $\ell(w) \in \R^n$ into a univariate summary to be optimized. While~\eqref{eq:erm} uses the simple average (meant to estimate the expected loss under the training distribution), a deployed model often observes data from different distributions. Motivated by this phenomenon, we consider an objective that explicitly captures sensitivity to such distribution shifts:
\begin{align}
    \min_{w \in \R^d} \risk_\Pcal(\ell(w)) \quad \text{ where } \quad \risk_\Pcal(l) := \max_{q \in \Pcal}\br{\sum_{i=1}^n q_i l_i  - \nu D(q \Vert \ones_n/n)},
    \label{eq:risk-sensitive}
\end{align}
in which $\Pcal \sse \Delta^n = \br{\text{probability distributions on $n$ atoms}}$ is an \emph{uncertainty set} of distributions, $\nu \geq 0$ is a hyperparameter, and $D(q \Vert \ones_n/n)$ is a penalty that represents the divergence of $q$ from the original uniform weights $\ones_n / n = (1/n, \ldots, 1/n)$ (e.g.~the $\chi^2$ or Kullback Leibler divergence). The risk value $\risk_\Pcal(\ell(w))$~emulates a game in which nature pays a price of $\nu$ per unit $D(q \Vert \ones_n/n)$ to replace the expected loss under the given training distribution $\ones_n / n$ with the expected loss under $q$. The distribution $q$ is a reweighting of the training data that is chosen to be maximally unfavorable for the current model performance $\ell(w)$. Accordingly, we refer to $\nu$ as the \emph{shift cost}. 

Objectives of the form~\eqref{eq:risk-sensitive}, known as distributionally robust (DR) optimization problems, have seen a wave of recent interest in machine learning theory and practice. Examples range throughout diverse contexts such as reinforcement \citep{Liu2022Distributionally, kallus2022doubly, Liu2022DistributionallyRobustQ, Xu2023Group, Wang2023AFinite, Lotidis2023Wasserstein, kallus2022doubly, ren2022distributionally, clement2021first}, continual \citep{Wang2022Improving},  interactive \citep{Yang2023Distributionally, mu2022factored, inatsu2021active, sinha2020formula}, Bayesian \citep{Tay2022Efficient, Inatsu2022Bayesian}, and federated \citep{deng2020distributionally} learning, along with dimension reduction \citep{vu2022distributionally}, computer vision \citep{samuel2021distributional, sapkota2021distributionally}, and structured prediction \citep{Li2022Moment, fathony2018distributionally}. Historically used in quantitative finance, a popular such objective is the conditional value-at-risk (a.k.a.~superquantile/expected shortfall/average top-$k$ loss), or CVaR. In terms of methods, the CVaR has been used as a canonical DR objective \citep{Fan2017Learningwith, kawaguchi2020ordered, Rahimian2022Frameworks}, as well as in unsupervised \citep{Maurer2021Robust}, reinforcement \citep{singh2020improving}, and federated learning \citep{pillutla2023federated}. In applications, it has also been employed for robust language modeling \citep{liu2021just} and robotics \citep{Sharma2020Risk}. The CVaR falls into the broader category of \emph{spectral risk measures (SRMs)}, a class of DR objectives that includes the extremile and exponential spectral risk measure (ESRM) \citep{acerbi2002coherence, cotter2006extreme, daouia2019extremiles}. Motivated by 1) the success of the CVaR in numerous applications and 2) the importance of stochastic optimization in ML, \emph{the principal goal of this paper is to develop stochastic\footnote{We use \emph{stochastic} interchangeably with \emph{incremental}, meaning algorithms that make $O(1)$ calls per iteration to a fixed set of oracles $\br{(\ell_i, \grad \ell_i)}_{i=1}^n$, and {\bf not} \emph{streaming} algorithms that sample fresh data at each iteration.} optimization algorithms for spectral risk measures}. 

\myparagraph{Contributions}
In this paper we propose \lsaga, a stochastic algorithm for optimizing spectral risk measures with only one tunable hyperparameter: a constant learning rate. Theoretically, \lsaga converges linearly for \emph{any} positive shift cost on regularized convex losses. This contrasts with previous stochastic methods that may fail to converge due to bias \citep{Levy2020Large-Scale, kawaguchi2020ordered}, may not converge for small shift costs \citep{Mehta2022Stochastic}, or require the tuning of multiple hyperparameters \citep{palaniappan2016stochastic}.
Experimentally, \lsaga demonstrates equal or faster convergence than competitors on the training objective on nearly all problems and datasets considered, and exhibits higher stability with respect to external metrics on fairness and distribution shift benchmarks.

\myparagraph{Related Work}
Besides spectral risk measures (SRMs), other DR objectives can be recovered by changing the uncertainty set $\Pcal$. Examples include those based on $f$-divergences \citep{dommel2021Convex, Levy2020Large-Scale, Ben-Tal2013Robust}, the Wasserstein metric \citep{blanchet2019Robust, esfahani2018Data, Kuhn2019Wasserstein, bui2022a, abadeh2018wasserstein, nguyen2020distributionally, chen2019selecting, zhu2022distributionally, phan2023global},
maximum mean discrepancy \citep{kirschner2020distributionally, staib2019distributionally, nemmour2021approximate}, or more generally integral probability metrics \citep{husain2020distributional}. This work focuses on optimizing SRM objectives.

We compare against stochastic algorithms that either are single-hyperparameter ``out-of-the-box'' methods such as stochastic gradient descent and stochastic regularized dual averaging \citep{Xiao2009Dual}, or multi-hyperparameter methods that converge linearly on strongly convex SRM-based objectives, such as LSVRG \citep{Mehta2022Stochastic} and saddle-point SAGA \citep{palaniappan2016stochastic}.
Note that LSVRG may not converge for small shift costs.
Other methods may only achieve sublinear convergence rates, even in the strongly convex regime \citep{yu2022fast, ghosh2021efficient, carmon2022distributionally, li2019afirst, shen2022TowardsScalable, Hamedani2023AStochastic}. Non-convex settings have also been studied \citep{jin2021nonconvex, jiao2022distributed, Sagawa2020Distributionally, luo2020stochastic, ho2023adversarial}, as well as statistical aspects \citep{liu2022distributionallyrobust, blanchet2019multivariate, zeng2022generalization, Maurer2021Robust, Lee2020LearningBounds, Khim2020uniformConvergence, zhou2023sample, zhou2021finite, cranko2021generalized, LA2022AWasserstein, Pandey2019EstimationOS}. Our goal is to achieve unconditional linear convergence for smooth, strongly convex (regularized) losses with a single hyperparameter.

Objectives of the form~\eqref{eq:risk-sensitive} yield connections to other areas in modern machine learning.
They are a special case of \emph{subpopulation shift}, wherein the data-generating distribution 
is modeled as a mixture 
of subpopulations, 
and the distribution shift stems from changes in the mixture.
In our case, the subpopulations are point masses at the 
observed data points. In the context of \emph{algorithmic fairness}, the 
subpopulations
may represent data conditioned on 
some protected attribute (e.g.~race, gender, age range), and common notations of fairness such as \emph{demographic/statistical parity} \citep{Agarwal2018AReductions, Agarwal2019FairRegression} impose (informally) that model performance with respect to each 
subpopulation
should be roughly equal. As such, robustness to reweighting and algorithmic fairness are often aligned notions \citep{williamson2019fairness}, with recent research arguing that distributionally robust models are more fair \citep{hashimoto2018fairness, vu2022distributionally} and that fair models are more distributionally robust \citep{mukherjee2022domain}. In supervised learning, the data distribution is modeled as $P = P_{X, Y}$ for a feature-label pair $(X, Y)$ and related settings of \emph{covariate shift} (changes in $P_{X}$ and not $P_{Y|X}$) \citep{sugiyama2007direct} as well as \emph{label shift} (changes in $P_Y$ and not $P_{X|Y}$) \citep{lipton2018detecting} may also modeled with distributional robustness \citep{zhang2021coping}
as illustrated in \Cref{fig:subpopulations}.

\section{Minimizing Spectral Risk with Bias and Variance Reduction}\label{sec:problem}
This section describes the key technical challenges in constructing a stochastic optimizer for spectral risk measures and how \lsaga tackles them. In order to build a convergent stochastic algorithm, 
we will construct an estimate $v_i$ for the gradient of~\eqref{eq:risk-sensitive} based on a single data index $i$, such that $v_i \rightarrow \grad \risk_\Pcal(\ell(w))$ as the iteration counter approaches infinity. Precisely, we require that for $i \sim \Unif[n]$,
\begin{align}
    \expect\|\grad \risk_\Pcal(\ell(w)) - v_i\|_2^2 = \underbrace{\|\grad \risk_\Pcal(\ell(w)) - \expect[v_i]\|_2^2}_{\text{bias}} + \underbrace{\expect\|\expect[v_i] - v_i\|_2^2}_{\text{variance}}
    \label{eq:decomposition}
\end{align}
decreases to zero asymptotically. In the remainder of this section, we first identify concretely our target estimand (i.e.~$\grad \risk_\Pcal(\ell(w))$ for the spectral risk uncertainty set), construct an estimate, and then describe individual procedures to ensure that the bias and variance terms in~\eqref{eq:decomposition} vanish.

\begin{figure}[t]
    \centering
    \includegraphics[width=\linewidth]{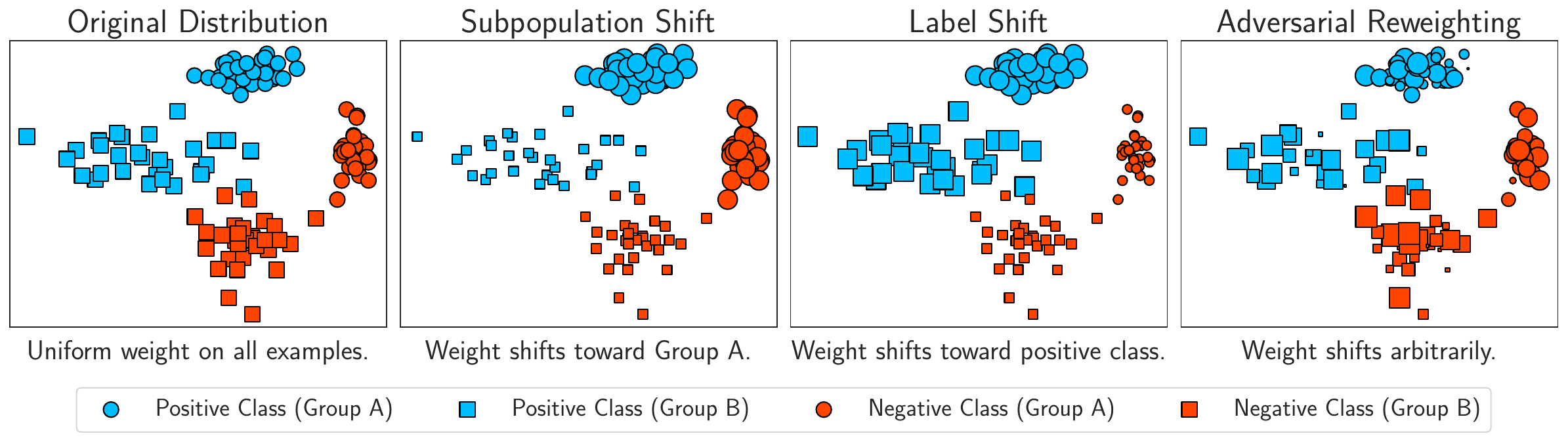}
    \caption{\textbf{Notions of Distribution Shift}. Illustration of various forms of distribution shift that are characterized by maintaining the same training data but changing the weight of each example.} 
    \label{fig:subpopulations}
\end{figure}

\myparagraph{The Gradient of a Spectral Risk Measure}
Each SRM is parameterized by a vector $\sigma = (\sigma_1, \ldots, \sigma_n)$ of non-negative weights satisfying $\sigma_1 \leq \cdots \leq \sigma_n$ and $\sum_{i=1}^n \sigma_i = 1$, called the \emph{spectrum}. The uncertainty set $\Pcal = \Pcal(\sigma)$ is the polytope
$
        \Pcal(\sigma) = \operatorname{ConvexHull}\br{ \text{permutations of } (\sigma_1, \ldots, \sigma_n) }.
$
See \Cref{fig:permutahedron}, \Cref{sec:a:objective}, for a visualization of $\Pcal(\sigma)$ for the CVaR \citep{Rockafellar2013Superquantiles,kawaguchi2020ordered, Laguel2022Superquantiles}, extremile~\citep{daouia2019extremiles}, and ESRM~\citep{cotter2006extreme}. The respective formulae for their spectra $\sigma$ and additional background on SRMs are also given in \Cref{sec:a:objective}. Define $\risk_{\sigma} := \risk_{\Pcal(\sigma)}$.
When $\nu > 0$ and the map $q \mapsto D(q \Vert \ones_n/n)$ is strongly convex over $\Pcal(\sigma)$, we have that (\Cref{lem:danskin}, \Cref{sec:a:objective}) $\risk_\sigma$ is differentiable with gradient given by
\begin{align}
    \grad \risk_\sigma(l) = \argmax_{q \in \Pcal(\sigma)} \br{q^\top l - \nu D(q \Vert \ones_n/n)} \in \R^n.
    \label{eq:gradient}
\end{align}
This means the full-batch gradient $w \mapsto \grad \risk_{\sigma}(\ell(w)) \in \R^d$ can be computed by solving the inner maximization to retrieve $l \mapsto \grad \risk_\sigma(l) \in \R^n$, calling the oracles to retrieve $w \mapsto \grad \ell(w) \in \R^{n\times d}$, and multiplying them by the chain rule. To solve for the maximizer, we prove by standard convex duality arguments (\Cref{prop:isotonic}, \Cref{sec:a:objective}) that when $D = D_f$ is an $f$-divergence, the maximum over $q$ can be expressed as a minimization problem that reduces to isotonic regression problem involving $f^*$, the convex conjugate of $f$.
Isotonic regression can be solved \emph{exactly} by the Pool Adjacent Violators algorithm~\citep{best2000minimizing}, which runs in
$O(n)$ time when the losses are sorted; see~\cref{sec:a:efficient}.

\myparagraph{Bias Reduction via Loss Estimation}
We now have a formula for the gradient and proceed to estimation. Denote by $q^l := \grad \risk_\sigma(l)$ from~\eqref{eq:gradient}, and observe that by the chain rule, $\grad \risk_\sigma(\ell(w)) = \sum_{i=1}^n q^{\ell(w)}_i \grad \ell(w)$. 
In words, we 
compute the ``most adversarial'' distribution $q^{\ell(w)} \in \Pcal(\sigma)$ for a given set of losses $\ell(w)$, and then take a convex combination of the gradients $\grad \ell_1(w), \ldots, \grad \ell_n(w)$
weighted by the probability mass function $q^{\ell(w)}$.
While the gradient is computable, however, accessing $\ell(w)$ and $\grad \ell(w)$ requires $n$ calls to the function/gradient oracles $\{\ell_i, \grad \ell_i\}_{i=1}^n$, which can be prohibitive. While using a plugin estimate with a mini-batch of size $m < n$ is a natural choice in ERM (making the first term in~\eqref{eq:decomposition} zero), this will be biased for our objective due to the maximization over $q$. For example, for $m = 1$, we have that $\risk_\sigma(\ell_i(w)) = \ell_i(w) \text{ and } \grad_w\risk_\sigma(\ell_i(w)) = \grad_w \ell_i(w)$, which are unbiased estimates of the ERM objective and gradient, respectively (not the SRM objective).
\begin{figure}[t]
    \centering
    \includegraphics[width=0.95\linewidth]{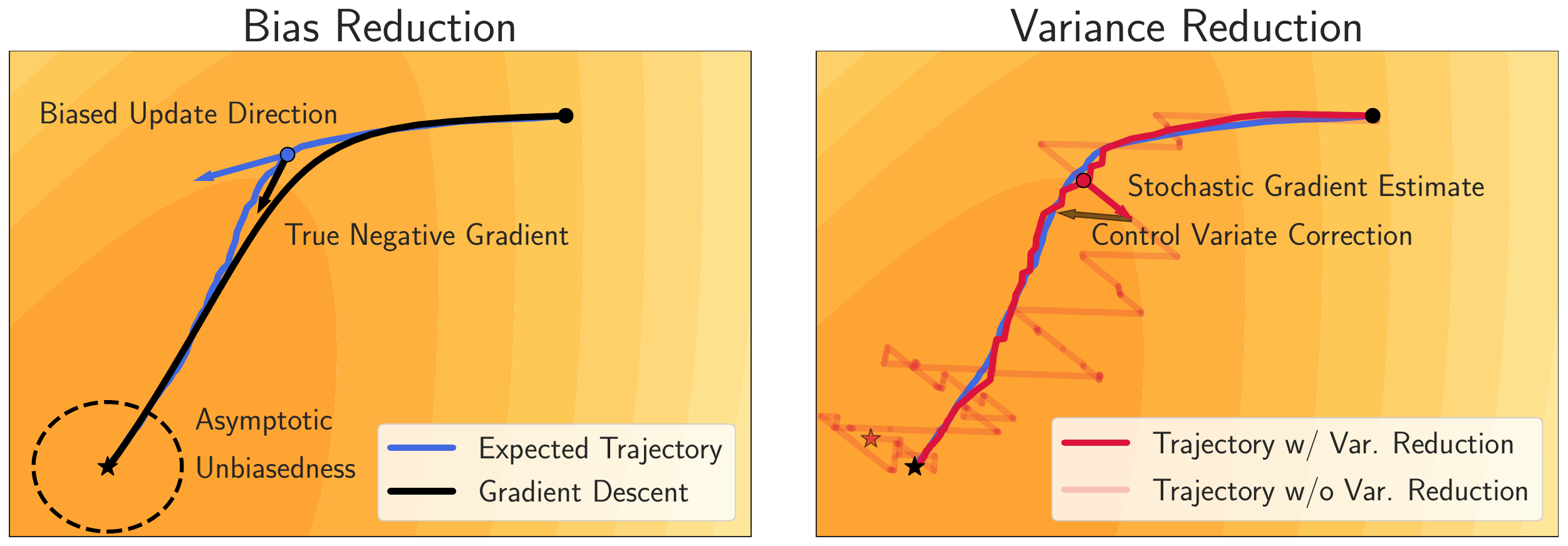}
    \caption{\small 
    \textbf{Prediction Error Reduction}.
    Optimization trajectories on the DR objective. Darker shade indicates lower objective value.
    {\bf Left:} Average trajectory of \lsaga over 20 seeds compared to full-batch gradient descent. {\bf Right:} Single trajectory of \lsaga with/without adding a control variate.}
    \label{fig:gametheory}
\end{figure}
However, note that if the optimal weights $q^{\ell(w)}$ were known, then for $i$ sampled from the uniform distribution on $[n]$, that $nq_i^{\ell(w)} \grad \ell_i(w)$ is an unbiased estimate for $\sum_{i=1}^n q^{\ell(w)}_i \grad \ell(w) = \grad \risk_\sigma(\ell(w))$.
While computing $q^{\ell(w)}$ again requires computing $\ell(w)$, the key ingredient of bias reduction in \lsaga is maintaining a table $l \in \R^n$ of losses such that $l \approx \ell(w)$ for the current iterate $w \in \R^d$, and using $q^l$ as a running estimate of $q^\ell(w)$. This is justified as when $q \mapsto D(q\Vert \ones_n)$ is strongly convex, we have by \Cref{prop:primobjgradient} the map $l \mapsto q^l$ is Lipschitz continuous in $l$ with respect to $\norm{\cdot}_2$.
Thus,
\begin{align*}
     l \approx \ell(w)  \implies q^l \approx q^{\ell(w)} \implies \E{i \sim \Unif[n]}{nq^l_i \grad \ell_i(w)} \approx \grad \risk_\sigma(\ell(w)).
\end{align*}
We prove in \Cref{sec:lsaga_algo} that $l - \ell(w) \rightarrow 0$ as the iterate counter goes to infinity for our particular choice of $l$, yielding an asymptotically unbiased gradient estimate as illustrated in \Cref{fig:gametheory} (left). 

\myparagraph{Variance Reduction via Control Variates}
The final ingredient of our stochastic gradient estimate is a variance reduction scheme. Given any estimator $\hat{\theta}$ of $\theta \in \R^d$, a \emph{control variate} is another estimator $\hat{\psi}$ over the same probability space with a known expectation $\Ex[\hat{\psi}] = \psi \in \R^d$, such that $\expect[(\hat{\theta} - \theta)^\top(\hat{\psi} - \psi)] > 0$. We can exploit this positive correlation to construct an estimator with strictly smaller variance that $\hat{\theta}$. 
Indeed, for $\alpha > 0$, we have that
\begin{align}
    \Ex\|\hat{\theta} - \alpha(\hat{\psi} - \psi) - \theta\|_2^2 &= \Ex\|\hat{\theta} - \theta\|_2^2 - 2\alpha \Ex[(\hat{\theta} - \theta)^\top(\hat{\psi} - \psi)] + o(\alpha) < \Ex\|\hat{\theta}{-}\theta\|_2^2,
    \label{eq:control}
\end{align}
demonstrating the improvement of $\hat{\theta} - \alpha(\hat{\psi} - \psi)$ over $\hat{\theta}$ for small $\alpha$. Note that $\hat{\theta}$ need not be unbiased. In our case, we have $\hat{\theta} = nq^l_i \grad \ell_i(w)$, where $l$ is the table of losses approximating $\ell(w)$. We also keep approximations $g \in \R^{n \times d}$ of $\grad \ell(w)$ and $\rho$ of $q^{\ell(w)}$, and define
\begin{align*}\textstyle
    \hat{\psi} = n\rho_i g_i \text{ and } \psi = \E{i \sim \Unif[n]}{n\rho_i g_i} = \textstyle\sum_{j=1}^n \rho_j g_j.
\end{align*}
In the unrealistic case in which $\hat{\psi} = \hat{\theta}$, the optimal multiplier is $\alpha = 1$, trivially achieving zero variance. Similar to $l$, we prove in \Cref{sec:lsaga_algo} that $g - \grad \ell(w) \rightarrow 0$ and $\rho - q^{\ell(w)} \rightarrow 0$, so we have in the notation of~\eqref{eq:control} that $\hat{\psi} - \hat{\theta} \rightarrow 0$. Thus, by using $\alpha = 1$, our final stochastic gradient estimate is
\begin{align}\textstyle
    \hat{\theta} - \alpha(\hat{\psi} - \psi) = nq^l_i \grad \ell_i(w) - n\rho_i g_i + \textstyle\sum_{j=1}^n \rho_j g_j,
    \label{eqg:grad_estimate}
\end{align}
which has asymptotically vanishing variance \emph{without} decreasing the learning rate, as illustrated in \Cref{fig:gametheory} (right). This variance reduction scheme generalizes (and is inspired by) the one employed in the SAGA optimizer \citep{defazio2014saga} for ERM, in which $\rho$ is set to $\ones_n/n$. 
Finally, while ignored in this section for ease of presentation, 
each $g_i$ will actually approximate the gradients of the regularized losses $\ell_i + \mu \norm{\cdot}_2^2$ for $\mu > 0$.

\section{The \lsaga Algorithm}\label{sec:lsaga_algo}
By combining the bias reduction and variance reduction schemes from the previous section, we build an algorithm that achieves overall \emph{prediction error reduction}. Thus, we now present the {\bf P}rediction Error-{\bf R}educed {\bf O}ptimizer for {\bf Spect}ral Risk Measures (\lsaga) algorithm to solve 
\begin{align}
    \min_{w \in \R^d} \sbr{\primobj(w) := \risk_\sigma(\ell(w)) + \frac{\mu}{2}\norm{w}_2^2},
    \label{eq:lsaga_obj}
\end{align}
where $\mu > 0$ is a regularization constant. The full algorithm is given in \Cref{algo:lsaga}. 
We offer in this section an intuitive explanation of the algorithm, discussion of computational complexity, theoretical convergence guarantees, and extensions to non-smooth settings.

\myparagraph{Instantiating Bias and Variance Reduction}
Consider a current iterate $w \in \R^d$. As mentioned in \Cref{sec:problem}, bias and variance reduction relies on the three approximations: the losses $l$ for $\ell(w) \in \R^n$, each gradient $g_i$ for $\grad \ell_i(w) + \mu w \in \R^{d}$, and the weights $\rho$ for $q^{\ell(w)} \in \Pcal$. Given initial point $w_0 \in \R^d$, we initialize $l = \ell(w_0)$, $g = \grad \ell(w_0)$, and $\rho = q^{\ell(w_0)}$ (including $\bar{g} := g^\top \rho$). 

At each iterate, we sample indices $i, j \sim \Unif[n]$ independently. The index $i$ is used to compute the stochastic gradient estimate~\eqref{eqg:grad_estimate}, yielding the update direction $v$ in \cref{alg:lsaga:update} at the cost of a call to a $(\ell_i, \grad \ell_i)$ oracle. Then, $l$ is updated by replacing $l_j$ with $\ell_j(w)$ costing another call to $(\ell_j, \grad \ell_j)$, and we reset $q$ (the variable that stores $q^l$). Next, we use $i$ again to make the replacements of $g_i$ with $\grad \ell_i(w) + \mu w$ and $\rho_i$ with $q_i = q_i^l$. In summary, each approximation is updated every iteration by changing one component based on the current iterate $w$. The indices $i, j$ are ``decoupled'' for theoretical convenience, but in practice using only $i$ works similarly, which we use in \Cref{sec:experiments}.

\begin{algorithm}[t]\caption{\lsaga \label{algo:lsaga}}
\begin{algorithmic}[1]
    \Statex {\bf Inputs:}
    Initial $w_0$, spectrum $\sigma$, number of iterations $T$, regularization $\mu > 0$, shift cost $\nu > 0$. 
    \Statex {\bf Hyperparameter:} Stepsize $\eta > 0$.
    \State \textbf{Initialize} $l \gets \ell(w_0)$ and $g_i \gets \grad \ell_i(w_0) + \mu w_0$ for $i=1,\ldots, n$.
    \State Set $q \gets \argmax_{\bar{q} \in \mc{P}(\sigma)} \bar{q}^\top l - \nu D(q \Vert \ones_n/n)$ and $\rho \gets q$.
    \State Set $\bar g \gets \sum_{i=1}^n \rho_i g_i \in \R^d$.
    \State Set $w \gets w_0$.
    \For{$T$ iterations}
        \State Sample $i, j \sim \Unif[n]$ independently.
        \State $v \gets nq_{i}(\grad \ell_{i}(w) + \mu w) - n\rho_{i} g_{i_t} + \bar g$. \Comment{\textbf{Iterate Update}}
        \label{alg:lsaga:update}
        \State $w \gets w - \eta v$.
        \label{alg:lsaga:param}
        \State $l_{j} \gets \ell_{j}(w)$. \Comment{\textbf{Bias Reducer Update}}
        \State $q \gets \argmax_{\bar{q} \in \mc{P}(\sigma)} \bar{q}^\top l - \nu D(\bar{q} \Vert \ones_n/n)$.
        \label{alg:lsaga:dual}
        \State $\bar g \gets \bar g  - \rho_{i} g_{i} + q_{i} \p{\grad \ell_{i}(w) + \mu w}$. \Comment{\textbf{Variance Reducer Update}}
        \State $g_{i} \gets \grad \ell_{i}(w) + \mu w$.
        \State $\rho_{i} \gets q_{i}$. 
        \label{alg:lsaga:cv}
    \EndFor
    \Statex {\bf Output:} Final point $w$.
\end{algorithmic}
\end{algorithm}

\myparagraph{Computational Aspects}
The weight update in \Cref{alg:lsaga:dual} is solved exactly by (i) sorting the vector of losses in $O(n \log n)$, (ii) plugging the sorted loss table $l$ into the {Pool Adjacent Violators (PAV)} algorithm running in $O(n)$ time, as mentioned in \Cref{sec:problem}. Because only one element of $l$ changes every iterate, we may simply bubble sort $l$ starting from the index that was changed. While in the worst case, this cost is $O(n)$, it is exactly $O(s)$ where $s$ is the number of swaps needed to resort $l$. We find in experiments that the sorted order of $l$ stabilizes quickly.
The storage of the gradient table $g$ requires $O(nd)$ space in general, but it can be reduced to $O(n)$ for generalized linear models and nonlinear additive models. For losses of the form $\ell_i(w) = h(x_i^\top w, y_i)$, for a differentiable loss $h$ and scalar output $y_i$,
we have $\grad \ell_i(w) = x_i \, h'(x_i\T w, y_i)$. We only need to store the scalar $h'(x_i\T w, y_i)$, so \lsaga requires $O(n + d)$ memory.
In terms of computational complexity, Lines~\ref{alg:lsaga:param} and~\ref{alg:lsaga:cv} require $O(d)$ operations and Line~\ref{alg:lsaga:dual} requires at most $O(n)$ operations, so that in total the iteration complexity is $O(n+d)$. In comparison, a full batch gradient descent requires $O(nd)$ operations so \lsaga decouples efficiently the cost of computing the losses, gradients, and weights. 

\myparagraph{Convergence Analysis}
We assume throughout that each $\ell_i$ is convex, $G$-Lipschitz, and $L$-smooth. We also assume that the $D = D_f$ is an $f$-divergence with the generator $f$ being $\alpha_n$-strongly convex on the interval $[0, n]$ (e.g.~$\alpha_n = 2n$ for the $\chi^2$-divergence and $\alpha_n = 1$ for the KL-divergence).

The convergence guarantees depend on the condition numbers $\kappa = 1 + L / \mu$ of the individual regularized losses, as well as a measure $\kappa_\sigma = n\sigma_n$ of the skewness of the spectrum. Note that both $\kappa$ and $\kappa_\sigma$ are necessarily larger than or equal to one. Define $w^\star := \argmin_{w} \primobj(w)$, which exists and is unique due to the strong convexity of $\primobj$. The proof is given in \Cref{sec:a:main_result}.
\begin{restatable}{theorem}{lsagamain}
\label{thm:lsaga:main}
\lsaga with a small enough step size is guaranteed to converge linearly for all $\nu > 0$.
If, in addition, the shift cost is 
$\nu \ge \Omega(G^2 / \mu \alpha_n)$, then 
the sequence of iterates $(w\pow{t})_{t\geq 1}$ generated by \lsaga and 
learning rate $\eta = (12 \mu (1+\kappa) \kappa_\sigma)^{-1} $ converges linearly at a rate
$\tau = 2\max\left\{n, 24 \kappa_\sigma(\kappa + 1)\right\}$, i.e., 
\[
\expect\|w\pow{t} - w^\star\|_2^2 \le (1 + \sigma_n^{-1} + \sigma_n^{-2}) \exp(-t/\tau) \|w\pow{0} - w^\star\|_2^2 \,.
\]
\end{restatable}
The number of iterations $t$ required by \lsaga to achieve $\expect\|w\pow{t} - w^\star\|_2^2 \le \epsilon$ (provided that $\nu$ is large enough) is
$t = O((n + \kappa \kappa_\sigma) \ln(1/\eps))$. This exactly matches the rate of the \lsvrg \citep{Mehta2022Stochastic}, the only primal stochastic optimizer that converges linearly for spectral risk measures. However, unlike \lsvrg, \lsaga is guaranteed to converge linearly for any shift cost and has a single hyperparameter, the stepsize $\eta$. Similarly, compared to primal-dual stochastic saddle-point methods, our algorithm requires only one learning rate, streamlining its implementation. 

\myparagraph{\lsaga Variants for Non-Smooth Objectives}
We may wonder about the convergence behavior of \lsaga when either the shift cost $\nu = 0$, or the underlying losses $\ell_i$ are non-smooth. While the smoothness of the objective is then lost, \lsaga still converges to the minimizer $w_0^\star$ as we prove below. 
The first setting is relevant as historically, SRMs such as the conditional value-at-risk have been employed as coherent risk measures for loss distributions \citep{acerbi2002coherence} in the form of an $L$-estimator $\sum_{i=1}^n \sigma_i l_{(i)}$ (as seen in \Cref{sec:problem}). If these losses are separated at the optimum, however, we may achieve linear convergence with \lsaga even with $\nu = 0$. 
This behavior can be explained as ``hidden smoothness'' in the objective \eqref{eq:lsaga_obj}; the objective is indeed differentiable at points satisfying $\ell_{(1)}(w) < \cdots < \ell_{(n)}(w)$, where $\ell_{(i)}(w)$ denotes the $i$-th smallest loss at $w$. Assume convex losses $\ell_1, \ldots, \ell_n$ and $\mu > 0$.
\begin{proposition}\label{prop:no_shift}
Let $w^\star_\nu$ be the unique minimizer of \eqref{eq:lsaga_obj} with shift cost $\nu \geq 0$. Assume that the values $\ell_1(w^\star_0), \ldots, \ell_n(w^\star_0)$ are all distinct. Then, there exists a constant $\nu_0 > 0$ such that $w^\star_0 = w^\star_\nu$ exactly for all $\nu \leq \nu_0$. Thus, running \lsaga with $\nu \in (0, \nu_0]$ converges to the minimizer $w_0^\star$.
\end{proposition}
In particular, $\nu_0$ is chosen so that $\nu_0 \p{\sigma_{i+1} - \sigma_i} < \ell_{(i+1)}(w^\star_0) - \ell_{(i)}(w^\star_0)$ for each $i$, or as the multiplicative factor that relates gaps in the spectrum to the gaps in the loss at optimality (see \Cref{sec:a:objective}).

For the setting in which any $\ell_i$ may be non-smooth, we generalize \lsaga by applying it to the Moreau envelope of each loss $\ell_i$ and their gradients~\citep{bauschke2011convex, rockafellar1976monotone}, allowing for accelerated performance and non-smooth losses (such as those containing an $\ell_1$ penalty).
Specifically, we consider oracles $\nabla \env(\ell_i)(w)$ where $\env(\ell_i) = \inf_{v \in \R^d} \ell_i(v) + \|w-v\|_2^2$; the updates can be expressed in terms of the proximal operators of the losses~\citep{bauschke2011convex}. Such an approach has been considered for ERM by \citep{defazio2016simple} to accelerate the SAGA algorithm. The oracles can be accessed either in closed form or by efficient subroutines in common machine learning settings~\citep{defazio2016simple,frerix2018proximal,roulet2022differentiable}, we can adapt \cref{algo:lsaga} to leverage such oracles.
The resulting algorithm enjoys a linear convergence guarantee similar to \Cref{thm:lsaga:main} with a more liberal condition on the shift cost $\nu$ while providing competitive performance in practice. 
We refer to \Cref{sec:a:prox_saga} for details.

\section{Experiments}\label{sec:experiments}
We compare \lsaga against competitors in a variety of learning tasks. While we focus attention on its performance as an optimizer with respect to its training objective, we also highlight metrics of interest on the test set in fairness and distribution shift benchmarks. The code containing the algorithm implementation and data preparation is made publicly available online: \href{https://github.com/ronakdm/prospect}{https://github.com/ronakdm/prospect}.

\myparagraph{Setting, Baselines, Evaluation} 
We consider supervised learning tasks where data points $z_i = (x_i, y_i)$ are input-label pairs. Losses are of the form $\f_i(w) := h(y_i, w^\top \phi(x_i))$, with $\phi$ a fixed feature embedding,
and $h$ measuring prediction error.
The spectrums considered are:
$0.5$-CVaR, $2$-extremile,
and $1$-ESRM. 

We compare against four baselines: minibatch stochastic gradient descent (SGD), stochastic regularized dual averaging (SRDA) \citep{Xiao2009Dual}, Saddle-SAGA \citep{palaniappan2016stochastic}, and LSVRG \citep{Mehta2022Stochastic}.  For SGD/SRDA, we use a batch size of 64 and for LSVRG we use an epoch length of $n$. For Saddle-SAGA, we show that allowing different primal and dual learning rates provides theoretically and experimentally improvement (\Cref{sec:a:saddle_saga}) and use this improved heuristic (setting the dual stepsize $10n$ times smaller than the primal one).
We plot
\begin{align}
    \operatorname{Suboptimality}(w) = (\primobj(w) - \primobj(w^\star)) \, /\, (\primobj(w\pow{0}) - \primobj(w^\star)) \,,
    \label{eqn:subopt}
\end{align}
where $w^\star$ is approximated by running LBFGS~\citep{nocedal1999numerical} on the objective until convergence. The $x$-axis displays the number of calls to any first-order oracle $w \mapsto (\f_i(w), \grad \f_i(w))$ divided by $n$, i.e. the number of passes through the training set. We fix the shift cost $\nu = 1$ and regularization parameter $\mu = 1/n$.
Further details of the setup such as hyperparameter tuning, and additional results are given in \Cref{sec:a:experiments,sec:a:additional} respectively.

\subsection{Tabular Least-Squares Regression}

\begin{figure*}[t]
    \centering
    \includegraphics[width=\linewidth]{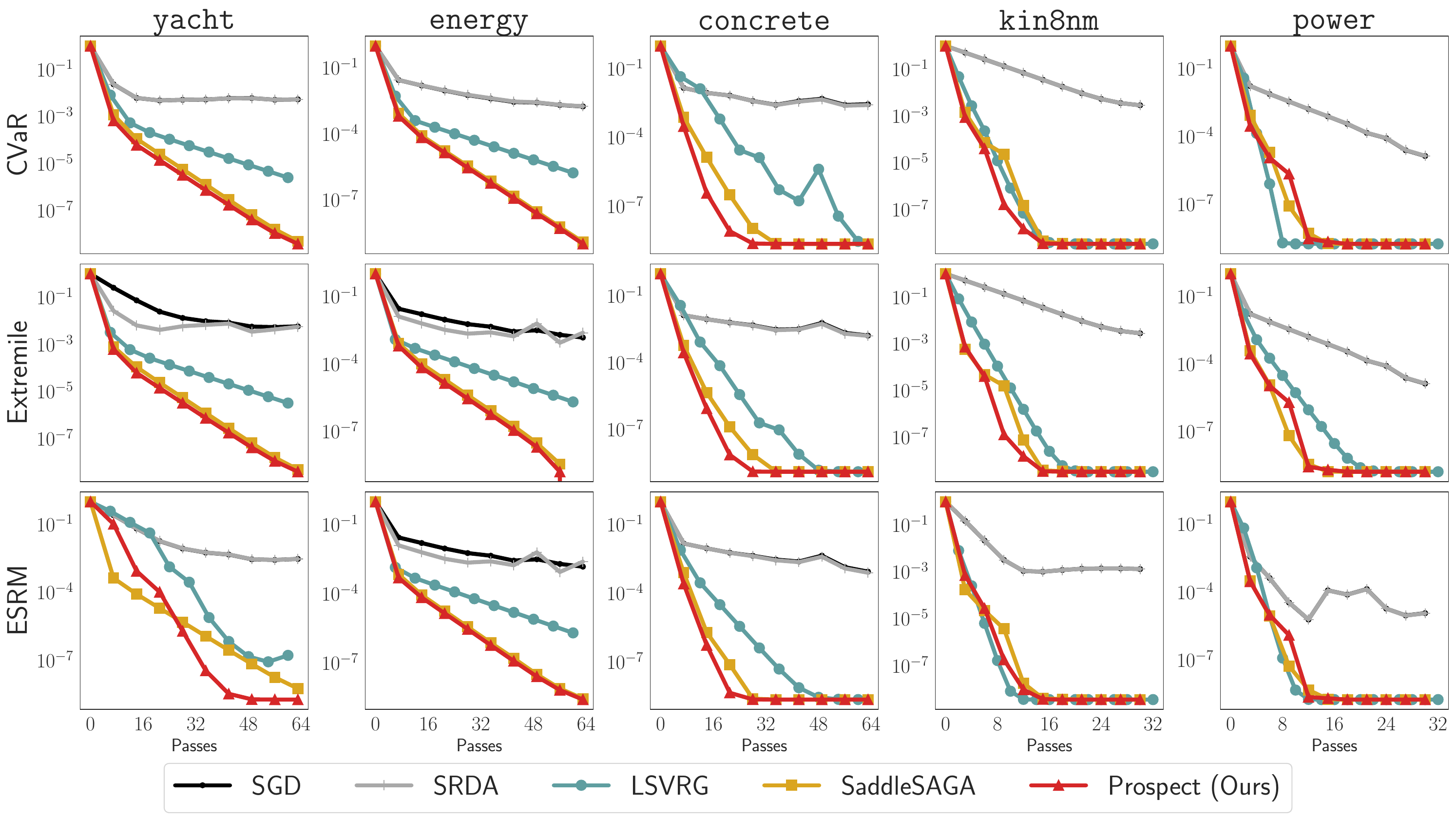}
    \caption{{\bf Regression benchmarks}. The $y$-axis measures the suboptimality as given by \eqref{eqn:subopt}, while the $x$-axis measures the number of calls to the function value/gradient oracle divided by $n$. Rows indicate different spectral risk objectives and columns indicate datasets.}
    \label{fig:uci}
\end{figure*}

We consider five regression benchmarks under square loss. The datasets are 
\texttt{yacht} 
($n=244$) 
\citep{Tsanas2012AccurateQE}, 
\texttt{energy} 
($n=614$) 
\citep{Segota2020Artificial},  
\texttt{concrete} 
($n=824$) 
\citep{Yeh2006Analysis}, 
\texttt{kin8nm} 
($n=6553$) 
\citep{Akujuobi2017Delve}, 
and \texttt{power} 
($n=7654$) 
\citep{Tufekci2014Prediction}.
The training curves are shown in \Cref{fig:uci}.

\myparagraph{Results}
Across datasets and objectives, we find that \lsaga exhibits linear convergence at a rate no worse than \saddlesaga and \lsvrg but often much better.
For example, \lsaga converges to precision $10^{-8}$ for the CVaR on {\tt concrete} and the extremile on {\tt power} within half the number of passes that LSVRG takes for the same suboptimality. Similarly, for the ESRM on {\tt yacht}, \saddlesaga requires 64 epochs to reach the same precision as \lsaga at 40 epochs. The direct stochastic methods, SGD/SRDA, are biased and fail to converge for any learning rate.

\subsection{Fair Classification and Regression}

Inspired by \cite{williamson2019fairness}, 
we explore the relationship between robust learning and group fairness on 2 common tabular benchmarks.
{\bf Diabetes 130-Hospitals} ({\tt diabetes}) is a binary classification task of predicting readmission for diabetes patients based on 10 years worth of clinical data from 130 US hospitals \citep{Rizvi2014Impact}.
{\bf Adult Census} ({\tt acsincome}) is a regression task of predicting income of US adults given features compiled from the American Community Survey \citep{Ding2021Retiring}. 

\myparagraph{Evaluation}
We evaluate fairness with the \emph{statistical parity score}, which compares predictive distributions of a model given different values of a particular protected attribute \citep{Agarwal2018AReductions, Agarwal2019FairRegression}. Letting $Z = (X, Y, A)$ denote a random (input, label, metadata attribute) triplet, a model $g$ is said to satisfy statistical parity (SP) if the conditional distribution of $g(X)$ over predictions given $A = a$ is equal for any value $a$. Intuitively, statistical parity scores measure the maximum deviation between these distributions for any over $a$, so values close to zero indicate SP-fairness.
In {\tt diabetes}, we use gender as the protected attribute $A$, whereas in {\tt acsincome} we use race as the protected attribute. Note that the protected attributes are not supplied to the models. The results are given in \Cref{fig:fairness}. 

\myparagraph{Results}
Firstly, we note that \lsaga converges rapidly on both datasets while  LSVRG fails to converge on {\tt diabetes} and \saddlesaga fails to converge on {\tt acsincome}. 
Secondly, LSVRG does not stabilize with respect to classification SP, showing a mean/std SP score of $1.38  \pm 0.25\%$ within the final ten passes on the {\tt diabetes} CVaR, whereas \lsaga gives $0.82 \pm 0.00\%$, i.e., a $40\%$ relative improvement with greater stability. 
While \saddlesaga does stabilize in SP on {\tt diabetes}, it fails to qualitatively decrease at all on the {\tt acsincome}. Interestingly, while suboptimality and SP-fairness are correlated for \lsaga, SGD (reaching only $10^{-1}$ suboptimality with respect to the CVaR objectives on {\tt acsincome}) achieves a lower fairness score. 
Again, across both suboptimality and fairness, \lsaga is either the best or close to the best.

\begin{figure*}[t]
    \centering
    \includegraphics[width=\linewidth]{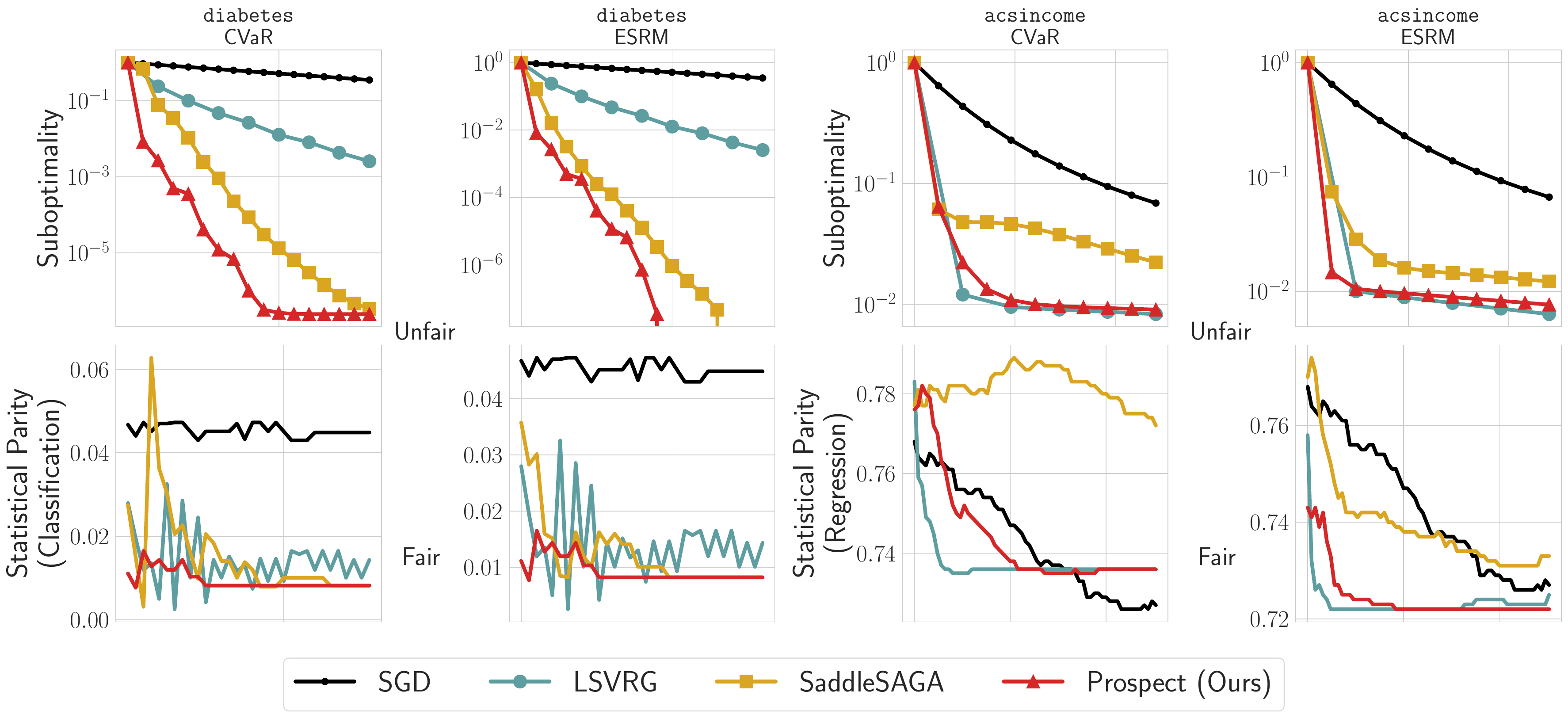}
    \caption{{\bf Fairness benchmarks}. {\bf Top:} Training curves on the CVaR and extremile for {\tt diabetes} ({\bf left}) and CVaR and extremile for {\tt acsincome} ({\bf right}). {\bf Bottom:} Statistical parity scores for the two classification objectives on {\tt diabetes} and regression objectives on {\tt acsincome}. Values closer to zero indicate better SP-fairness.
    }
    \label{fig:fairness}
\end{figure*}

\subsection{Image and Text Classification under Distribution Shift}

We consider two tasks from the WILDS distribution shift benchmark \citep{koh2021wilds}. 
The {\bf Amazon Reviews} ({\tt amazon}) task \citep{ni2019justifying} consists of classifying text reviews of products to a rating of 1-5, with disjoint train and test reviewers. 
The {\bf iWildCam} ({\tt iwildcam}) image classification challenge \citep{beery2020iwildcam} contains labeled images of animals, flora, and backgrounds from cameras placed in wilderness sites. Shifts are due to changes in camera angles, locations, lighting...
We use 
$n=10000$ and $n=20000$ examples respectively.
For both datasets, we train a \emph{linear probe classifier}, i.e., a linear model over a frozen deep 
representation.
For {\tt amazon}, we use a pretrained BERT model~\citep{Devlin2019BERTPO} fine-tuned on a held-out subset of the Amazon Reviews training set for 2 epochs. For {\tt iwildcam}, we use a ResNet50 pretrained on ImageNet (without fine-tuning).

\myparagraph{Evaluation}
Apart from the training suboptimality, we evaluate the spectral risk objectives on their robustness to subpopulation shifts. 
We define each subpopulation group based on the true label.
For {\tt amazon}, we use the \emph{worst group misclassification error} on the test set
\citep{Sagawa2020Distributionally}.
For \texttt{iwildcam}, we use the \emph{median group error} owing to its larger number of classes.

\myparagraph{Results}
For both {\tt amazon} and {\tt iwildcam}, \lsaga and \saddlesaga (with our heuristic) outperform \lsvrg in training suboptimality. We hypothesize that this phenomenon is due to checkpoints of \lsvrg getting stale over the $n$-length epochs for these datasets with large $n$ (leading to a slow reduction of bias). In contrast, \lsaga and \saddlesaga avoid this issue by dynamically updating the running estimates of the importance weights.
For the worst group error for {\tt amazon}, \lsaga and \saddlesaga outperform \lsvrg. \lsaga has a mean/std worst group error of $77.38 {\pm} 0.00 \%$ over the last ten passes on the extremile, whereas \saddlesaga has a slightly worse $77.53 {\pm} 1.57 \%$.
Interestingly, on {\tt iwildcam}, \lsvrg and \lsaga give stronger generalization performance, nearly $1$pp better, than \saddlesaga  in terms of median group misclassification rate. 
In summary, across tasks and objectives, \lsaga demonstrates  best or close to best performance.

\begin{figure*}[t]
    \centering
\includegraphics[width=\linewidth]{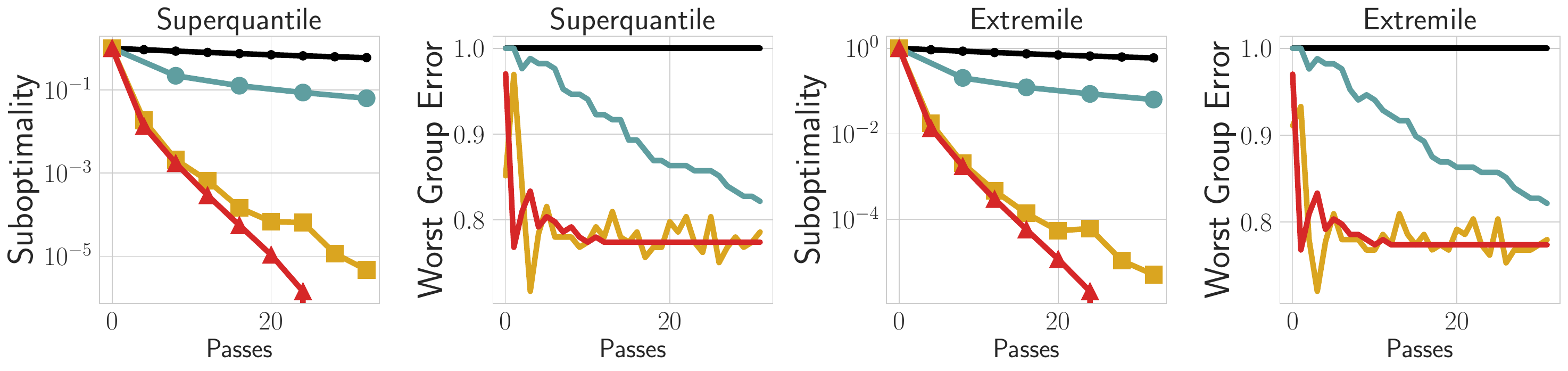}
\includegraphics[width=\linewidth]{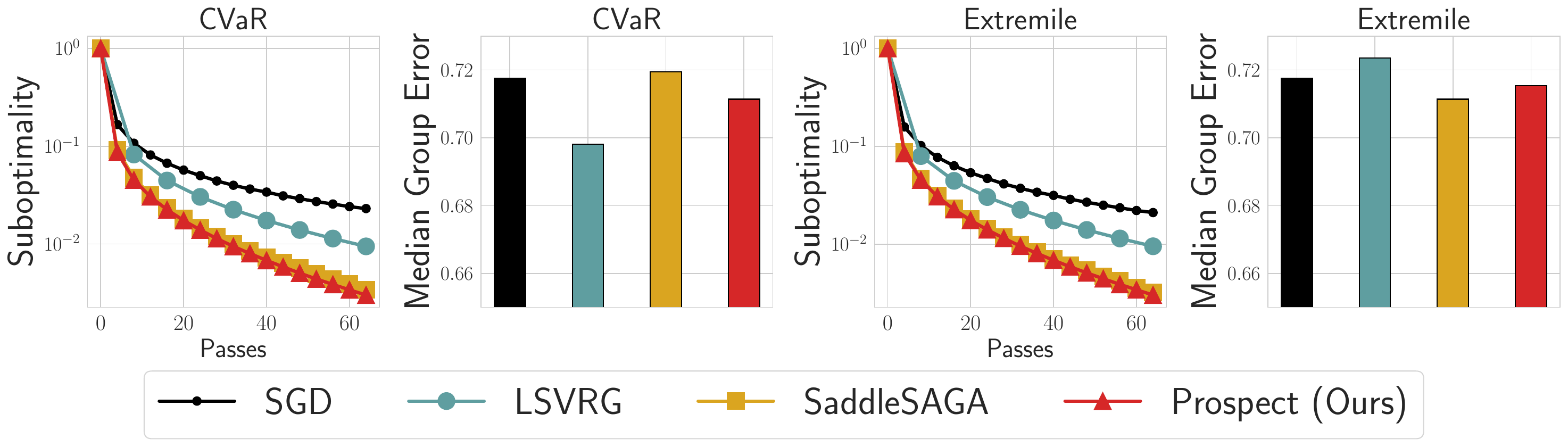}
    \caption{{\bf Distribution shift results}. {\bf Top row:} Training curves and worst group misclassification error on {\tt amazon} test. {\bf Bottom row:} Training curves and median group misclassification error on the {\tt iwildcam} test set.
    Smaller values indicate better performance for all metrics.}
    \label{fig:wilds}
\end{figure*}

\section{Discussion}\label{sec:discussion}
In this paper, we introduced \lsaga, a distributionally robust optimization algorithm for minimizing spectral risk measures with a linear convergence guarantee. Prospect demonstrates rapid linear convergence on benchmark examples 
and has the practical benefits of converging for any shift cost while only
having a single hyperparameter.
Promising avenues for future work include extensions to the non-convex setting
by considering the regular subdifferential, 
variations using other uncertainty sets, and further exploring connections to algorithmic fairness.

\subsubsection*{Acknowledgements}
This work was supported by NSF DMS-2023166, CCF-2019844, DMS-2052239, DMS-2134012, 
DMS-2133244, NIH, CIFAR-LMB, and faculty research awards. Part of this work was done while Zaid Harchaoui was visiting the Simons Institute for the Theory of Computing.

\clearpage
\bibliographystyle{abbrvnat}
\bibliography{main_arxiv.bib}

\clearpage
\appendix
\begingroup
\let\clearpage\relax 
\onecolumn 
\endgroup

\clearpage
\appendix
\addcontentsline{toc}{section}{Appendix} %
\part{Appendix} %
In the Appendix sections, we give summarize notation in \Cref{sec:a:notation} and provide intuition and results regarding the primal/dual objective function in \Cref{sec:a:objective}. We describe in detail efficient implementations of the proposed algorithm in \Cref{sec:a:efficient}. In \Cref{sec:a:any_shift}, we describe the convergence analyses of the main algorithm. In \Cref{sec:a:prox_saga} and \Cref{sec:a:saddle_saga}, we describe an Moreau envelope-based variant of our method and an improved version of an existing saddle point method, respectively. \Cref{sec:a:convex} contains technical results shared to multiple proofs. We then describe the experimental setup in detail in \Cref{sec:a:experiments} and give additional results in \Cref{sec:a:additional}. 

\parttoc %
\clearpage

\section{Summary of Notation}\label{sec:a:notation}

We summarize the notation used throughout in \Cref{tab:notation}.

\begin{table}[h!]
    \centering

    \begin{adjustbox}{max width=\linewidth}
    \renewcommand{\arraystretch}{1.4}
    \begin{tabular}{cc}
    \toprule
        {\bf Symbol} & {\bf Description}\\
        \midrule
        $\mu \geq 0$ & Standard regularization constant.\\

        $\nu \geq 0$ & Shift cost.\\

        $\alpha_n$ & Strong convexity constant for any $f$ generating an $f$-divergence.\\

        $\bar \nu$ & Shorthand $\bar \nu  = n\alpha_n  \nu$ (used in the convergence proofs). \\
        
        \midrule
        
        $\f_1(w), \ldots, \f_n(w)$ & Loss functions $\f_i: \R^d \rightarrow \R$.\\
        $\f(w)$ & Vector of losses $\f(w) = (\f_1(w), \ldots, \f_n(w))$ for $w \in \R^d$.\\
        $r_i(w)$ & Regularized loss $r_i(w) = \f_i(w) + \tfrac{\mu}{2}\|w\|_2^2$. \\
        $r(w)$ & Vector of regularized losses $ r(w) = (r_1(w), \ldots, r_n(w))$.\\

        $\grad \f(w)$ & Jacobian matrix of $\f: \R^d \rightarrow \R^n$ at $w$ (shape = $n \times d$).\\

        \midrule

        $\sigma$ & The vector $\sigma = (\sigma_1, \ldots, \sigma_n) \in [0, 1]^n$ where each $\sigma_1 \leq \ldots \leq \sigma_n$ and they sum to 1.\\
        $\mc{P}(\sigma)$ & The set $\{\Pi \sigma \, :\, \Pi\in [0,1]^{n\times n}, \Pi\ones_n=\ones_n, \Pi^\top \ones_n = \ones_n \}$, known as the permutahedron.\\

        \midrule

        $f$ & Convex function $f: [0, \infty) \to \reals \cup \{+\infty\}$ generating an $f$-divergence. \\
        $f^*$ & Convex conjugate $f^*(y) := \sup_{x \in \R} \br{xy - f(x)}$. \\
        $\Omega_f$ or $\Omega$ & 
        \begin{tabular}{c} 
        Shift penalty function $\Omega_f: \mc{P}(\sigma) \mapsto [0, \infty)$. \\ We consider $f$-divergence penalties $\Omega_f(q) = D_f(q \Vert \ones_n/n)$. 
        \end{tabular}\\
        
        \midrule

        $\primobj$ &  Main objective
        $\primobj(w) = \max_{q \in \Pcal(\sigma)} \br{q^\top \ell(w) - \nu D_f(q \Vert\ones_n/n)} + \frac{\mu}{2}\norm{w}_2^2$.
        \\
        \midrule
        $\q(l)$ or $q^{l}$ & 
        \begin{tabular}{c}
        Most unfavorable distribution for a given vector $l$ of losses, i.e., \\ $\q(l) = \argmax_{q \in \Pcal(\sigma)} q^\top l - \nu D(q\Vert \ones_n / n)$. \\
        $q^l$ used only in main text for readability.
        \end{tabular}
        \\
        $w^\star$ & Optimal weights $\argmin_{w \in \R^d} \max_{q \in \Pcal(\sigma)} q^\top l - \nu D(q\Vert \ones_n / n) + (\mu/2)\norm{w}_2^2$. \\
        $q^\star$ & 
        Most unfavorable distribution of $\ell(w^\star)$, i.e., 
        $q^\star = \q(\ell(w^\star))$ \\
        \midrule
        
        $G$ & Lipschitz constant of each $\f_i$ w.r.t. $\norm{\cdot}_2$.\\
        $L$ & Lipschitz constant of each $\grad \f_i$ w.r.t. $\norm{\cdot}_2$.\\
        $M$ & $M = L + \mu$, the Lipschitz constant of each $\grad r_i$ w.r.t. $\norm{\cdot}_2$.\\
        \midrule
        $\E{t}{\cdot}$ & \begin{tabular}{c} Shorthand for $\E{}{\, \cdot \, | \, w\pow{t}}$, i.e., expectation conditioned on $w\pow{t}$. \end{tabular}\\
        \bottomrule
    \end{tabular}
    \end{adjustbox}
    \vspace{6pt}
    \caption{Notation used throughout the paper.}
    \label{tab:notation}
\end{table}

\clearpage

\section{Properties of the Primal and Dual Objectives}\label{sec:a:objective}

\begin{figure}[t]
    \centering
    \includegraphics[width=0.3\linewidth]{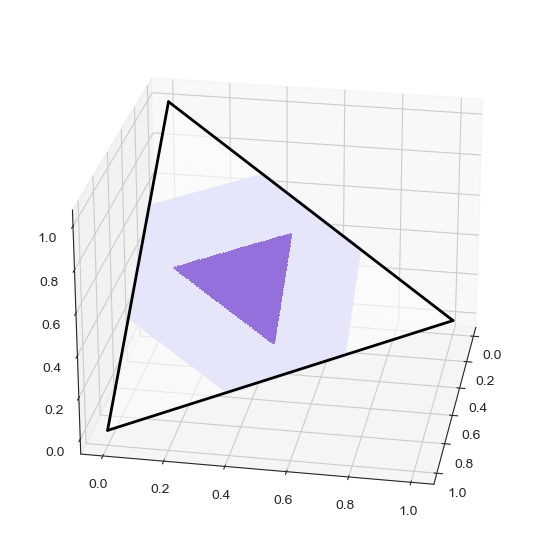}
    \includegraphics[width=0.3\linewidth]{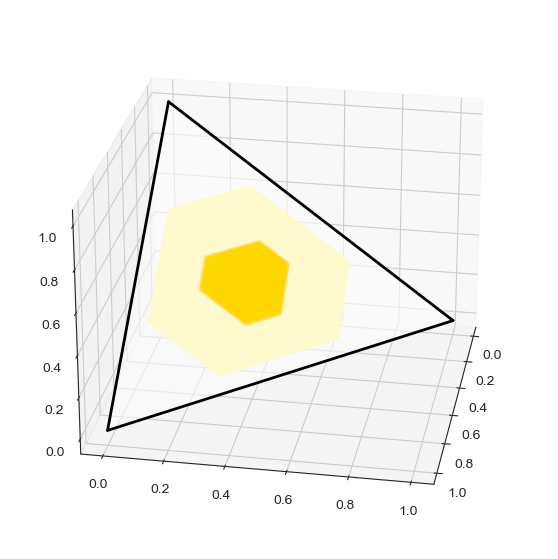}
    \includegraphics[width=0.3\linewidth]{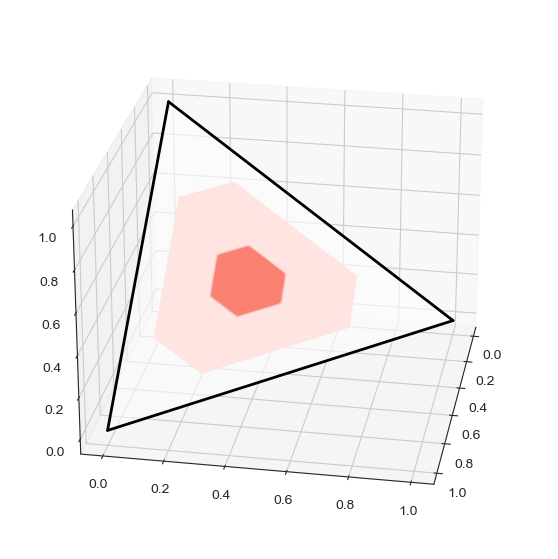}
    \\\vspace{-9pt}
    \includegraphics[width=0.925\linewidth]{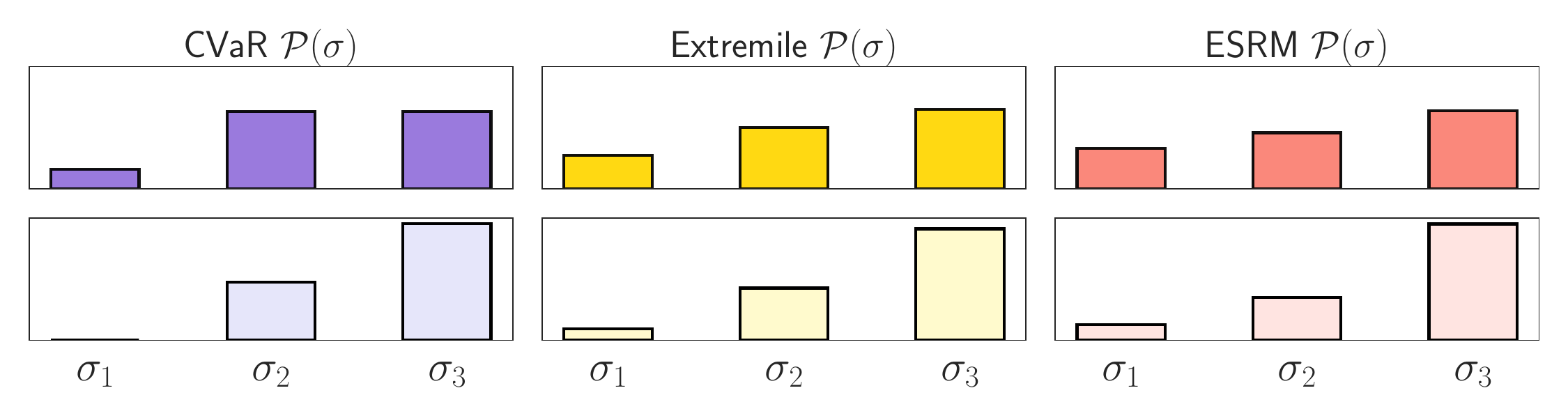}
    \caption{\textbf{Geometry of ambiguity sets}. Illustration of the permutahedra $\mc{P}(\sigma)$ within the three-dimensional probability simplex for the $0.25$ and $0.5$-CVaR ({\bf left}), $1.5$ and $2.5$-extremile ({\bf center}), and $1$ and $3$-ESRM ({\bf right}). The size of $\mc{P}(\sigma)$ increases for more non-uniform spectra $\sigma$.} 
    \label{fig:permutahedron}
\end{figure}

In this section, we state and prove the properties of the objectives we consider. 
Recall that we are interested in the optimization problem
\begin{align}
    \min_{w \in \R^d}\sbr{\primobj(w) := \max_{q \in \mc{P}(\sigma)} q^\top \ell(w) - \nu D_f(q \Vert \ones_n/n) + \frac{\mu}{2} \norm{w}_2^2},
    \label{eqn:primobj2}
\end{align}
where $D_f(q \Vert \ones_n/n)$ denotes an $f$-divergence between the distribution given by $q$ and the discrete uniform distribution $\ones_n/n = (1/n, \ldots, 1/n)$. 

Our goal for this section will be to derive properties of the function $\primobj(w)$, or the \emph{primal objective}, as well as the inner maximization problem, which we refer to as the \emph{dual objective}. Both will be useful in motivating and analyzing \lsaga (used for the primal minimization) and various subroutines used to compute the maximally unfavorable distribution (i.e., the maximizer over $q$ in the inner maximization).

\myparagraph{Uncertainty Set}
Recall that for a spectral risk measure with spectrum $\sigma = (\sigma_1, \ldots, \sigma_n)$, the uncertainty set $\Pcal(\sigma)$ is given by:
\begin{align}
        \Pcal(\sigma) = \operatorname{ConvexHull}\br{\p{\sigma_{\pi(1)}, \ldots, \sigma_{\pi(n)}}\, :\, 
    \pi \text{ is a permutation on } [n]
    }.
    \label{eq:permutahedron}
\end{align}
As for particular examples: for $p \in [0, 1]$, the $p$-\emph{CVaR} (a.k.a.~superquantile) \citep{Rockafellar2013Superquantiles,kawaguchi2020ordered, Laguel2022Superquantiles} requires that $k=np$ elements of $\sigma$ be non-zero with equal probability and that the remaining $n-k$ are zero. The $b$-\emph{extremile}~\citep{daouia2019extremiles} and $\gamma$-\emph{exponential spectral risk measure}~\citep{cotter2006extreme} define their spectra by $\sigma_i = (i / n)^b - ((i-1)/n)^b$ for $b \geq 1$ and 
$\sigma_i = e^{\gamma}/(1-e^{-\gamma})[e^{\gamma i/n} - e^{\gamma (i-1)/n}]$
for $\gamma > 0$, respectively. notice that when $\nu = 0$, we have that for any $l \in \R^n$,
\begin{align*}
    \risk_{\sigma}(l) = \max_{q \in \Pcal(\sigma)} q^\top l = \sum_{i=1}^n \sigma_i l_{(i)},
\end{align*}
where $l_{(1)} \leq \ldots \leq l_{(n)}$ are the order elements of $l$. For this reason, SRMs may also be called $L$-risks \citep{Maurer2021Robust}, based on classical $L$-estimators (linear combinations of order statistics) from the statistics literature \citep{Shorack2017Probability}. For a visualization of the feasible set $\Pcal(\sigma)$ for the CVaR, extremile, and ESRM, see \Cref{fig:permutahedron}.

\myparagraph{Review of $f$-Divergences}
Let $q$ and $p$ be any two probability mass functions defined on atoms $\{1, \ldots, n\}$.
Consider a convex function $f: [0, \infty) \mapsto \R \cup \{+\infty\}$ such that $f(1) = 0$, $f(x)$ is finite for $x > 0$, and $\lim_{x \rightarrow 0^+} f(x) = 0$. 
 The \emph{$f$-divergence} from $q$ to $p$ generated by this function $f$ is given by
\begin{align*}
    D_f(q \Vert p) := \sum_{i=1}^n f\p{\frac{q_i}{p_i}} p_i,
\end{align*}
where we define $0 f\p{0 / 0} := 0$ in the formula above. If there is an $i$ such that $p_i = 0$ but $q_i > 0$, we say $D_f(q \Vert p) = \infty$.
The $\chi^2$-divergence is generated by $f_{\chi^2}(x) = x^2 - 1$ and the KL divergence is generated by $f_\text{KL}(x) = x \ln x + \iota_{+}(x)$ where $\iota_{+}$ denotes the indicator that is zero for $x \geq 0$ and $+\infty$ otherwise, and we define $x\ln x = 0$ for all $x < 0$.

\myparagraph{The Dual Problem}
We describe the inner maximization first, that is
\begin{align}
    \max_{q \in \mc{P}(\sigma)} \br{q^\top l - \nu D_f(q \Vert \ones_n)} \,.
    \label{eqn:dualprobf}
\end{align}
Its properties will inform the algorithmic implementation for the minimization over $w$ in \eqref{eqn:primobj2}.
In our specific case, since we care about the $f$-divergence between $q$ and the uniform distribution $\ones_n/n$, we have
\begin{align}
\label{eq:fdiv-to-unif}
    D_f(q \Vert \ones_n/n) := \frac{1}{n}\sum_{i=1}^n f\p{n q_i}. 
\end{align}

We now derive the dual problem to \Cref{eqn:dualprobf}. This will lead to an algorithm to solve the optimization problem efficiently.
Throughout, we denote $f^*(y) := \sup_{x \in \R} \br{xy - f(x)}$ as the convex conjugate of $f$.

\begin{prop}
    \label{prop:isotonic}
    Let $l \in \R^n$ be a vector and $\pi$ be a permutation that sorts its entries in non-decreasing order, i.e., $l_{\pi(1)} \leq \ldots \leq l_{\pi(n)}$. Assume that $f^*$ is defined over $\R$ and $\abs{f^*(y)} < \infty$ for all $y \in \R$.
    Then, the maximization over the permutahedron subject to the shift penalty can be expressed as
    \begin{align}
        \max_{q \in \mc{P}(\sigma)} \br{q^\top l - \nu D_f(q \Vert \ones_n/n)} = \min_{\substack{\c \in \R^n \\ \c_1 \leq \ldots \leq \c_n}} \sum_{i=1}^n g_i(\c_i\,; l),
        \label{eqn:isotonic}
    \end{align}
    where we define
    \begin{align*}
        g_i(\c_i \,; \, l) := \sigma_i \c_i + \frac{\nu}{n}\,\, f^*\p{\frac{l_{\pi(i)} - \c_i}{\nu}}.
    \end{align*}
\end{prop}
\begin{proof}

Let $\iota_{\mathcal{P}(\sigma)}$ denote the indicator function of the permutahedron $\mathcal{P}(\sigma)$, which is $0$ inside $\mathcal{P}(\sigma)$ and $+\infty$ outside of $\mathcal{P}(\sigma)$. Its convex conjugate is the support function of the permutahedron, i.e., 
\[
\iota_{\mathcal{P}(\sigma)}^*(l) = \max_{q \in \mathcal{P}(\sigma)} q^\top l.
\]
For two closed convex functions $h_1$ and $h_2$ that are bounded from below, the convex conjugate of their sum is the infimal convolution of their conjugate~\citep[Proposition 6.3.1]{hiriart2004fundamentals}: 
\begin{align*}
    (h_1+h_2)^*(x) = \inf_y \left\{ h_1^*(y) + h_2^*(x-y)  \right\} \,.
\end{align*}
In our context, taking $h_1(q) = \iota_{\mathcal{P}(\sigma)}(q)$ and $h_2(q) = \Omega_f(q) := \nu D_f(q \Vert \ones_n/n)$, we have
\begin{align*}
    \sup_{q\in \mathcal{P}(\sigma)} \br{q^\top l  - \Omega_f(q)}
    & = \sup_{q\in \R^n} \br{ q^\top l - (\iota_{\mathcal{P}(\sigma)}(q) + \Omega_f(q)) } \\
    & = (\iota_{\mathcal{P}(\sigma)} + \Omega_f)^*(l) \\
    & = \inf_{y \in \R^n} \br{\iota_{\mathcal{P}(\sigma)}^*(y) + \Omega_f^*(l-y)} \\
    & = \inf_{y \in \R^n} \br{\max_{q\in \mathcal{P}(\sigma)} q^\top y + \Omega_f^*(l-y) }\\
    & = \inf_{y \in \R^n} \br{ \sum_{i=1}^n \sigma_i y_{(i)} + \Omega_f^*(l-y) },
    \numberthis 
    \label{eq:dual-of-dual}
\end{align*}
where $y_{(1)} \leq \ldots \leq y_{(n)}$ are the ordered values of $y \in \R^n$. 

Since for any $x \in \R^n$, $\Omega_f$ is decomposable into a sum of identical functions evaluated at the coordinates $(x_1, \ldots, x_n)$, that is, $\Omega_f(x) = \sum_{i=1}^n \omega(x_i)$, its convex conjugate is $\Omega_f^*(y) = \sum_{i=1}^n \omega^*(y_i)$. In our case, $\omega(x_i) = \frac{\nu}{n} f(n x_i)$ from \Cref{eq:fdiv-to-unif}, so $\omega^*(y_i) = (\nu/n)f^*(y_i/\nu)$.

Next, by convexity of $\omega^*$, and that $\omega^*$ is never $\pm\infty$ and defined over $\R$ by assumption, we have that if for scalars $l_i, l_j, y_i, y_j$ such that $l_i \leq l_j$ and $y_i \geq y_j$, then using \cref{lem:cvx_ordering}, we have that
\begin{align*}
    \omega^*(l_i-y_i) + \omega^*(l_j-y_j) \geq \omega^*(l_i-y_j) + \omega^*(l_j-y_i).
\end{align*}
Hence for $y$ to minimize  $\Omega_f^*(l-y) = \sum_{i=1}^n \omega^*(l_i -y_i)$, the coordinates of $y$ must be ordered as $l$. That is, if $\pi$ is an argsort for $l$, s.t. $l_{\pi(1)} \leq \ldots \leq l_{\pi(n)}$, then $y_{\pi(1)} \leq \ldots \leq y_{\pi(n)}$. Since  $\iota_{\mathcal{P}(\sigma)}^*(y) =  \sum_{i=1}^n \sigma_i y_{(i)}$ does not depend on the ordering of $y$, the solution of \eqref{eq:dual-of-dual} must also be ordered as $l$ such that  the dual problem \eqref{eq:dual-of-dual} can be written as 
\begin{align*}
    \inf_{\substack{y \in \R^n \\ y_{\pi(1)} \leq \ldots \leq y_{\pi(n)}}}
    \sum_{i=1}^n \sigma_i y_{\pi(i)} + \frac{\nu}{n} \,\, f^*\p{\frac{l_{\pi(i)} - y_{\pi(i)}}{\nu}} 
    &= \inf_{\substack{\c \in \R^n \\ \c_1 \leq \ldots \leq \c_n}} \sum_{i=1}^n \sigma_i \c_i + \frac{\nu}{n} \,\, f^*\p{\frac{l_{\pi(i)} - \c_i}{\nu}}\\
    &= \min_{\substack{\c \in \R^n \\ \c_1 \leq \ldots \leq \c_n}} \sum_{i=1}^n g_i(\c_i ; l).
\end{align*}
\end{proof}
The $f$-divergences we consider as running examples are:
\begin{align*}
    \fchi(x) &= x^2 - 1 \text{ and } \fchi^*(y) = y^2/4 +1 &\text{($\chi^2$-divergence)}\\
    \fkl(x) &= x\ln x + \iota_+(x) \text{ and } \fkl^*(y) = \exp\p{y-1}. &\text{(KL-divergence)}
\end{align*}
In both cases, the conjugates $f^*$ are well-behaved with regards to the conditions of \Cref{prop:isotonic}.

Because we are interested in computing the maximizer of \eqref{eqn:dualprobf}, we denote it as
\begin{align*}
    \q(l) &= \argmax_{q \in \mc{P}(\sigma)} \br{q^\top l - \nu D_f(q \Vert \ones_n/n)}.
\end{align*}
We establish conditions on $f$ under which the above is well-defined.
\begin{prop}
    Assume that $f: \R \rightarrow \R$ is $\alpha_n$-strongly convex on $[0, n]$. Then, $q \mapsto \nu D_f(q \Vert \ones_n/n)$ is $(\nu n\alpha_n)$-strongly convex with respect to $\norm{\cdot}_2$, and $\q(l)$ is well-defined.
    \label{prop:strong_cvx_fdiv}
\end{prop}
\begin{proof}
    Due to the $\alpha_n$-strong convexity of $f$, for any $q, \rho \in [0, 1]^n$ and any $\theta \in (0, 1)$ and any $i \in [n]$,
    \begin{align*}
        f\p{\theta nq_i + (1-\theta)n\rho_i} \leq \theta f(nq_i) + (1-\theta)f(n\rho_i) - \frac{\alpha_n}{2} \theta(1-\theta)(nq_i - n\rho_i)^2.
    \end{align*}
    We average this inequality over $i$, yielding
    \begin{align*}
        \frac{1}{n} \sum_{i=1}^n f\p{n(\theta q_i + (1-\theta)\rho_i)} \leq \theta \frac{1}{n} \sum_{i=1}^n  f(nq_i) + (1-\theta)\frac{1}{n} \sum_{i=1}^n f(n\rho_i) - \frac{\alpha_n}{2} \theta(1-\theta)\|nq_i - n\rho_i\|^2.
    \end{align*}
    Defining $\Omega_f(q) := D_f(q\Vert \ones_n/n)$, the statement above can be succinctly written as 
    \begin{align*}
        \Omega_f(\theta q + (1-\theta)\rho) \leq \theta \Omega_f(q) + (1-\theta)\Omega_f(\rho) -\frac{\alpha_n n}{2} \theta(1-\theta)\|q_i - \rho_i\|^2 \,.
    \end{align*}
    Therefore, $\Omega_f$ is $(\alpha_n n)$-strongly convex with respect to $\norm{\cdot}_2$ on $[0, 1]^n$, so $q \mapsto \nu D_f(q \Vert \ones_n/n)$ is $(\nu n \alpha_n)$-strongly convex. Because $\Pcal(\sigma)$ is closed and convex, and the maximization is of a strongly concave function, there is a unique maximizer.
\end{proof}
The next result allows use to use a minimizer of \eqref{eqn:isotonic} to compute the maximizer of \eqref{eqn:dualprobf}.
\begin{prop}
    In the setting of \Cref{prop:isotonic}, if 
    \begin{align*}
        \copt(l) \in \argmin_{\substack{\c \in \R^n \\ \c_1 \leq \ldots \leq \c_n}} \sum_{i=1}^n g_i(\c_i; l),
    \end{align*}
    then
    \begin{align}
        \q_i(l) = \frac{1}{n} [f^*]'\p{\tfrac{1}{\nu}(l_i - \copt_{\pi^{-1}(i)}(l))}.
        \label{eqn:mintomax}
    \end{align}
\end{prop}
The Pool Adjacent Violators (PAV) algorithm is designed exactly for the minimization~\eqref{eqn:isotonic}. The algorithm is described for the the $\chi^2$-divergence with implementation steps in \Cref{sec:a:efficient}. Both the argsort $\pi$ and the inverse argort $\pi^{-1}$ are mappings from $[n] = \{1, \ldots, n\}$ onto itself, but the interpretation of these indices are different for the input and output spaces $[n]$. The argsort $\pi$ can be thought of as an \emph{index finder}, in the sense that for a vector $l \in \R^n$, because $l_{\pi(1)} \leq \ldots \leq l_{\pi(n)}$, $\pi(i)$ can be interpreted as the index of an element of $l$ which achieves the rank $i$ in the sorted vector. On the other hand, $\pi^{-1}(i)$ can be thought of as a \emph{rank finder}, in that $\pi^{-1}(i) = \rank(i)$ is the position that $l_i$ takes in the sorted form of $l$. To summarize
\begin{align*}
    \pi: \underbrace{[n]}_{\text{ranks of losses}} \rightarrow \underbrace{[n]}_{\text{indices of training examples}} \ \text{ while } \ \pi^{-1}: \underbrace{[n]}_{\text{indices of training examples}} \rightarrow \underbrace{[n]}_{\text{ranks of losses}}\\
\end{align*}
We may equivalently write \eqref{eqn:mintomax} as
\begin{align}
    \q_i(l) = \frac{1}{n} [f^*]'\p{\tfrac{1}{\nu}(l_i - \copt_{\rank(i)}(l))}.
\end{align}
Finally, as seen in \Cref{sec:a:efficient}, it will be helpful to compute $\q$ in sorted order.  Because the $f$-divergence is agnostic to the ordering of the $q$ vector (as it is being compared to the uniform distribution), $q$ can also be sorted by $\pi$. Thus, we may also write
\begin{align}
    \q_{(i)}(l) = \frac{1}{n} [f^*]'\p{\tfrac{1}{\nu}(l_{(i)} - \copt_{i}(l))}.
\end{align}

\myparagraph{The Primal Function}
When divergence generator $f$ is strongly convex and the loss function $\ell: \R^d \rightarrow \R^n$ is convex and differentiable, 
we have that \Cref{eqn:primobj2} is differentiable, as we show next.

\begin{lemma}\label{lem:danskin}
    When the map $q \mapsto \nu D(q\Vert \ones_n/n)$ is strongly convex over $\Pcal(\sigma)$, then $\risk_\sigma$ is continuously differentiable with gradient given by
    \begin{align*}
        \grad \risk_\sigma(l) = \argmax_{q \in \Pcal(\sigma)} \br{q^\top l - \nu D(q \Vert \ones_n/n)} \in \R^n.
    \end{align*}
\end{lemma}
\begin{proof}
    Because $\Pcal(\sigma)$ as defined in~\eqref{eq:permutahedron} is closed and convex, the strongly concave function $q \mapsto q^\top l - \nu D(q \Vert \ones_n/n)$ has a unique maximizer. Because $\Pcal(\sigma)$ is closed subset of the compact set $\Delta^n$, it is also compact. By Danskin's theorem~\citep[Proposition B.25]{bertsekas1997nonlinear}, we have that $\primobj$ is continuously differentiable with the given gradient formula. 
\end{proof}

\begin{lemma}\label{prop:primobjgradient}
    Let $\ell: \R^d \rightarrow \R^n$ be differentiable with Jacobian $w \mapsto \grad \ell(w) \in \R^{n \times d}$. Let each $\ell_i: \R^d \rightarrow \R$ be convex. Assume $\nu > 0$. Let $f$ be $\alpha_n$-strongly convex on the interval $[0, n]$. Then, the function $\primobj$ from \Cref{eqn:primobj2}
    is differentiable with its gradient equal to
    \begin{align*}
        \grad \primobj(w) = \p{\grad \ell(w)}^\top \q(\ell(w)) + \mu w.
    \end{align*}
    Furthermore $l \mapsto \q(l)$ is $(\alpha_n n \nu)^{-1}$-Lipschitz continuous w.r.t. $\norm{\cdot}_2$.
\end{lemma}
\begin{proof}
    First, by \Cref{prop:strong_cvx_fdiv}, we have that $q \mapsto \nu D_f(q \Vert \ones_n)$ is $(\alpha_n n)$-strongly convex with respect to $\norm{\cdot}_2$ on $[0, 1]^n$. 
    Next, due to the convexity of each $\ell_i$ and the non-negativity of any $q \in \mc{P}(\sigma)$, we have that
    \begin{align*}
        w \mapsto  \max_{q \in \Pcal(\sigma)} 
        \br{q^\top \ell(w) - \nu \Omega_f(q) }
    \end{align*}
    is convex, as is its pointwise maximum (over $q$) of a family of convex functions $q\T \ell(w)$. 
    We have by \Cref{lem:danskin} that 
    $\primobj$ is continuously differentiable with
    \[ 
        \grad \primobj(w) = \grad \ell(w)^\top \q(\ell(w)) + \mu w \,.
    \]
    Moreover, by  \citet[Theorem 1]{nesterov2005smooth}, we have that $l \mapsto \q(l)$ is Lipschitz continuous with Lipschitz constant equal to the inverse of the strong convexity constant of $\nu\Omega_f$, which is $\nu \alpha_n n$.
\end{proof}

Returning to our canonical examples, we have that for the $\chi^2$, $\fchi(x) = x^2 - 1$ is $2$-strongly convex on $\R$ and that $\fkl(x) = x \ln x$ is $(1/n)$-strongly convex on $[0, n]$. Thus, the function $l \mapsto \q(l)$ will have Lipschitz constant $2n\nu$ and $\nu$, respectively.

\myparagraph{Smoothness Properties}
By applying \Cref{prop:primobjgradient} to Lipschitz continuous losses, we may achieve the following guarantee regarding the changes in $\q$ with respect to $w$.
\begin{lemma} \label{eq:q-smoothing}
    Let $f$ be $\alpha_n$-strongly convex on the interval $[0, n]$.
    For any $w_1, \ldots, w_n, w'_1, \ldots, w'_n \in \R^d$ construct $\bar{\ell}(w_1, \ldots, w_n) = \big(\ell_i(w_i)\big)_{i=1}^n \in \reals^n$, as well as $\bar{\ell}(w'_1, \ldots, w'_n)$ where each $\f_i$ is $G$-Lipschitz w.r.t. $\norm{\cdot}_2$. Then, we have
    \[
    \norm{\q(\bar{\ell}(w_1, \ldots, w_n)) - \q(\bar{\ell}(w'_1, \ldots, w'_n))}_2^2 = \frac{G^2}{n^2\alpha_n^2\nu^2} \sum_{i=1}^n \norm{w_i - w'_i}_2^2 \,.
    \]
\end{lemma}

\begin{proof}
    By the Lipschitz property of $\q$ (\Cref{prop:primobjgradient}), we have,
    \begin{align*}
        \norm{\q(\bar{\ell}(w_1, \ldots, w_n)) - \q(\bar{\ell}(w'_1, \ldots, w'_n))}_2^2
        &\le \frac{1}{n^2\alpha_n^2 \nu^2} \norm{\bar{\ell}(w_1, \ldots, w_n) - \bar{\ell}(w'_1, \ldots, w'_n)}_2^2\\
        &\le \frac{1}{n^2\alpha_n^2 \nu^2} \sum_{i=1}^n \p{\ell_i(w_i) - \ell_i(w'_i)}_2^2 \\
        &\le \frac{G^2}{n^2\alpha_n^2 \nu^2} \sum_{i=1}^n \norm{w_i - w'_i}_2^2 \,.
    \end{align*}
\end{proof}
As a special case of \Cref{eq:q-smoothing}, we may consider $w_1 = \cdots = w_n = w \in \R^d$ and $w'_1 = \cdots = w'_n = w' \in \R^d$, in which case the result reads
\[
    \norm{\q(\ell(w)) - \q(\ell(w'))}_2^2 = \frac{G^2}{n\alpha_n^2\nu^2} \norm{w - w'}_2^2 \,.
    \]

\myparagraph{Properties under No Shift Penalty}
Next, we use the smoothness properties above to prove \Cref{prop:no_shift} by virtue of the following proposition, which states the equivalence of the minimizers of ``no-cost'' and ``low-cost'' objectives.

\begin{proposition} \label{prop:no-shift-appendix}
Let $w^\star_\nu$ be the unique minimizer of \eqref{eq:lsaga_obj} with shift cost $\nu \geq 0$ and $\chi^2$-divergence penalty. Define $\ell_{(1)}(w^\star_0) < \ldots, < \ell_{(n)}(w^\star_0)$ to be the order statistics of $\ell_1(w^\star_0), \ldots, \ell_n(w^\star_0)$, which are assumed to be distinct. Consider $\nu_0$ such that
\begin{align}
    n\nu_0 \p{\sigma_{i+1} - \sigma_i} < \ell_{(i+1)}(w^\star_0) - \ell_{(i)}(w^\star_0) \text{ for } i = 1, \ldots, n.
    \label{eqn:hidden_smoothness}
\end{align}
We have that $w^\star_0 = w^\star_{\nu}$ for all $\nu \leq \nu_0$.
\end{proposition}
\begin{proof}
    For a vector $l \in \R^n$ and $\nu \geq 0$, consider
    \begin{align*}
        h_\nu(l) &\coloneqq \max_{q \in \Pcal(\sigma)} q^\top l - \nu n \norm{q - \ones_n/ n}_2^2  \\
        & = \max_{q \in \Pcal(\sigma)} q^\top \p{l + 2\nu\ones_n} - \nu n \norm{q}_2^2 - (\nu / n) \norm{\ones_n}_2^2\\
        & = \max_{q \in \Pcal(\sigma)} q^\top l - \nu n \norm{q}_2^2 + \nu \\
        & \coloneqq g_\nu(l) + \nu,
    \end{align*}
    where we used that $q^\top \ones = 1$ for all $q \in \mathcal{P}(\sigma).$
    For $\nu > 0$, by Danskin's theorem~\citep[Proposition B.25]{bertsekas1997nonlinear},
    \begin{align*}
        \grad h_\nu(l) =\grad g_\nu(l)=  \argmax_{q \in \Pcal(\sigma)} q^\top l - \nu n \norm{q}_2^2.
    \end{align*}
    By applying Proposition 5 of \cite{blondel2020fast}, we have that if
    \begin{align}
        n\nu_0 \p{\sigma_{i+1} - \sigma_i} < \ell_{(i+1)}(w^\star_0) - \ell_{(i)}(w^\star_0) \text{ for } i = 1, \ldots, n,
        \label{eqn:cost_condition}
    \end{align}
    for some $\nu_0>0$, then for any $\nu \leq \nu_0$,
    \begin{align*}
        \grad g_\nu(\ell(w^\star_0)) = \grad g_0(\ell(w^\star_0)).
    \end{align*}
    Denote our objective as
    \begin{align*}
        \mc{L}_{\sigma, \nu}(w) = h_\nu(\ell(w)) + \frac{\mu}{2}\norm{w}_2^2,
    \end{align*}
    where we explicitly show the dependence on the shift cost $\nu \geq 0$. For $\nu=0$, since the losses are differentiable and $\ell(w^\star_0)$ is composed of distinct coordinates,  $\mc{L}_{\sigma, 0}$ is differentiable at $w^\star_0$ with gradient $\grad \ell(w^\star_0)^\top \grad h_0(\ell(w_0^\star)) + \mu w^\star_0$~\citep[Proposition 2]{Mehta2022Stochastic}, where $\grad \ell(w^\star_0) \in \R^{n \times d}$ denotes the Jacobian of $\ell$ at $w_0^\star$. 
    Using the chain rule, we successively deduce
    \begin{align*}
        \grad \mc{L}_{\sigma, 0}(w^\star_0) = 0 &\iff \grad \ell(w^\star_0)^\top \grad h_0(\ell(w_0^\star)) + \mu w^\star_0 = 0\\
        &\iff \grad \ell(w^\star_0)^\top \grad g_0(\ell(w_0^\star)) + \mu w^\star_0 = 0\\
        &\iff \grad \ell(w^\star_0)^\top \grad g_\nu(\ell(w_0^\star)) + \mu w^\star_0 = 0\\
        &\iff \grad \ell(w^\star_0)^\top \grad h_{\nu}(\ell(w_0^\star)) + \mu w^\star_0 = 0\\
        &\iff \grad \mc{L}_{\sigma, \nu}(w^\star_0) = 0.
    \end{align*}
    Applying the first-order optimality conditions of $\mc{L}_{\sigma, 0}$ and $ \mc{L}_{\sigma, \nu}$, as well as the uniqueness of $w^\star_0$ completes the proof. 
\end{proof}

\Cref{prop:no_shift} of the main paper then follows by combining \Cref{prop:no-shift-appendix} above with the convergence guarantee \Cref{thm:lsaga:main} of \lsaga. 
Indeed, \Cref{thm:lsaga:main} shows that \lsaga is able to converge linearly for arbitrarily small $\nu > 0$ and as long as $\nu \leq \nu_0$. Under \Cref{prop:no-shift-appendix}, the minimizer will be equal to $w^\star_0$.

We interpret this phenomenon as the ``hidden smoothness'' of $\primobj$, in that the non-differentiable points of the map $w \mapsto \max_{q \in \Pcal(\sigma)} q^\top \ell(w)$ are precisely the points at which $\ell_i(w) = \ell_j(w)$ for some $i \neq j$, as the subdifferential may contain multiple elements \citep[Proposition 2]{Mehta2022Stochastic}. Thus, if the losses are well-separated enough (in comparison to the spectrum $\sigma$) at the minimizer $w^\star_0$, the objective for the non-smooth setting $\nu = 0$ and regularized setting $\nu > 0$ result in the same minimizer.

\section{Efficient Implementation of \lsaga}\label{sec:a:efficient}

\setcounter{algorithm}{0}
\begin{algorithm}[t]\caption{\lsaga (restated from \Cref{sec:lsaga_algo}) \label{algo:lsaga_restated}}
\begin{algorithmic}[1]
    \Statex {\bf Inputs:}
    Initial $w_0$, spectrum $\sigma$, number of iterations $T$, regularization $\mu > 0$, shift cost $\nu > 0$. 
    \Statex {\bf Hyperparameter:} Stepsize $\eta > 0$.
    \State \textbf{Initialize} $l \gets \ell(w_0)$ and $g_i \gets \grad \ell_i(w_0) + \mu w_0$ for $i=1,\ldots, n$.
    \State Set $q \gets \argmax_{\bar{q} \in \mc{P}(\sigma)} \bar{q}^\top l - \nu D(q \Vert \ones_n/n)$ and $\rho \gets q$.
    \State Set $\bar g \gets \sum_{i=1}^n \rho_i g_i \in \R^d$.
    \State Set $w \gets w_0$.
    \For{$T$ iterations}
        \State Sample $i, j \sim \Unif[n]$ independently.
        \State $v \gets nq_{i}(\grad \ell_{i}(w) + \mu w) - n\rho_{i} g_{i_t} + \bar g$. \Comment{\textbf{Iterate Update}}
        \State $w \gets w - \eta v$.
        \State $l_{j} \gets \ell_{j}(w)$. \Comment{\textbf{Bias Reducer Update}}
        \State $q \gets \argmax_{\bar{q} \in \mc{P}(\sigma)} \bar{q}^\top l - \nu D(\bar{q} \Vert \ones_n/n)$.
        \State $\bar g \gets \bar g  - \rho_{i} g_{i} + q_{i} \p{\grad \ell_{i}(w) + \mu w}$. \Comment{\textbf{Variance Reducer Update}}
        \State $g_{i} \gets \grad \ell_{i}(w) + \mu w$.
        \State $\rho_{i} \gets q_{i}$. 
    \EndFor
    \Statex {\bf Output:} Final point $w$.
\end{algorithmic}
\end{algorithm}
In this section, we describe \lsaga including computational details, in a way that is amenable to implementation. For convenience, the conceptual description of the algorithm from \Cref{sec:lsaga_algo} is restated in \Cref{algo:lsaga_restated}. 

\myparagraph{Efficient Implementation}
We exactly solve the maximization problem 
\begin{align}
    q = \q\p{l} = \argmax_{q \in \Pcal(\sigma)} \br{q^\top l - (\nu/n)\sum_{i=1}^n f(nq_i)}.
    \label{eqn:qt}
\end{align}
by a sequence of three steps:
\begin{itemize}
    \item {\bf Sorting:} Find $\pi$ such that $l_{\pi(1)} \leq \ldots \leq l_{\pi(n)}$.
    \item {\bf Isotonic regression:} Apply Pool Adjacent Violators (PAV) (Subroutine \ref{algo:pavchi2}) to solve the isotonic regression minimization problem \eqref{eqn:isotonic}, yielding solution $\c = \copt(l)$.
    \item {\bf Conversion:} Use \eqref{eqn:mintomax} to convert $\c$ back to $q = \q(l)$.
\end{itemize}
The sorting step runs in $O(n \ln n)$ elementary operations whereas the isotonic regression and conversion steps run in $O(n)$ operations. Crucially, retrieving $q$ from the output $\c = \copt(l)$ in the third step can be done by a single $O(n)$-time pass by setting
\begin{align*}
    q_{\pi(i)} = \frac{1}{n} [f^*]'\p{\tfrac{1}{\nu}(l_{\pi(i)} - \c_i)}
    \label{eqn:convert}
\end{align*}
for $i = 1, \ldots, n$, as opposed to computing the inverse $\pi^{-1}$ and use \eqref{eqn:mintomax} directly, which in fact requires another sorting operation and can be avoided. Because only one element of $l$ changes on every iteration, we may sort it by simply bubbling the value of the index that changed into its correct position to generate the newly sorted version. The full algorithm is given \Cref{algo:lsaga_efficient}. We give a brief explanation on the PAV algorithm for general $f$-divergences below.

\myparagraph{Pool Adjacent Violators (PAV) Algorithm}
First, recall the optimization problem we wish to solve:
\begin{align}
        \min_{\substack{\c \in \R^n \\ \c_1 \leq \ldots \leq \c_n}} \sum_{i=1}^n g_i(\c_i ; l), \quad\text{where}\quad g_i(\c_i ; l) := \sigma_i \c_i + \frac{\nu}{n} f^*\p{\frac{l_{\pi(i)} - \c_i}{\nu}}\,.
\end{align}
The objective can be thought of as fitting a real-valued monotonic function to the points $(1, l_{\pi(1)}), \ldots, (n, l_{\pi(n)})$, which would require specifying its values $(c_1, \ldots, c_n)$ on $(1, \ldots, n)$ and defining the function as any $x \in [c_j, c_{j+1}]$ on $(j, j+1)$. Because $l_{\pi(1)} \leq \ldots \leq l_{\pi(n)}$, if we evaluated our function $(c_1, \ldots, c_n)$ on a loss such as $\sum_{i=1}^n (l_{\pi(i)} - \c_i)^2$, we may easily solve the problem by returning $c_1 = \ell_{\pi(1)}, \ldots, c_n = l_{\pi(n)}$. However, by specifying functions $g_1, \ldots, g_n$ we allow our loss function to change in different regions of the inputs space $\{1, \ldots, n\}$. In such cases, the monotonicity constraint $c_1 \leq \ldots \leq c_n$ is often violated because individually minimizing $g_i(\c_i)$ for each $\c_i$ has no guarantee of yielding a function that is monotonic. 

The idea behind the PAV algorithm is to attempt a pass at minimizing each $g_i$ individually, and correcting \emph{violations} as they appear. To provide intuition, define $c^*_i \in \argmin_{\c_i \in \R} g_i(\c_i)$, and consider $i < j$ such that $c^*_i > c^*_j$. If $f^*$ is strictly convex, then $g_i(x) > g_i(c^*_i)$ for any $x < c^*_i$ and similarly $g_j(x) > g_j(c^*_j)$ for any $x > c^*_j$. Thus, to correct the violation, we decrease $c^*_i$ to $\bar{c}_i$ and increase $c^*_j$ to $\bar{c}_j$ until $\bar{c}_i = \bar{c}_j$. We determine this midpoint precisely by
\begin{align*}
    \bar{c}_i = \bar{c}_j = \argmin_{x \in \R} g_i(x) + g_j(x)
\end{align*}
as these are exactly the contributions made by these terms in the overall objective. The computation above is called \emph{pooling} the indices $i$ and $j$. We may generalize this viewpoint to \emph{violating chains}, that is collections of contiguous indices $(i, i+1, \ldots, i+m)$ such that $c_{j}^* < c^*_i$ for all $j < i$ and $c_{j}^* > c^*_{i+m}$ for all $j > i + m$, but $c_{i}^* > c^*_{i+m}$. One approach is use dynamic programming to identify such chains and then compute the pooled quantities
\begin{align*}
    \bar{c}_i = \argmin_{x \in \R} \sum_{k=1}^m g_{i+k}(x).
\end{align*}
This requires two passes through the vector: one for identifying violators and the other for pooling. The Pool Adjacent Violators algorithm, on the other hand, is able to perform both operations in one pass by greedily pooling violators as they appear. This can be viewed as a meta-algorithm, as it hinges on the notion that the solution of ``larger'' pooling problems can be easily computed from solutions of ``smaller'' pooling problems. Precisely, for indices $S \sse [n] = \{1, \ldots, n\}$ define
\begin{align*}
    \sol(S) = \argmin_{x \in \R} \sum_{i \in S} g_i(x).
\end{align*}
We rely on the existence of an operation $\pool$, such that for any $S, T \sse [n]$ such that $S \cap T = \emptyset$, we have that
\begin{align}
    \sol(S \cup T) = \pool\p{\sol(S), m(S), \sol(T), m(T)},
    \label{eqn:pool}
\end{align}
where $m(S)$ denotes ``metadata'' associated to $S$, and that the number of elementary operations in the $\pool$ function is $O(1)$ with respect to $\abs{S} + \abs{T}$. We review our running examples.

For the $\chi^2$-divergence, we have that $\fchi(x) = x^2 - 1$ and $\fchi^*(y) = y^2/4 +1$, so
\begin{align*}
    \sol(S) &= \argmin_{x \in \R} \br{x\p{\sum_{i \in S} \sigma_i} + \abs{S} + \frac{1}{4n\nu} \sum_{i\in S} (l_{\pi(i)} - x)^2}\\
    &= \frac{1}{\abs{S}}\sbr{\sum_{i \in S} l_{\pi(i)} - 2n\nu\sum_{i \in S} \sigma_i}\\
    \sol(S \cup T) &= \frac{1}{\abs{S} + \abs{T}} \sbr{\sum_{i \in S \cup T} l_{\pi(i)} - 2n\nu\sum_{i \in S \cup T} \sigma_i}\\
    &= \frac{\abs{S} \sol(S) + \abs{T} \sol(T)}{\abs{S} + \abs{T}}.
\end{align*}
Thus, the metadata $m(S) = \abs{S}$ used in the pooling step \cref{eqn:pool} is the size of each subset.

For the KL divergence, $\fkl(x) = x\ln x$ and $\fkl^*(y) = e^{-1}\exp\p{y}$, so
so
\begin{align*}
    \sol(S) &= \argmin_{x \in \R} \br{x\p{\sum_{i \in S} \sigma_i} + \frac{\nu}{ne} \sum_{i \in S} \exp\p{l_{\pi(i)}/ \nu} \exp\p{-x/\nu}}\\
    &= \nu\sbr{\ln\sum_{i\in S} \exp\p{l_{\pi(i)} /\nu} - \ln \sum_{i \in S} \sigma_i - \ln n - 1}\\
    \sol(S \cup T) &= \nu\sbr{\ln\sum_{i\in S \cup T} \exp\p{l_{\pi(i)} /\nu} - \ln \sum_{i \in S \cup T} \sigma_i - \ln n - 1}\\
    &= \nu\sbr{\ln\p{\sum_{i\in S} \exp\p{l_{\pi(i)} /\nu} + \sum_{i\in T} \exp\p{l_{\pi(i)} /\nu}} - \ln \p{\sum_{i \in S} \sigma_i + \sum_{i \in T} \sigma_i} - \ln n - 1}.
\end{align*}
Here, we carry the metadata $m(S) = (\ln \sum_{i\in S} \exp\p{l_{\pi(i)} /\nu}, \ln \sum_{i \in S} \sigma_i)$, which can easily be combined and plugged into the function
\begin{align}
    (m_1, m_2), (m'_1, m'_2) \mapsto \nu \sbr{\ln\p{\exp m_1 + \exp m'_1} - \ln\p{\exp m_2 + \exp m'_2} - \ln n - 1}.
    \label{eqn:klpool}
\end{align}
for two instances of metadata $(m_1, m_2)$ and $(m'_1, m'_2)$. We carry the ``logsumexp'' instead of just the sum of exponential quantities for numerical stability, and \Cref{eqn:klpool} applies this operation as well. It might be that $\sum_{i \in S} \sigma_i = 0$, e.g. for the superquantile. In this case, we may interpret $\sol(S) = -\infty$ and evaluate $\exp\p{-\infty} = 0$ in the conversion formula \eqref{eqn:convert}. Two examples of the PAV algorithm are given in Subroutine \ref{algo:pavchi2} and Subroutine \ref{algo:pavkl}, respectively. These operate by selecting the unique values of the optimizer and partitions of indices that achieve that value.

\myparagraph{Hardware Acceleration}
Finally, note that all of the subroutines in \Cref{algo:lsaga_efficient} (Subroutine \ref{algo:pavchi2}/Subroutine \ref{algo:pavkl}, Subroutine \ref{algo:convert}, and Subroutine \ref{algo:bubbling}) all require primitive operations such as control flow and linear scans through vectors. Because these steps are outside of the purview of oracle calls or matrix multiplications, they benefit from just-in-time compilation on the CPU. We accelerate these subroutines using the Numba package in Python and are able to achieve an approximate 50\%-60\% decrease in runtime across benchmarks.

\begin{algorithm}[t]\caption{\lsaga (with exact implementation details)\label{algo:lsaga_efficient}}
\begin{algorithmic}[1]
    \Statex {\bf Inputs:}
    Initial points $w_0$, spectrum $\sigma$, stepsize $\eta > 0$, number of iterations $T$, regularization parameter $\mu > 0$, shift cost $\nu > 0$, loss/gradient oracles $\ell_1, \ldots, \ell_n$ and $\grad \ell_1, \ldots, \grad \ell_n$.
    \State $l \gets \ell(w_0) \R^n$.
    \State $g \gets (\grad \ell_i(w_0) + \mu w_0)_{i=1}^n \in \R^{n \times d}$.
    \State $\pi \gets \argsort(l)$.
    \State $c \gets \pav(l, \pi, \sigma)$. \label{alg:lsaga:line:pav} \Comment{Subroutine \ref{algo:pavchi2} or Subroutine \ref{algo:pavkl}}
    \State $q \gets \convert(c, l, \pi, \nu, \mathbf{0}_n)$. \Comment{Subroutine \ref{algo:convert}}
    \State $\rho \gets q$.
    \State $\bar g \gets \sum_{i=1}^n \rho_i g_i \in \R^d$.
    \For{$T$ iterations}
        \State Sample $i, j \sim \Unif[n]$.
        \State $v \gets nq_{i}(\grad \ell_{i}(w) + \mu w) - n\rho_{i} g_{i} - \bar g$. \Comment{{\bf Update iterate}}
        \State $w \gets w - \eta v \pow t$.
        \State $l_{j} \gets \ell_{j}(w)$. \Comment{{\bf Update bias reducer}}
        \State $\pi \gets \bubble(\pi, l)$.\label{alg:lsaga:line:bubble}
        \Comment{Subroutine \ref{algo:bubbling}}
        \State $c \gets \pav(l, \pi, \sigma)$.
        \State $q \gets \convert(c, l, \pi, \nu, q)$.
        \State $\bar g \gets  \bar g - \rho_i g_{i} + q_i (\grad \ell_{i}(w) + \mu w)$. \Comment{{\bf Update variance reducer}}
        \State $g_{i} \gets \grad \ell_{i}(w) + \mu w$.
        \State $\rho_i \gets q_i$. 
    \EndFor
    \Statex {\bf Output:} Final point $w$.
\end{algorithmic}
\end{algorithm} 

\begin{subroutine}\caption{Pool Adjacent Violators (PAV) Algorithm for $\chi^2$ divergence\label{algo:pavchi2}}
\begin{algorithmic}[1]
    \Statex {\bf Inputs:}
    Losses $(\ell_{i})_{i \in [n]}$, argsort $\pi$, and spectrum $(\sigma_i)_{i \in [n]}$.
    \State Initialize partition endpoints $(b_0, b_1) = (0, 1)$, partition value $v_1 = l_{\pi(1)} - 2n\nu \sigma_1$, number of parts $k = 1$.
    \For{$i=2, \ldots, n$}
        \State Add part $k = k + 1$.
        \State Compute $v_{k} =  l_{\pi(i)} - 2n\nu \sigma_i$.
        \While{$k \geq 2$ and $v_{k-1} \geq v_{k}$}
            \State $v_{k-1} = \frac{(b_k - b_{k-1}) v_{k-1} + (i - b_k) v_k}{i - b_{k-1}}$.
            \State Set $k = k - 1$.
        \EndWhile
        \State $b_k = i$.
    \EndFor
    \Statex {\bf Output:} Vector $c$ containing $\c_i = v_k$ for $b_{k-1} < i \leq b_k$.
\end{algorithmic}
\end{subroutine}

\begin{subroutine}\caption{Pool Adjacent Violators (PAV) Algorithm for KL divergence \label{algo:pavkl}}
\begin{algorithmic}[1]
    \Statex {\bf Inputs:}
    Losses $(\ell_{i})_{i \in [n]}$, argsort $\pi$, and spectrum $(\sigma_i)_{i \in [n]}$.
    \State Initialize partition endpoints $(b_0, b_1) = (0, 1)$, number of parts $k = 1$.
    \State Initialize partition value $v_1 = \nu\p{l_{\pi(1)}/\nu - \ln \sigma_1 - \ln n - 1}$.
    \State Initialize metadata $m_1 = \ell_{\pi(1)}/\nu$ and $t_1 = \ln \sigma_1$.
    \For{$i=2, \ldots, n$}
        \State Add part $k = k + 1$.
        \State Compute $v_{k} =  \nu\p{l_{\pi(i)}/\nu - \ln \sigma_i - \ln n - 1}$.
        \State Compute $m_{k} = \ell_{\pi(i)}/\nu$ and $t_{k} = \ln \sigma_i$
        \While{$k \geq 2$ and $v_{k-1} \geq v_{k}$}
            \State $m_{k-1} = \operatorname{logsumexp}(m_{k-1}, m_{k})$ and $t_{k-1} = \operatorname{logsumexp}(t_{k-1}, t_{k})$.
            \State $v_{k-1} = \nu\p{m_{k-1} - t_{k-1} - \ln n - 1}$.
            \State Set $k = k - 1$.
        \EndWhile
        \State $b_k = i$.
    \EndFor
    \Statex {\bf Output:} Vector $c$ containing $\c_i = v_k$ for $b_{k-1} < i \leq b_k$.
\end{algorithmic}
\end{subroutine}

\begin{subroutine}[ht]
    \caption{Convert
    }\label{algo:convert}
    \begin{algorithmic}[1]
    \Require Sorted vector $c \in \R$, vector $l \in \R^n$, argsort $\pi$ of $l$, shift cost $\nu \geq 0$, vector $q \in \R^n$.
    \For{$i = 1, \ldots, n$}
        \State Set $q_{\pi(i)} = (1/n) [f^*]'\p{(l_{\pi(i)} - \c_i) / \nu}$.
    \EndFor
    \State \Return $q$.
    \end{algorithmic}
\end{subroutine}

\begin{subroutine}[ht]
    \caption{Bubble
    }\label{algo:bubbling}
    \begin{algorithmic}[1]
    \Require Index $j_{\text{init}}$, sorting permutation $\pi$, loss table $l$.
    \State $j = j_{\text{init}}$. \Comment{If $l_{\pi(j_{\text{init}})}$ too small, bubble left.}
    \While{$j > 1$ and $l_{\pi(j)} < l_{\pi(j-1)}$}
        \State Swap $\pi(j)$ and $\pi(j-1)$.
    \EndWhile
    \State $j = j_{\text{init}}$. \Comment{If $l_{\pi(j_{\text{init}})}$ too large, bubble right.}
    \While{$j < n$ and $l_{\pi(j)} > l_{\pi(j+1)}$}
        \State Swap $\pi(j)$ and $\pi(j+1)$.
    \EndWhile
    \State \Return $\pi$
    \end{algorithmic}
\end{subroutine}

\clearpage

\section{Convergence Analysis of \lsaga}\label{sec:a:any_shift}
This section provides the main convergence analysis for \lsaga. For readability of the proof, we reference the version of the algorithm presented in~\cref{algo:lsaga_conceptual}, which explicitly write the values of the iterates that fill the loss and gradient tables. Note that when implementing the algorithm, storing the iterates $\{z\pow{t}_i\}_{i=1}^n$ and $\{\zeta_i\pow t\}_{i=1}^n$ is not necessary. For simplicity, we use the shorthand
\begin{align*}
    \Omega(q) := \frac{1}{n \alpha_n} D_f(q \Vert \ones_n/n)
\end{align*}
for any $f$-divergence $D_f$, where $\alpha_n$ is the strong convexity constant of the generator $f$ over the interval $[0, n]$. By \Cref{prop:strong_cvx_fdiv}, this gives that $\Omega$ a $1$-strongly convex function over the probability simplex. 

\begin{algorithm}[t]
\caption{\lsaga (with iteration counters specified to accompany convergence analysis) \label{algo:lsaga_conceptual}}
    \begin{algorithmic}[1]
        \Statex {\bf Inputs:}
         Initial points $w \pow 0$, stepsize $\eta>0$, number of iterations $T$.
        \State Set  $z_i \pow 0 = \zeta_i \pow 0=  w \pow 0$ for all $i \in [n]$.
        \State 
        $q \pow 0 = 
        \argmax_{q  \in \mathcal{P}(\sigma)} q^\top \ell(w\pow 0) 
        - \bar\nu \Omega(q)$,
        $\rho \pow 0 = q \pow 0$.
        \State Set $l \pow 0 = (\ell_i(\zeta_i \pow 0))_{i=1}^n \in \R^n$,  
        $g \pow 0 = (\nabla r_i(z_i \pow 0))_{i=1}^n \in \R^{d\times n}$,  
        $\bar g \pow 0 = \sum_{i=1}^n \rho_i\pow 0 g_i \pow 0 \in \R^d$.
        \For{$t=0, \ldots, T-1$}
            \State $i_t \sim \Unif([n])$, $j_t \sim \Unif([n])$.
            \State
            \State  
            $v\pow t = n q_{i_t} \pow t \nabla r_{i_t}(w\pow t) 
            - (n \rho_{i_t} \pow t \nabla r_{i_t}(z_{i_t} \pow t)   - \bar g \pow t)$. \Comment{{\bf Iterate update.}}
            \State $w \pow {t+1} = w\pow t - \eta v \pow t$.
            \State
            \State $\zeta_{j_t} \pow {t+1} = w\pow t$ and $\zeta_{j} \pow {t+1} = \zeta_{j} \pow t$ for $j \neq j_t$. \Comment{{\bf Update bias reducer.}}
            \State $l \pow {t+1} = \ell(\zeta \pow {t+1})$.
            \State  $q \pow {t+1} 
            = \argmax_{q  \in \mathcal{P}(\sigma)} q^\top l \pow {t+1} - \bar\nu \Omega(q)$. 
            \State
            \State $z_{i_t} \pow {t+1} = w\pow t$ and $z_{i} \pow {t+1} = z_{i} \pow t$ for $i \neq i_t$. \Comment{{\bf Update variance reducer.}}
            \State $g \pow {t+1} = (\nabla r_i(z\pow {t+1}))_{i=1}^n$.
            \State $\rho_{i_t} \pow {t+1} = q_{i_t} \pow t$ and $\rho_i \pow {t+1} = \rho_i \pow t$ for $i \neq i_t$.
            \State $\bar g \pow {t+1} 
            = \sum_{i=1}^{n} \rho_i \pow {t+1} g_i \pow {t+1}$. 
        \EndFor
        \Statex {\bf Output:} Final point $w\pow T$
    \end{algorithmic}
\end{algorithm}

In the following, we denote $M = L + \mu$ the smoothness constant of the
regularized losses $r_i$. We denote $\mathbb{E}_{t}$ the expectation w.r.t to the randomness induced by picking $i_t, j_t$ given $w\pow t$, i.e. the conditional expectation given $w\pow{t}$. 
The optimum of~\eqref{eq:lsaga_obj} is denoted $w^\star$ and satisfies
\begin{equation}
  \nabla ({q^\star}^\top r(w^\star)) = 0, \ 
  \mbox{for} \ q^\star = \argmax_{q \in \mathcal{P}(\sigma)} q^\top \ell(w^\star) - \bar\nu \Omega(q).
\end{equation}
Denote in addition $l^\star = \ell(w^\star)$. All forthcoming statements will reference the setting of \Cref{algo:lsaga_conceptual}.

We first define the Lyapunov function $V\pow{t}$ that will be tracked in the proof, with $\|w\pow t-w^\star\|_2^2$ being called the ``main term''.
\begin{align*}
        V \pow t 
        & = \|w\pow t-w^\star\|_2^2 + c_1 \green{S \pow t} + c_2 \red{T \pow t} + c_3 \blue{U \pow t} + c_4 \yellow{R\pow t}
\end{align*}
where $c_1$, $c_2$, $c_3$, and $c_4$ are constants to be determined later, and
\begin{align*}
    &\green{S \pow t =  \frac{1}{n}\sum_{i=1}^n{\|n\rho_{i}\pow t \nabla r_{i}(z_{i}\pow t) - nq_{i}^*\nabla r_{i}(w^\star)\|_2^2}}, \quad
    \red{T \pow t =  \sum_{i=1}^n \|\zeta_i\pow t - w^\star\|_2^2}, \\
    & \blue{U \pow t = \frac{1}{n}\sum_{j=1}^n  \|w\pow t -\zeta_j \pow t\|_2^2}, \quad
    \yellow{R \pow t = 2\eta n (q \pow t - q^\star)^\top (l \pow t - l^\star)}.
\end{align*}
Though not included in the Lyapunov function, we will also introduce
\begin{align*}
    \purple{Q\pow t = \frac{1}{n}\sum_{i=1}^n \|n q_{i}\pow t \nabla r_{i}(w\pow t) - nq_{i}^\star\nabla r_{i}(w^\star)\|_2^2}.
\end{align*}
When the shift cost $\nu$ is large, we will be able to simplify the analysis by excluding the terms $U\pow{t}$ and $R\pow{t}$. The outline of the proof is as follows.
\begin{enumerate}
    \item We bound the evolution of the Lyapunov terms that are not the main term. For the large shift cost setting, only $\green{S\pow{t}}$ and $\red{T\pow{t}}$ are needed, while $\blue{U\pow{t}}$ and $\yellow{R\pow{t}}$ can be ignored.
    \item We expand the main term and identify ``descent'' and ``noise'' terms, as in a standard analysis of stochastic gradient methods. We bound the noise and establish a technical lemma that will be used to bound the descent terms.
    \item We split the proof into two subsections, one for the large shift cost and one for any shift cost. The descent lemma from the previous step will be materialized, and then we tune all the constants to give the final rate.
\end{enumerate}

\subsection{Step 1: Bound the evolution of the Lyapunov terms.}

We describe the evolution of the terms $S\pow{t}$, $T\pow{t}$, $U\pow{t}$, $R\pow{t}$ from iterate $t$ to iterate $t + 1$.

The first two terms are simply the closeness of the iterates $\{z_{i_t}\pow t\}_{i=1}^n$ and $\{\zeta_{i_t}\pow t\}_{i=1}^n$ within the table to the optimum $w^\star$, measured either in closeness in weighted gradients ($\green{S \pow t =  \frac{1}{n}\sum_{i=1}^n {\|n\rho_{i_t}\pow t \nabla r_{i_t}(z_{i_t}\pow t) - nq_{i_t}^\star\nabla r_{i_t}(w^\star)\|_2^2}}$) or directly ($\red{T \pow t =  \sum_{i=1}^n \|\zeta_i\pow t - w^\star\|_2^2}$).
Recall that $\purple{Q\pow t = \frac{1}{n}\sum_{i=1}^n \|n q_{i}\pow t \nabla r_{i}(w\pow t) - nq_{i}^\star\nabla r_{i}(w^\star)\|_2^2}$.
\begin{lemma}\label{lem:St_Tt}
    The following hold.
    \begin{align*}
        \E{t}{\green{S\pow{t+1}}} &= \frac{1}{n} \purple{Q\pow{t}} + \p{1 - \frac{1}{n}} \green{S\pow{t}},\\
        \E{t}{\red{T\pow{t+1}}} &= \|w\pow{t} - w^\star\|_2^2 + \p{1 - \frac{1}{n}} \red{T\pow{t}}.
    \end{align*}
\end{lemma}
\begin{proof}
    Write
    \begin{align*}
        &\E{t}{\green{S\pow{t+1}}} \\
        &= \frac{1}{n}\sum_{i=1}^n \E{t}{\|n\rho_{i}\pow {t+1} \nabla r_{i}(z_{i}\pow {t+1}) - nq_{i}^*\nabla r_{i}(w^\star)\|_2^2}\\
        &= \frac{1}{n}\sum_{i=1}^n \sbr{\frac{1}{n} \|n q_{i}\pow t \nabla r_{i}(w\pow t) - q_{i}^\star\nabla r_{i}(w^\star)\|_2^2 + \p{1 - \frac{1}{n}} \|n\rho_{i}\pow {t} \nabla r_{i_t}(z_{i}\pow {t}) - nq_{i}^*\nabla r_{i}(w^\star)\|_2^2}\\
        &= \frac{1}{n} \purple{Q\pow{t}} + \p{1 - \frac{1}{n}} \green{S\pow{t}}.
    \end{align*}
    Similarly,
    \begin{align*}
        \E{t}{\red{T\pow{t+1}}}
        &= \sum_{i=1}^n \E{t}{\|\zeta_i\pow{t+1} - w^\star\|_2^2}\\
        &= \sum_{i=1}^n \sbr{\frac{1}{n} \|w\pow{t} - w^\star\|_2^2 + \p{1 - \frac{1}{n}} \|\zeta_i\pow{t} - w^\star\|_2^2}\\
        &= \|w\pow{t} - w^\star\|_2^2 + \p{1 - \frac{1}{n}} \red{T\pow{t}}.
    \end{align*}
\end{proof}

Next, we handle $\blue{U \pow t = \frac{1}{n}\sum_{j=1}^n  \|w\pow t -\zeta_j \pow t\|_2^2}$, which can be ignored if we assume a particular lower bound on $\bar\nu$.
\begin{lemma}\label{lem:Ut}
    For any value of $\beta_2 > 0$, we have that
    \begin{align*}
        \E{t}{\blue{U\pow{t+1}}} &\leq \eta^2 (1 + \beta_2) \purple{Q\pow t} + \eta^2(1 + \beta_2^{-1})\green{S\pow t}\\
        &+ \frac{\eta M^2}{\mu n} \p{1 - \frac{1}{n}}\red{T\pow t} + \p{1 - \frac{1}{n}}\frac{G^2}{2\bar{\nu} \mu n} \yellow{R\pow t} + \p{1 - \frac{1}{n}} \blue{U\pow{t}}.
    \end{align*}
\end{lemma}
\begin{proof}
    First, note that a separate index $j_t$ (independent of $i_t$) is used to update $\{\zeta\pow{t}_j\}_{j=1}^n$, so we may first take the expected value with respect to $j_t$ conditioned on $i_t$:
    \begin{align*}
        \E{t}{U\pow{t+1}} &= \E{t}{\frac{1}{n}\sum_{j=1}^n  \|w\pow {t+1} -\zeta_j \pow {t+1}\|_2^2} \\
        &= \frac{1}{n} \E{t}{\frac{1}{n}\sum_{j=1}^n \norm{w\pow{t+1} -\zeta_j \pow {t+1}}_2^2 \mid j_t = j} + \p{1 - \frac{1}{n}} \E{t}{\frac{1}{n}\sum_{j=1}^n  \|w\pow {t+1} -\zeta_j \pow {t+1}\|_2^2 \mid j_t \neq j}\\
        &= \frac{1}{n} \E{t}{\norm{w\pow{t+1} - w\pow{t}}_2^2} + \p{1 - \frac{1}{n}} \E{t}{\frac{1}{n}\sum_{j=1}^n  \|w\pow {t+1} -\zeta_j \pow{t}\|_2^2}\\
        &= \frac{\eta^2}{n} \E{t}{\norm{v\pow{t}}_2^2} + \p{1 - \frac{1}{n}} \E{t}{\frac{1}{n}\sum_{j=1}^n  \|w\pow {t+1} -\zeta_j \pow{t}\|_2^2}.
    \end{align*}
    Next, we expand the second term. 
    \begin{align*}
        & \frac{1}{n}\E{t}{\sum_{j=1}^n\|w\pow {t+1} - \zeta_j \pow t\|_2^2}  \\
        & = \frac{1}{n}\E{t}{\sum_{j=1}^n \|w \pow {t+1} - w\pow t\|_2^2} 
         + \frac{2}{n}\E{t}{ \sum_{j=1}^n (w \pow {t+1} - w \pow t)^\top(w\pow t - \zeta_j \pow t)} 
         + \frac{1}{n}\E{t}{\sum_{j=1}^{n}\|\zeta_j \pow t- w\pow t\|_2^2} \\
        & = \eta^2 \E{t}{\|v\pow t\|_2^2} 
         - \frac{2\eta}{n} \sum_{j=1}^{n} \nabla ({q \pow t}^\top r)(w\pow t)^\top(w\pow t - \zeta_j \pow t)
          + \frac{1}{n}\sum_{j=1}^{n}\|\zeta_j \pow t - w\pow t\|_2^2.
    \end{align*}
    The first term is simply the noise term that appears in \Cref{lem:noise_bound_generic}, whereas the last term is $U\pow{t}$.
    Next, we have
    \begin{align*}
        -2\nabla ({q \pow t}^\top r)(w\pow t)^\top(w \pow t - \zeta_j\pow t) 
        & = -2(\nabla ({q \pow t}^\top r)(w\pow t) - \nabla (q {\pow t}^\top r)(\zeta_j \pow t))^\top(w\pow t - \zeta_j\pow t) \\
        & \quad - 2(\nabla (q {\pow t}^\top r)(\zeta_j \pow t) - \nabla ({q \pow t}^\top r)(w^\star))^\top (w\pow t - \zeta_j\pow t) \\
        & \quad - 2(\nabla ({q \pow t}^\top r)(w^\star) - \nabla ({q^\star}^\top r)(w^\star))^\top (w\pow t - \zeta_j\pow t),
    \end{align*}
    where the last term is introduced because $\nabla ({q^\star}^\top r)(w^\star) = 0$. We bound each of the three terms. First,
    \begin{align*}
        -2(\nabla ({q \pow t}^\top r)(w\pow t) - \nabla (q {\pow t}^\top r)(\zeta_j \pow t))^\top(w\pow t - \zeta_j\pow t) \leq -2\mu \blue{\norm{w\pow t - \zeta_j\pow t}_2^2}
    \end{align*}
    because ${q \pow t}^\top r$ is $\mu$-strongly convex \citep[Theorem 2.1.9]{Nesterov2018Lectures}. Second,
    \begin{align*}
        - 2(\nabla (q {\pow t}^\top r)(\zeta_j \pow t) - \nabla ({q \pow t}^\top r)(w^\star))^\top (w\pow t - \zeta_j\pow t)
        &\leq \beta_4 \|\nabla (q {\pow t}^\top r)(\zeta_j\pow t) - \nabla ({q \pow t}^\top r)(w^\star)\|_2^2  + \beta_4^{-1} \|\zeta_j\pow t - w\pow t\|_2^2\\
        &\leq \beta_4 M^2\red{\|\zeta_j\pow t - w^\star\|_2^2} + \beta_4^{-1} \blue{\|\zeta_j\pow t - w\pow t\|_2^2} 
    \end{align*}
    by Young's inequality with parameter $\beta_4$ and the $M$-Lipschitz continuity of $\nabla (q {\pow t}^\top r)$. Third,
    \begin{align*}
        - 2(\nabla ({q \pow t}^\top r)(w^\star) - \nabla ({q^\star}^\top r)(w^\star))^\top (w\pow t - \zeta_j\pow t) &= - 2(\nabla ((q {\pow t} - q^\star)^\top \ell)(w^\star))^\top (w\pow t - \zeta_j\pow t)\\
        &\leq \beta_5 \|\nabla ((q {\pow t} - q^\star)^\top \ell)(w^\star)\|_2^2 
        + \beta_5^{-1}\|\zeta_j\pow t - w\pow t\|_2^2 \\
        &\leq \beta_5 G^2 \yellow{\|q\pow t -q\|_2^2 }
        + \beta_5^{-1}\blue{\|\zeta_j\pow t - w\pow t\|_2^2},
    \end{align*}
    by Young's inequality with parameter $\beta_5$ and the $G$-Lipschitz continuity of each $\ell_i$. Combining with the above, we have
    \begin{align*}
        -2\sum_{j=1}^n \nabla ({q \pow t}^\top r)(w\pow t)^\top(w \pow t - \zeta_j\pow t) &\leq \beta_4 M^2\red{T\pow{t}} 
        + (\beta_4^{-1}+ \beta_5^{-1} -2\mu)\blue{U\pow{t}} 
         + \beta_5 G^2\yellow{n\|q\pow t-q^\star\|_2^2}\\
         &\leq \mu^{-1} M^2\red{T\pow{t}} + \mu^{-1} G^2\yellow{n\|q\pow t-q^\star\|_2^2}
    \end{align*}
    when we set $\beta_4 = \beta_5 = \mu^{-1}$.
    Hence, we get 
    \begin{align*}
        \E{t}{U \pow {t+1}} &= \frac{\eta^2}{n} \E{t}{\norm{v\pow{t}}_2^2} + \p{1 - \frac{1}{n}} \E{t}{\frac{1}{n}\sum_{j=1}^n  \|w\pow {t+1} -\zeta_j \pow{t}\|_2^2}\\
        &\leq \eta^2 \E{t}{\norm{v\pow{t}}_2^2} - \frac{\eta}{n}\p{1 - \frac{1}{n}} 2 \sum_{j=1}^{n} \nabla ({q \pow t}^\top r)(w\pow t)^\top(w\pow t - \zeta_j \pow t) + \p{1 - \frac{1}{n}} \blue{U\pow{t}}\\
        &\leq \eta^2 \E{t}{\norm{v\pow{t}}_2^2} + \frac{\eta}{n} \p{1 - \frac{1}{n}} \sbr{\mu^{-1} M^2\red{T\pow{t}} + \mu^{-1} G^2\yellow{n\|q\pow t-q^\star\|_2^2}} + \p{1 - \frac{1}{n}} \blue{U\pow{t}}\\
        &= \eta^2 \E{t}{\norm{v\pow{t}}_2^2} + \p{1 - \frac{1}{n}} \frac{\eta M^2}{\mu n}\red{T\pow{t}} + \p{1 - \frac{1}{n}}\frac{G^2}{2n\mu}\yellow{2n\eta\|q\pow t-q^\star\|_2^2} + \p{1 - \frac{1}{n}} \blue{U\pow{t}}\\
        &= \eta^2 \E{t}{\norm{v\pow{t}}_2^2} + \p{1 - \frac{1}{n}} \frac{\eta M^2}{\mu n}\red{T\pow{t}} + \p{1 - \frac{1}{n}}\frac{G^2}{2n\mu \bar{\nu}}\yellow{R\pow{t}} + \p{1 - \frac{1}{n}} \blue{U\pow{t}}\\
        &\leq \eta^2 (1 + \beta_2) \purple{Q\pow t} + \eta^2(1 + \beta_2^{-1})\green{S\pow t}\\
        &+ \frac{\eta M^2}{\mu n} \p{1 - \frac{1}{n}}\red{T\pow t} + \p{1 - \frac{1}{n}}\frac{G^2}{2\bar{\nu} \mu n} \yellow{R\pow t} + \p{1 - \frac{1}{n}} \blue{U\pow{t}},
    \end{align*}
    where the two last steps follow from \cref{lem:noise_bound_generic} and \cref{lem:smooth_weights} to claim 
    $\|q\pow t - q^\star\|_2^2 \leq \frac{1}{\bar \nu} (q\pow t-q^\star) (l\pow t-l^\star)$.
\end{proof}

Finally, we consider $\yellow{R \pow t = 2\eta n (q \pow t - q^\star)^\top (l \pow t - l^\star)}$. This can be viewed as a measurement of orthogonality between the vectors $q \pow t - q^\star$ and $l \pow t - l^\star$, which in turn can be viewed as the directions to the optimal gradient and optimal solution of a constrained optimization problem. Indeed, we may define
\begin{align*}
    l^\star = \argmin_{l \in \mc{L}} \sbr{h(l) := \max_{q \in \Pcal(\sigma)} q^\top l - \bar{\nu}\Omega(q)},
\end{align*}
and
\begin{align*}
    \mc{L} = \br{l \in \R^n: l \geq \ell(w) \text{ for some } w \in \R^d},
\end{align*}
where the inequality is taken element-wise. The set $\mc{L}$ is a convexification of the set $\ell(\R^d)$ which shares a minimizer and has the same minimum value. Indeed, letting $\bar{l}$ be any minimizer of $h$, select $\bar{w}$ such that $\bar{l} = \ell(\bar{w})$. Define $\bar{q} = \grad h(\bar{l}) = \argmax_{q \in \Pcal(\sigma)} q^\top \bar{l} - \bar{\nu}\Omega(\bar{q})$, and due to the non-negativity of $\bar{q}$, we have that
\begin{align*}
    \min_{l \in \mc{L}} h(l) = h(\bar{l}) = \bar{q}^\top \bar{l} - \bar\nu\Omega(\bar{q}) \geq \bar{q}^\top \ell(\bar{w}) - \bar\nu\Omega(\bar{q}).
\end{align*}
Taking the maximum over $\bar{q}$ shows that $\min_{l \in \mc{L}} h(l) = h(\ell(\bar{w}))$. For convexity, for any $l_1, l_2 \in \mc{L}$ satisfying $l_1 \geq \ell(w_1)$ and $l_2 \geq \ell(w_2)$, and any $\theta \in (0, 1)$, apply the following inequalities element-wise:
\begin{align*}
    \theta l_1 + (1- \theta) l_2 \geq \theta \ell(w_1) + (1-\theta) \ell(w_2) \geq \ell(\theta w_1 + (1-\theta) w_2),
\end{align*}
with $\theta w_1 + (1-\theta) w_2 \in \R^d$. By convexity, $(q \pow t - q^\star)^\top (l \pow t - l^\star) \geq 0$. Finally, this term is of particular importance because the term $-(q - q^\star)^\top (\ell(w) - \ell(w^\star))$ that appears in the original descent lemma \Cref{lem:descent_generic} can likely be used for cancellation in this case. The next result describes its evolution.
\begin{lemma}\label{lem:Rt}
    For any $\beta_3 > 0$, it holds that
    \begin{align*}
        \E{t}{\yellow{R\pow{t+1}}} &\leq 2\eta(q\pow t - q^\star)^\top(\ell(w\pow{t}) - l^\star) + \p{1 - \frac{1}{n}} \yellow{R\pow{t}}\\
        &+\frac{\eta G^2 n}{2\bar{\nu}} \beta_3^{-1} \red{T\pow{t}} + \frac{2\eta G^2n}{\bar{\nu}} (1 + \beta_3) \blue{U\pow{t}}.
    \end{align*}
\end{lemma}
\begin{proof}
    First decompose
    \begin{align}
        (q \pow {t+1} - q^\star)^\top (l \pow {t+1} - l^\star) & = 
        (q \pow t - q^\star)^\top (l \pow {t+1} - l^\star) 
        + (q \pow {t+1} - q \pow t)^\top (l \pow {t+1} - l \pow t) \\
        & \quad  
        + (q \pow {t+1} - q \pow t)^\top (l \pow t - l^\star).
        \label{eq:decompose}
    \end{align}
    Because $q\pow{t} = \q(l\pow{t})$ for all $t$, and $\q(\cdot)$ is the gradient of a convex and $(1/\bar{\nu})$-smooth function, we have for the second term of~\eqref{eq:decompose} that
    \[
        (q \pow {t+1} - q \pow t)^\top (l \pow {t+1} - l \pow t) 
        \leq \frac{1}{\bar \nu} \|l \pow {t+1} - l \pow t\|_2^2.
    \] 
    Next, using Young's inequality, that is, $a^\top b \leq
    \frac{\beta_3}{2} \|a\|_2^2 + \frac{\beta_3^{-1}}{2}\|b\|_2^2$ for any
    $\beta_3 >0$, we have for the third term term of~\eqref{eq:decompose} that
    \begin{align*}
        (q \pow {t+1} - q \pow t)^\top (l \pow t - l^\star) \leq
        \frac{\beta_3}{2} \|q \pow {t+1} - q \pow t\|_2^2 
        + \frac{\beta_3^{-1}}{2} \|l\pow t - l^\star\|_2^2 \\
        \leq \frac{\beta_3}{2 \bar \nu^2} \|l \pow {t+1} - l \pow t\|_2^2 
        + \frac{\beta_3^{-1}}{2} \|l\pow t - l^\star\|_2^2.
    \end{align*}
    Note that we have
    \[
      \E{t}{l\pow {t+1}} = \frac{1}{n} \ell(w \pow t) 
      + \left(1-\frac{1}{n}\right) l \pow t.
    \]
    Hence, we get, 
    \begin{align*}
        \frac{1}{2\eta n}\E{t}{\yellow{R \pow {t+1}}} 
        &=
        \frac{1}{n}(q \pow t - q^\star)^\top (l(w\pow{t}) - l^\star) + \p{1 - \frac{1}{n}} \yellow{(q \pow t - q^\star)^\top (l\pow{t} - l^\star)} \\
        &\quad + \E{t}{(q \pow {t+1} - q \pow t)^\top (l \pow {t+1} - l \pow t)} + \E{t}{(q \pow {t+1} - q \pow t)^\top (l \pow t - l^\star)}\\
        &\leq \frac{1}{n}(q \pow t - q^\star)^\top (l(w\pow{t}) - l^\star) + \p{1 - \frac{1}{n}} \yellow{(q \pow t - q^\star)^\top (l\pow{t} - l^\star)} \\
        &\quad + \p{\frac{1}{\bar\nu} + \frac{\beta_3}{2 \bar \nu^2}}\E{t}{\norm{l \pow {t+1} - l \pow t}_2^2}
        + \frac{\beta_3^{-1}}{2} \|l\pow t - l^\star\|_2^2\\
        &= \frac{1}{n}(q \pow t - q^\star)^\top (l(w\pow{t}) - l^\star) + \p{1 - \frac{1}{n}} \yellow{(q \pow t - q^\star)^\top (l\pow{t} - l^\star)} \\
        &\quad + \frac{1}{n\bar \nu}\p{1 + \frac{\beta_3}{2 \bar \nu}}\cdot 0 + \p{1 - \frac{1}{n}}\p{1 + \frac{\beta_3}{2 \bar \nu}}\norm{l \pow {t} - l \pow t}_2^2
        + \frac{\beta_3^{-1}}{2} \|l\pow t - l^\star\|_2^2\\
        & = \frac{1}{n}(q\pow t - q^\star)(\ell(w\pow t) - l^\star)
        + \left(1- \frac{1}{n}\right)\yellow{(q\pow t - q^\star)^\top(l \pow t - l^\star)} \\
        & \quad + \frac{1}{n\bar \nu}\left(1 + \frac{\beta_3}{2\bar \nu}\right)
        \sum_{j=1}^n(\ell_j(w \pow t) - \ell_j(\zeta_j))^2
        \\
        & \quad + \frac{\beta_3^{-1}}{2}\sum_{j=1}^n (\ell_j(\zeta_j) - \ell_j(w^\star))^2 \\
        & \leq \frac{1}{n}(q\pow t - q^\star)(\ell(w\pow t) - l^\star)
        + \left(1- \frac{1}{n}\right)\yellow{(q\pow t - q^\star)^\top(l \pow t - l^\star)} \\
        & \quad + \frac{G^2}{n\bar \nu}\left(1 + \frac{\beta_3}{2\bar \nu}\right)
        \blue{\sum_{j=1}^n\|w \pow t - \zeta_j\pow t\|_2^2}
        \\
        & \quad + \frac{G^2\beta_3^{-1}}{2}\red{\sum_{j=1}^{n}\|\zeta_j \pow t- w^\star\|_2^2}.
    \end{align*}
    Replacing $\beta_3$ by $2\bar \nu\beta_3$ gives the claim. 
\end{proof}

\subsection{Step 2: Bound the distance between the iterate and minimizer.}

Now that we have bounds on the evolution of the Lyapunov terms, we move to the main term. First, expand
\begin{align}
    \expect_t\|w\pow{t+1}-w^\star\|_2^2 = \|w\pow{t}-w^\star\|_2^2 \underbrace{- 2\eta \ip{\expect_{t}[v\pow{t}], w\pow{t} - w^\star}}_{\text{descent term}} + \underbrace{\eta^2\expect_t\|v \pow{t}\|_2^2}_{\text{noise term}}.
    \label{eq:expansion}
\end{align}

We use the following as a generic technical lemma to bound the descent term in~\eqref{eq:expansion}.
\begin{lemma}[Analysis of 1st order term]\label{lem:descent_generic} 
    Consider any $w \in \R^d$, $l \in \R^n$, and $\bar{q} \in \Pcal(\sigma)$. Define 
    \begin{align*}
        q &:= \q(l) = \argmax_{p \in \mc{P}(\sigma)} p^\top l - \bar\nu \Omega(p).
    \end{align*}
    For any $\beta_1 \in [0, 1]$,
    \begin{align*}
        -(\grad r(w)^\top q - \grad r(w^\star)^\top \bar{q})^\top (w - w^\star)
        &\leq -(q - \bar{q})^\top (\ell(w) - \ell(w^\star)) - \frac{\mu}{2} \norm{w - w^\star}_2^2 \\
        &- \frac{\beta_1}{4(M + \mu) \kappa_\sigma} \purple{\frac{1}{n}\sum_{i=1}^n \|nq_i \grad r_i(w) - nq^\star_i \grad r_i(w^\star)\|_2^2} + \frac{2\beta_1 G^2}{\bar\nu (M+\mu)\kappa_\sigma} \yellow{n(q - q^\star)^\top(l - l^\star)}.
    \end{align*}
\end{lemma}
\begin{proof}
    First, for any $q_i > 0$, we have that $w \mapsto q_i r_i(w)$ is $(q_i M)$-smooth and $(q_i \mu)$-strongly convex, so by applying standard convex inequalities (\cref{lem:smooth_strg_cvx}) we have that
    \begin{align*}
        q_ir_i(w^\star) &\geq q_i r_i(w) + q_i \grad r_i(w)^\top(w^\star - w)\\
        &+ \frac{1}{2\red{q_i}(M+\mu)} \norm{q_i \grad r_i(w) - q_i \grad r_i(w^\star)}_2^2 + \frac{q_i \mu}{4}\norm{w - w^\star}_2^2\\
        &\geq q_i r_i(w) + q_i \grad r_i(w)^\top(w^\star - w) \\
        &+ \frac{1}{2 \red{\sigma_n} (M+\mu)} \norm{q_i \grad r_i(w) - q_i \grad r_i(w^\star)}_2^2 + \frac{q_i \mu}{4}\norm{w - w^\star}_2^2
    \end{align*}
    as $q_i \leq \sigma_n$. The second inequality holds for $q_i = 0$ as well, so by summing the inequality over $i$ and using that $\sum_i q_i = 1$, we have that
    \begin{align*}
        q^\top r(w^\star) &\geq q^\top r(w) + q^\top \grad r(w) (w^\star - w) \\
        &+ \frac{1}{2 \sigma_n (M+\mu)} \sum_{i=1}^n \norm{q_i \grad r_i(w) - q_i \grad r_i(w^\star)}_2^2 + \frac{\mu}{4}\norm{w - w^\star}_2^2.
    \end{align*}
    Applying the same argument replacing $q$ by $\bar{q}$ and swapping $w$ and $w^\star$ yields
    \begin{align*}
        \bar{q}^\top r(w) &\geq \bar{q}^\top r(w^\star) + \bar{q}^\top \grad r(w^\star) (w - w^\star) \\
        &+ \frac{1}{2 \sigma_n (M+\mu)} \sum_{i=1}^n \norm{\bar{q}_i \grad r_i(w) - \bar{q}_i \grad r_i(w^\star)}_2^2 + \frac{\mu}{4}\norm{w - w^\star}_2^2.
    \end{align*}
    Summing the two inequalities yields
    \begin{align*}
        &-(q - \bar{q})^\top(r(w) - r(w^\star))\\ 
        &\geq -\p{\grad r(w)^\top q - \grad r(w^\star)^\top \bar{q}}^\top (w - w^\star) + \frac{\mu}{2}\norm{w - w^\star}_2^2\\
        &+ \frac{1}{2 \sigma_n (M+\mu)} \sbr{\sum_{i=1}^n \norm{q_i \grad r_i(w) - q_i \grad r_i(w^\star)}_2^2 + \sum_{i=1}^n \norm{\bar{q}_i \grad r_i(w) - \bar{q}_i \grad r_i(w^\star)}_2^2}.
    \end{align*}
    Dropping the $\sum_{i=1}^n \norm{\bar{q}_i \grad r_i(w) - \bar{q}_i \grad r_i(w^\star)}_2^2$ term and applying a weight of $\beta_1 \in [0, 1]$ to $\sum_{i=1}^n \norm{q_i \grad r_i(w) - q_i \grad r_i(w^\star)}_2^2$ still satisfies the inequality, which can equivalently be written as
    \begin{align*}
        -\p{\grad r(w)^\top q - \grad r(w^\star)^\top \bar{q}}^\top (w - w^\star) &\leq -(q - \bar{q})^\top(r(w) - r(w^\star)) - \frac{\mu}{2}\norm{w - w^\star}_2^2\\
        &- \frac{\beta_1}{2 \sigma_n (M+\mu)} \sum_{i=1}^n \norm{q_i \grad r_i(w) - q_i \grad r_i(w^\star)}_2^2. \numberthis \label{eq:descent:g1}
    \end{align*}
    Next, because
    \begin{align*}
        \norm{q_i \grad r_i(w) - q^\star_i \grad r_i(w^\star)}_2^2 &\leq 2 \norm{q_i \grad r_i(w) - q_i \grad r_i(w^\star)}_2^2 + 2(q_i - q^\star_i)^2 \norm{\grad r_i(w^\star)}_2^2,
    \end{align*}
    we have that (by summing over $i$) that
    \begin{align}
        -\sum_{i=1}^n \norm{q_i \grad r_i(w) - q_i \grad r_i(w^\star)}_2^2 \leq -\frac{1}{2} \sum_{i=1}^n \norm{q_i \grad r_i(w) - q^\star_i \grad r_i(w^\star)}_2^2 + 4G^2 \norm{q - q^\star}_2^2, \label{eq:descent:2}
    \end{align}
    where we used that each $\norm{\grad r_i(w^\star)}_2 \leq 2G$. To see this, use that $\grad r(w^\star)^\top q^\star = 0$ and $\grad r(w^\star) = \grad \ell(w^\star) + \mu w^\star$, so 
    \begin{align*}
        \norm{\grad r_i(w^\star)}_2 = \norm{\grad \ell_i(w^\star) + \mu w^\star}_2 = \norm{\textstyle \grad \ell_i(w^\star) - \sum_{j=1}^n q^\star_i \grad \ell_j(w^\star)}_2 \leq 2G.
    \end{align*}
    Because the map $\q$ is the gradient of a convex and $(1/\bar{\nu})$-smooth map, we also have that 
    \begin{align}
        \norm{q - q^\star}_2^2 = \norm{\q(l) - \q(\ell(w^\star))}_2^2 \leq \frac{1}{\bar\nu} (q - q^\star)^\top(l - \ell(w^\star)),
        \label{eqn:descent:nonnegative}
    \end{align}
    so we apply the above to \eqref{eq:descent:2} to achieve
    \begin{align*}
        &-\sum_{i=1}^n \norm{q_i \grad r_i(w) - q_i \grad r_i(w^\star)}_2^2\\
        &\leq -\frac{1}{2} \sum_{i=1}^n \norm{q_i \grad r_i(w) - q^\star_i \grad r_i(w^\star)}_2^2 + \frac{4G^2}{\bar\nu} (q - q^\star)^\top(l - \ell(w^\star)), \numberthis\label{eq:descent:3}
    \end{align*}
    We also use~\eqref{eqn:descent:nonnegative} to claim non-negativity of $(q - q^\star)^\top(l - \ell(w^\star))$.
    Finally, because $\sum_i q_i = \sum_i q^\star_i = 1$, we have that
    \begin{align*}
        (q - \bar{q})^\top (r(w) - r(w^\star)) &= (q - \bar{q})^\top \p{\ell(w) + \frac{\mu}{2}\norm{w}_2^2 \ones_n - \ell(w^\star) - \frac{\mu}{2}\norm{w^\star}_2^2\ones_n}\\
        &= (q - \bar{q})^\top \p{\ell(w) - \ell(w^\star)} + (q - \bar{q})^\top \ones_n \p{\norm{w}_2^2 - \norm{w^\star}_2^2}\\
        &= (q - \bar{q})^\top \p{\ell(w) - \ell(w^\star)}. \numberthis \label{eq:descent:4}
    \end{align*}
    Combine \eqref{eq:descent:g1}, \eqref{eq:descent:3}, and \eqref{eq:descent:4} along with $\kappa_\sigma = n\sigma_n$ to achieve the claim.
\end{proof}

Now, we move to analyzing the noise term term in~\eqref{eq:expansion}.
\begin{lemma}[Analysis of 2nd order term]\label{lem:noise_bound_generic}
    Consider the notations of~\cref{algo:lsaga_conceptual}, we have for any $\beta >0$,
    \begin{align*}
        \expect_{t}{\|v \pow t\|_2^2} 
        & \leq  (1+\beta)\expect_{t}{\|nq_{i_t} \pow t \nabla r_{i_t}(w \pow t) - nq_{i_t}^\star \nabla r_{i_t}(w^\star)\|_2^2} \\
        & \quad + (1+\beta^{-1})\expect_{t}{\|n\rho_{i_t}\pow t \nabla r_{i_t}(z_{i_t}\pow t) - nq_{i_t}^\star\nabla r_{i_t}(w^\star)\|_2^2}.
    \end{align*}
\end{lemma}
\begin{proof}
    In the following, we use the identity $\Ex\|X- \Ex[X]\|_2^2 = \Ex\|X\|_2^2 - \|\Ex[X]\|_2^2$ in equations denoted with $(\star)$. We denote by $\beta$ an arbitrary  positive number stemming from using Young's inequality $\|a+ b\|_2^2 \leq (1+ \beta) \|a\|_2^2 + (1+\beta^{-1}) \|b\|_2^2$ in equation $(\circ)$.
    Noting that $\sum_{i=1}^n q_i^\star \grad r_i(w^\star) = 0$, we get,
\begin{align*}
    & \E{t}{\|v\pow t -  \nabla ({q ^*}^\top r)(w^\star) \|_2^2} \\
    & = \expect_t\Big[
    \|nq_{i_t} \pow t \nabla r_{i_t}(w \pow t) - nq_{i_t}^\star \nabla r_{i_t}(w^\star) \\
    & \hspace{30pt}
    + nq_{i_t}^\star\nabla r_{i_t}(w^\star) - n\rho_{i_t}\pow t \nabla r_{i_t}(z_{i_t}\pow t) 
    -(\nabla ({q^\star}^\top r)(w^\star) - \bar g\pow t)\|_2^2 \Big] \\
    & \stackrel{(\star)}{=}  \|\nabla ({q \pow t}^\top r)(w\pow t) - \nabla ({q^\star}^\top r)(w^\star)\|_2^2 \\
    & \quad + \expect_t\Big[
    \|nq_{i_t} \pow t \nabla r_{i_t}(w \pow t) - nq_{i_t}^\star \nabla r_{i_t}(w^\star) 
    - (\nabla ({q \pow t}^\top r)(w\pow t) - \nabla ({q^\star}^\top r)(w^\star)) \\
    & \hspace{40pt}
    + nq_{i_t}^\star\nabla r_{i_t}(w^\star) - n\rho_{i_t}\pow t \nabla r_{i_t}(z_{i_t}\pow t) 
    -(\nabla ({q^\star}^\top r)(w^\star) - \bar g\pow t)\|_2^2 \Big] \\
    & \stackrel{(\circ)}{\leq}  \|\nabla ({q \pow t}^\top r)(w\pow t) - \nabla ({q^\star}^\top r)(w^\star)\|_2^2 \\
    & \quad + (1+\beta)\E{t}{
    \|nq_{i_t} \pow t \nabla r_{i_t}(w \pow t) - nq_{i_t}^\star \nabla r_{i_t}(w^\star) 
    - (\nabla ({q \pow t}^\top r)(w\pow t) - \nabla ({q^\star}^\top r)(w^\star))\|_2^2} \\
    & \quad
    + (1+\beta^{-1})\E{t}{
    \|nq_{i_t}^\star\nabla r_{i_t}(w^\star) - n\rho_{i_t}\pow t \nabla r_{i_t}(z_{i_t}\pow t) 
    -(\nabla ({q^\star}^\top r)(w^\star) - \bar g\pow t)\|_2^2} \\
    & \stackrel{(\star)}{=} -\beta \|\nabla ({q \pow t}^\top r)(w\pow t) - \nabla ({q^\star}^\top r)(w^\star)\|_2^2 \\
    & \quad 
    + (1+\beta)\E{t}{\|nq_{i_t} \pow t \nabla r_{i_t}(w \pow t) - nq_{i_t}^\star \nabla r_{i_t}(w^\star)\|_2^2} \\
    & \quad + (1+\beta^{-1})\E{t}{[\|nq_{i_t}^\star\nabla r_{i_t}(w^\star) - n\rho_{i_t}\pow t \nabla r_{i_t}(z_{i_t}\pow t)\|_2^2} \\
    & \quad -(1+\beta^{-1}) \|\nabla ({q^\star}^\top r)(w^\star) - \bar g\pow t\|_2^2.
\end{align*}
\end{proof}

We then combine the analyses of the first and second order terms to yield the main result of this subsection.

\begin{lemma}[Analysis of main term]\label{lem:main_term}
    For any constants $\beta_1 \in [0, 1]$ and $\beta_2 > 0$, and any $\bar{q} \in \Pcal(\sigma)$, we have that
    \begin{align*}
        \expect_t\|w\pow{t+1}-w^\star\|_2^2 &\leq (1- \eta \mu)\|w\pow{t}-w^\star\|_2^2 \\
        &-2\eta(w\pow{t} - w^\star)^\top \grad r(w^\star)\bar{q}\\
        &-\eta \p{\frac{\beta_1}{2(M + \mu) \kappa_\sigma} - \eta(1 +\beta_2)} \purple{Q\pow{t}} + \eta^2(1+\beta_2^{-1})\green{S \pow{t}}\\
        &+ \frac{2\beta_1 G^2}{\bar\nu (M+\mu)\kappa_\sigma} \yellow{R\pow{t}} -2\eta(q\pow{t} - \bar{q})^\top (\ell(w) - \ell(w^\star)).\\
    \end{align*}
\end{lemma}
\begin{proof}
    Recall the expansion given in~\eqref{eq:expansion}, which is:
    \begin{align}
        \expect_t\|w\pow{t+1}-w^\star\|_2^2 = \|w\pow{t}-w^\star\|_2^2 - 2\eta \ip{\expect_{t}[v\pow{t}], w\pow{t} - w^\star} + \eta^2\expect_t\|v \pow{t}\|_2^2.
    \end{align}
    Observe that
    \begin{align*}
        \expect_{t}[v\pow{t}] = \sum_{i=1}^n q\pow{t}_i \grad r(w\pow{t}) = \grad r(w\pow{t})^\top q\pow{t}
    \end{align*}
    By \Cref{lem:descent_generic} with $l = l\pow{t}$, $q = q\pow{t}$, $w = w\pow{t}$, and multiplying by $2\eta$, we have that
    \begin{align*}
    -2\eta(w\pow{t} - w^\star)^\top \grad r(w\pow{t})^\top q\pow{t} &\leq -2\eta(w\pow{t} - w^\star)^\top \grad r(w^\star)\bar{q} -2\eta(q\pow{t} - \bar{q})^\top (\ell(w\pow{t}) - \ell(w^\star))\\
    &- \mu \eta \norm{w\pow{t} - w^\star}_2^2 - \frac{\eta\beta_1}{2(M + \mu) \kappa_\sigma} Q\pow{t}\\
    &+ \frac{2\beta_1 G^2}{\bar\nu (M+\mu)\kappa_\sigma} \yellow{R\pow{t}}.
    \end{align*}
    The noise term is bounded by applying \cref{lem:noise_bound_generic}, so that for some $\beta_2 > 0$,
    \begin{align*}
        \eta^2 \expect_{t}{\|v \pow t\|_2^2}
        & \leq  \eta^2(1+\beta_2)\purple{Q\pow{t}} + \eta^2(1+\beta_2^{-1})\green{S \pow{t}}.
    \end{align*}
    Combine the two displays above to get the desired result.
\end{proof}

\subsection{Step 3: Tune constants and achieve final rate.}

Recall that our Lyapunov function is given by 
\begin{align*}
    V \pow t 
    & = \|w\pow t-w^\star\|_2^2 +  c_1 \green{S \pow t} + c_2 \red{T \pow t} + c_3 \blue{U \pow t} + c_4\yellow{R\pow t}.
\end{align*}
Recall in addition the definitions
\begin{align*}
    &\green{S \pow t =  \frac{1}{n}\sum_{i=1}^n{\|n\rho_{i}\pow t \nabla r_{i}(z_{i}\pow t) - nq_{i}^*\nabla r_{i}(w^\star)\|_2^2}}, \quad
    \red{T \pow t =  \sum_{i=1}^n \|\zeta_i\pow t - w^\star\|_2^2}, \\
    & \blue{U \pow t = \frac{1}{n}\sum_{j=1}^n  \|w\pow t -\zeta_j \pow t\|_2^2}, \quad
    \yellow{R \pow t = 2\eta n (q \pow t - q^\star)^\top (l \pow t - l^\star)}.
\end{align*}
We will derive a value $\tau > 0$ such that for all $t \geq 0$,
\[
    \E{t}{V\pow {t+1}} \leq (1-\tau^{-1})V\pow t.
\]
In order to describe our rates, we define the condition number $\kappa := M/\mu$ and recall that $\kappa_\sigma = n \sigma_n$.

\subsubsection{Step 3a: Analyze \emph{large} shift cost setting.}

The following gives the convergence rate for large shift cost.
\begin{theorem}
    Suppose the shift cost satisfies
    \begin{align*}
        \bar{\nu} \geq 8nG^2 / \mu.
    \end{align*}
    Then, the sequence of iterates produced by \Cref{algo:lsaga_conceptual} with $\eta = 1 / (12(\mu+M)\kappa_\sigma)$ achieves
    \[
        \expect\|w\pow{t} - w^\star\|_2^2 \le (1 + \sigma_n^{-1} + \sigma_n^{-2}) \exp(-t/\tau) \|w\pow{0} - w^\star\|_2^2 \,.
    \]
    with
    \begin{align*}
        \tau = 2\max\{n, 24\kappa_\sigma (\kappa + 1)\}.
    \end{align*}
    \label{thm:large}
\end{theorem}
\begin{proof}
    First, invoke \Cref{lem:main_term} with $q' = q\pow{t}$ and $\beta_1 = 1$ to obtain
    \begin{align}
        \expect_t\|w\pow{t+1}-w^\star\|_2^2 
        &\leq (1- \eta \mu)\|w\pow{t}-w^\star\|_2^2 \\
        &-2\eta(w\pow{t} - w^\star)^\top \grad r(w^\star)q\pow{t} + \frac{2 G^2}{\bar\nu (M+\mu)\kappa_\sigma} \yellow{R\pow{t}}\label{line:youngs}\\
        &-\eta \p{\frac{1}{2(M + \mu) \kappa_\sigma} - \eta(1 +\beta_2)} \purple{Q\pow{t}} + \eta^2(1+\beta_2^{-1})\green{S \pow{t}}.
    \end{align}
    We will first bound~\eqref{line:youngs}, by using that $\grad r(w^\star)q^\star = 0$ and Young's inequality with parameter $a > 0$ to write
    \begin{align*}
        \abs{(w\pow{t} - w^\star)^\top \grad r(w^\star)q\pow{t}} &= \abs{(w\pow{t} - w^\star)^\top \grad r(w^\star)(q\pow{t} - q^\star)}\\
        &\leq \frac{a}{2} \norm{\grad r(w^\star)^\top (q\pow{t} - q^\star)}_2^2 + \frac{1}{2a} \norm{w\pow{t} - w^\star}_2^2\\
        &\leq \frac{aG^2\gamma_*^2}{2\bar{\nu}^2} \red{T\pow{t}} + \frac{1}{2a} \norm{w\pow{t} - w^\star}_2^2,
    \end{align*}
    where we used in the second inequality that:
    \begin{align*}
        \norm{\grad r(w^\star)^\top (q\pow{t} - q^\star)}_2^2 &= \norm{\grad \ell(w^\star)^\top (q\pow{t} - q^\star)}_2^2 \leq \gamma_*^2 \norm{q\pow{t} - q^\star}_2^2 \leq \frac{\gamma_*^2}{\bar{\nu}^2} \norm{l\pow{t} - l^\star}_2^2 \\
        &\leq \frac{G^2\gamma_*^2}{\bar{\nu}^2} \sum_{i=1}^n \|\zeta_i\pow t - w^\star\|_2^2 = \frac{G^2\gamma_*^2}{\bar{\nu}^2} \red{T\pow{t}}.
    \end{align*}
    We also have by Cauchy-Schwartz and Lipschitz continuity that
    \begin{align*}
        \yellow{R\pow{t}} = 2\eta n (q \pow t - q^\star)^\top (l \pow t - l^\star) \leq \frac{2\eta n}{\bar{\nu}} \norm{l \pow t - l^\star}_2^2 \leq \frac{2\eta n G^2}{\bar{\nu}} \red{T\pow{t}}.
    \end{align*}
    Combining the above displays yields
    \begin{align*}
        &-2\eta(w\pow{t} - w^\star)^\top \grad r(w^\star)q\pow{t} + \frac{2 G^2}{\bar\nu (M+\mu)\kappa_\sigma} \yellow{R\pow{t}} \\
        &\leq \frac{\eta G^2}{\bar{\nu}^2}\sbr{a\gamma_*^2 + \frac{4 n G^2}{(M+\mu)\kappa_\sigma}} \red{T\pow{t}} + \frac{\eta}{a} \norm{w\pow{t} - w^\star}_2^2.
    \end{align*}
    We take $\beta_2 = 2$, $c_3 = c_4 = 0$, and apply \Cref{lem:St_Tt} to achieve
    \begin{align*}
        \E{t}{V\pow{t + 1}} - (1 - \tau^{-1}) V\pow{t}\leq &\sbr{\tau^{-1} -\eta\mu + \eta a^{-1} + c_2}\|w\pow{t}-w^\star\|_2^2 \\
        &+ \sbr{\tau^{-1} +\frac{3\eta^2}{2c_1} - \frac{1}{n}}c_1\green{S \pow{t}}\\
        &+\sbr{\tau^{-1} + \frac{\eta G^2}{\bar{\nu}^2c_2}\p{a\gamma_*^2 + \frac{4 n G^2}{(M+\mu)\kappa_\sigma}} - \frac{1}{n}}c_2\red{T\pow{t}}\\
        &+\sbr{-\frac{\eta}{2(M + \mu) \kappa_\sigma} + 3\eta^2 + \frac{c_1}{n}} \purple{Q\pow{t}},
    \end{align*}
    where $\tau > 0$ is a to-be-specified rate constant. We now need to set the various free parameters $a$, $c_1$, $c_2$, and $\eta$ to make each of the squared bracketed terms be non-positive. We enforce $\tau \geq 2n$ throughout. By setting
    \begin{align*}
        \eta = \frac{1}{12(\mu + M)\kappa_\sigma} \text{ and } c_1 = \frac{n\eta}{4(\mu + M)\kappa_\sigma},
    \end{align*}
    we have that the bracketed constants before $c_1\green{S \pow{t}}$ and $\purple{Q\pow{t}}$ vanish. Then, setting
    \begin{align*}
        a^{-1} = \frac{\mu}{2} \text{ and } c_2 = \frac{1}{48(\kappa + 1)\kappa_\sigma}
    \end{align*}
    make the bracketed constant before $\|w\pow{t}-w^\star\|_2^2$, assuming that we enforce
    \begin{align*}
        \tau \geq 48(\kappa+1)\kappa_\sigma.
    \end{align*}
    We turn to the final constant after substituting the values of $a$, $c_2$, and $\eta$. We need that
    \begin{align*}
        \frac{\eta G^2}{\bar{\nu}^2c_2}\p{a\gamma_*^2 + \frac{8 n G^2}{(M+\mu)\kappa_\sigma}} &= \frac{8G^2}{\bar{\nu}^2\mu^2}\p{\gamma_*^2 + \frac{2 n G^2}{(\kappa+1)\kappa_\sigma}} \leq \frac{1}{2n},
    \end{align*}
    which occurs when
    \begin{align*}
        \bar{\nu}^2 \geq \frac{16nG^2}{\mu^2}\sbr{\gamma_*^2 + \frac{2nG^2}{(\kappa+1)\kappa_\sigma}}.
    \end{align*}
    Because $\gamma_*^2 \leq nG^2 \leq 2nG^2$, this is achieved when
    \begin{align*}
        \nu \geq \frac{8nG^2}{\mu},
    \end{align*}
    completing the proof of the claim 
    \begin{align*}
        \E{t}{V\pow{t+1}} \leq (1 - \tau^{-1})V\pow{t}.
    \end{align*}
    To complete the proof, we bound the initial terms. Because $c_3 = c_4 = 0$, we need only to bound $S\pow{0}$ and $T\pow{0}$.
    \begin{align*}
        S \pow 0 &=  \frac{1}{n}\sum_{i=1}^n{\|n\rho_{i}\pow{0} \nabla r_{i}(z_{i}\pow 0) - nq_{i}^\star\nabla r_{i}(w^\star)\|_2^2} \\
        &=  \frac{1}{n}\sum_{i=1}^n{\|nq_{i}\pow{0} \nabla r_{i}(w\pow 0) - nq_{i}^*\nabla r_{i}(w^\star)\|_2^2} \\
        &\leq \frac{2}{n} \sum_{i=1}^n \|nq_{i}\pow{0} \nabla (r_{i}(w\pow 0) - \nabla r_{i}(w^\star))\|_2^2 + \frac{2}{n} \sum_{i=1}^n \|n(q_{i}\pow{0} - q_i^\star) \nabla r_{i}(w^\star)\|_2^2\\
        &\leq 2n \sum_{i=1}^n (q_{i}\pow{0})^2 M^2\|w\pow 0 - w^\star\|_2^2 + 8nG^2 \|q\pow{0} - q^\star\|_2^2\\
        &\leq \sbr{2n \norm{\sigma}_2^2 M^2 + \frac{8n^2G^4}{\bar{\nu}^2}} \|w\pow 0 - w^\star\|_2^2 \\
        &\leq \sbr{2n\norm{\sigma}_2^2 M^2 + \mu^2 / 8} \|w\pow 0 - w^\star\|_2^2 \leq 3nM^2\|w\pow 0 - w^\star\|_2^2.
    \end{align*}
    This means ultimately that
    \begin{align*}
        c_1 S \pow 0 \leq \frac{n^2}{16(1+\kappa^{-1})^2\kappa_\sigma^2} \|w\pow 0 - w^\star\|_2^2.
    \end{align*}
    Next, we have
    \begin{align*}
        c_2 T\pow{0} = \frac{n}{48(\kappa+1)\kappa_\sigma} \norm{w\pow{0} - w^\star}_2^2.
    \end{align*}
    Thus, we can write
    \begin{align*}
        V\pow 0 &\leq \sbr{1 + \frac{n^2}{16(1+\kappa^{-1})^2\kappa_\sigma^2} + \frac{n}{48(\kappa+1)\kappa_\sigma}} \norm{w\pow{0} - w^\star}_2^2\\
        &\leq (1 + \sigma_n^{-1} + \sigma_n^{-2}) \norm{w\pow{0} - w^\star}_2^2,
    \end{align*}
    completing the proof.
\end{proof}

\subsubsection{Step 3b: Analyze \emph{small} shift cost setting.}

To describe the rate, define $\delta := nG^2/(\mu \bar{\nu})$. The quantity $\delta$ captures the effect of the primal regularizer $\mu$ and dual regularizer $\bar{\nu}$ as compared to the inherent continuity properties of the underlying losses. 
\begin{theorem}\label{thm:lyapunov2} 
    Assume that $n \geq 2$ and that the shift cost $\bar{\nu} \leq 8nG^2/\mu$. 
    The sequence of iterates produced by \Cref{algo:lsaga_conceptual} with 
    \begin{align*}
        \eta &= \frac{1}{16n\mu}\min\br{\frac{1}{6[8\delta + (\kappa+1)\kappa_\sigma]}, \frac{1}{4\delta^2\max\br{2n\kappa^2, \delta}}}
    \end{align*}
    we have 
    \[
    \E{t}{V\pow {t+1}} \leq (1-\tau^{-1})V\pow t,  
    \]
    \begin{align*}
        \expect_t \norm{w\pow{t} - w^\star}_2^2 \leq \p{5 + 16\delta + \frac{6\kappa^2}{\sigma_n}}\exp\p{-t/\tau} \norm{w\pow{0} - w^\star}_2^2
    \end{align*}
    for
    \begin{align*}
        \tau = 32n \max\br{6[8\delta + (\kappa+1)\kappa_\sigma], 4\delta^2\max\br{2n\kappa^2, \delta}, 1/16}.
    \end{align*}
    \label{thm:small}
\end{theorem}
\begin{proof}
    First, we apply \Cref{lem:main_term} with $q' = q^\star$, as well as \Cref{lem:Rt}, \Cref{lem:St_Tt}, and \Cref{lem:Ut}, set $c_4 = 1$, and consolidate all constants to write 
    \begin{align}
        \E{t}{V\pow {t+1}} - (1-\tau^{-1})V\pow t &\leq (\tau^{-1} - \eta \mu + c_2) \norm{w\pow{t} - w^\star}_2^2\label{line:main}\\
        &+\sbr{\tau^{-1} - \frac{1}{n} + \frac{2\beta_1 G^2}{\bar{\nu} (M+\mu)\kappa_\sigma}+  \p{1 - \frac{1}{n}}\frac{G^2c_3}{2\bar{\nu} \mu n}}\yellow{R\pow{t}}\label{line:Rt}\\
        &+\sbr{\tau^{-1} + \frac{1 + c_3}{c_1}\eta^2(1 + \beta_2^{-1}) - \frac{1}{n}}c_1\green{S \pow t}\\
        &+\sbr{\tau^{-1} + \frac{\eta G^2 n}{2c_2\bar{\nu}} \beta_3^{-1} + \frac{c_3\eta M^2}{c_2\mu n} \p{1 - \frac{1}{n}} - \frac{1}{n}}c_2\red{T \pow t}\\
        &+\sbr{\tau^{-1} + \frac{2\eta G^2n}{c_3\bar{\nu}} (1 + \beta_3) - \frac{1}{n}}c_3\blue{U \pow t}\\
        &+\sbr{-\frac{\eta\beta_1}{2(M + \mu) \kappa_\sigma} + \eta^2(1+c_3) (1 + \beta_2) + \frac{c_1}{n}}\purple{Q \pow t}. \label{line:Qt}
    \end{align}
    We first set $c_1 = \frac{n\eta\beta_1}{4(M + \mu) \kappa_\sigma}$ and $c_2 = \eta\mu/2$ to clean up~\eqref{line:main} and~\eqref{line:Qt}.
    We also drop the terms $(1 - 1/n) \leq 1$. Then, we notice in~\eqref{line:Rt} that to achieve
    \begin{align*}
        \frac{2\beta_1 G^2}{\bar{\nu} (M+\mu)\kappa_\sigma} \leq \frac{1}{4n},
    \end{align*}
    we need that $\beta_1 \leq ((M+\mu)\kappa_\sigma)/(8nG^2/\bar{\nu})$. Combined with the requirement that $\beta_1 \in [0, 1]$, 
    we set $\beta_1 = ((M+\mu)\kappa_\sigma)/(8nG^2/\bar{\nu} + (M+\mu)\kappa_\sigma)$. We set $\beta_2 = 2$, and can rewrite the expression above.
    \begin{align*}
        \E{t}{V\pow {t+1}} - (1-\tau^{-1})V\pow t &\leq \p{\tau^{-1} - \frac{\eta \mu}{2}} \norm{w\pow{t} - w^\star}_2^2\\
        &+\sbr{\tau^{-1} - \frac{3}{4n} + \frac{G^2c_3}{2\bar{\nu} \mu n}}\yellow{R\pow{t}}\\
        &+\sbr{\tau^{-1} + \frac{6(1 + c_3)(M+\mu)\kappa_\sigma}{n\beta_1}\eta - \frac{1}{n}}c_1\green{S \pow t}\\
        &+\sbr{\tau^{-1} + \frac{G^2 n}{\mu\bar{\nu}} \beta_3^{-1} + \frac{c_3 M^2}{\mu^2 n} - \frac{1}{n}}c_2\red{T \pow t}\\
        &+\sbr{\tau^{-1} + \frac{2\eta G^2n}{c_3\bar{\nu}} (1 + \beta_3) - \frac{1}{n}}c_3\blue{U \pow t}\\
        &+\sbr{-\frac{\eta\beta_1}{4(M + \mu) \kappa_\sigma} + 3\eta^2(1+c_3) }\purple{Q \pow t}.
    \end{align*}
    Next, set the learning rate to be 
    \begin{align}
        \eta \leq \frac{\beta_1}{12(1+c_3)(M+\mu)\kappa_\sigma}
        \label{eq:learning_rate}
    \end{align}
    to cancel out $\purple{Q \pow t}$ and achieve
    \begin{align*}
        \E{t}{V\pow {t+1}} - (1-\tau^{-1})V\pow t &\leq \p{\tau^{-1} - \frac{\eta \mu}{2}} \norm{w\pow{t} - w^\star}_2^2\\
        &+\sbr{\tau^{-1} - \frac{3}{4n} +  \frac{G^2c_3}{2\bar{\nu} \mu n}}\yellow{R\pow{t}}\\
        &+\sbr{\tau^{-1} - \frac{1}{2n}}c_1\green{S \pow t}\\
        &+\sbr{\tau^{-1} + \frac{G^2 n}{\mu\bar{\nu}} \beta_3^{-1} + \frac{c_3 M^2}{\mu^2 n}  - \frac{1}{n}}c_2\red{T \pow t}\\
        &+\sbr{\tau^{-1} + \frac{2\eta G^2n}{c_3\bar{\nu}} (1 + \beta_3) - \frac{1}{n}}c_3\blue{U \pow t}.
    \end{align*}
    Requiring now that $\tau \geq 2n$, we may also cancel the $\green{S \pow t}$ term. 
    We substitute $\delta = nG^2/(\mu \bar{\nu})$ to achieve
    \begin{align*}
        \E{t}{V\pow {t+1}} - (1-\tau^{-1})V\pow t &\leq \p{\tau^{-1} - \frac{\eta \mu}{2}} \norm{w\pow{t} - w^\star}_2^2\\
        &+\sbr{-\frac{1}{4n} + \frac{c_3\delta}{2n^2}}\yellow{R\pow{t}}\\
        &+\sbr{-\frac{1}{2n} + \frac{\delta}{\beta_3} + \frac{c_3 M^2}{\mu^2 n}}c_2\red{T \pow t}\\
        &+\sbr{-\frac{1}{2n} + \frac{2\mu\eta\delta}{c_3} (1 + \beta_3)}c_3\blue{U \pow t}.
    \end{align*}
    It remains to select $c_3$ and $\beta_3$. As such, we set $\beta_3 = 4n\delta$ and use that $1 + 4n\delta \leq 8n\delta$ when $n \geq 2$ and $\delta \geq 1/8$ as assumed, and so 
    \begin{align*}
        \E{t}{V\pow {t+1}} - (1-\tau^{-1})V\pow t &\leq \p{\tau^{-1} - \frac{\eta \mu}{2}} \norm{w\pow{t} - w^\star}_2^2\\
        &+\sbr{-\frac{1}{4n} +  \frac{c_3\delta}{2n^2}}\yellow{R\pow{t}}\\
        &+\sbr{-\frac{1}{4n} + \frac{c_3 \kappa^2}{n}}c_2\red{T \pow t}\\
        &+\sbr{-\frac{1}{2n} + \frac{16n\mu\eta\delta^2}{c_3}}c_3\blue{U \pow t}.
    \end{align*}
    We require now that
    \begin{align*}
        c_3 = \frac{1}{2}\min\br{\frac{1}{2\kappa^2}, \frac{n}{\delta}},
    \end{align*}
    which cancels $\red{T \pow t}$ and $\yellow{R\pow{t}}$, leaving
    \begin{align*}
        \E{t}{V\pow {t+1}} - (1-\tau^{-1})V\pow t &\leq \p{\tau^{-1} - \frac{\eta \mu}{2}} \norm{w\pow{t} - w^\star}_2^2\\
        &+\sbr{-\frac{1}{2n} + 32\mu \eta \delta^2\max\br{2n\kappa^2, \delta}}c_3\blue{U \pow t}.
    \end{align*}
    From the above, we retrieve the requirement that
    \begin{align}
        \eta \leq \frac{1}{64n\mu\delta^2\max\br{2n\kappa^2, \delta}}.
        \label{eq:lr_case2}
    \end{align}
    It now remains to set $\eta$. By substituting in the values for $\beta_1$ and $c_3$ into~\eqref{eq:learning_rate}, we have that
    \begin{align*}
        \eta \overset{\text{want}}{\leq} \frac{\beta_1}{12(1+c_3)(M+\mu)\kappa_\sigma} &= \frac{1}{12(1+c_3)[8\mu\delta + (M+\mu)\kappa_\sigma]}\\
        &\geq \frac{1}{(12+6n/\delta)[8\mu\delta + (M+\mu)\kappa_\sigma]}\\
        &\geq \frac{1}{(12+48n)[8\mu\delta + (M+\mu)\kappa_\sigma]}\\
        &\geq \frac{1}{96n[8\mu\delta + (M+\mu)\kappa_\sigma]}.
    \end{align*}
    The combination of~\eqref{eq:lr_case2} and the above display yields
    \begin{align*}
        \eta &= \min\br{\frac{1}{96n[8\mu\delta + (M+\mu)\kappa_\sigma]}, \frac{1}{64n\mu\delta^2\max\br{2n\kappa^2, \delta}}}\\
        &= \frac{1}{16n\mu}\min\br{\frac{1}{6[8\delta + (\kappa+1)\kappa_\sigma]}, \frac{1}{4\delta^2\max\br{2n\kappa^2, \delta}}}.
    \end{align*}
    We need finally that $\tau \geq 2/(\mu \eta)$, resulting in the requirement
    \begin{align*}
        \tau \geq 32n \max\br{6[8\delta + (\kappa+1)\kappa_\sigma], 4\delta^2\max\br{2n\kappa^2, \delta}}.
    \end{align*}
    This is achieved by setting
    \begin{align*}
        \tau = 32n \max\br{6[8\delta + (\kappa+1)\kappa_\sigma], 4\delta^2\max\br{2n\kappa^2, \delta}, 1/16}.
    \end{align*}
    completing the proof of the claim 
    \begin{align*}
        \E{t}{V\pow{t+1}} \leq (1 - \tau^{-1})V\pow{t}.
    \end{align*}
    Next, we bound the initial terms to achieve the final rate. First, we bound $\eta$ which is used in all of the terms. Because $\delta \geq 1/8$,
    \begin{align}
        \eta \leq\frac{1}{16n\mu} \cdot \frac{1}{4\delta^2\max\br{2n\kappa^2, \delta}} \leq \frac{1}{64n\mu \delta^3} \leq \frac{8}{n\mu}.
        \label{eq:lr_bound}
    \end{align}
    Then,
    \begin{align*}
        S \pow 0 &=  \frac{1}{n}\sum_{i=1}^n{\|n\rho_{i}\pow{0} \nabla r_{i}(z_{i}\pow 0) - nq_{i}^\star\nabla r_{i}(w^\star)\|_2^2} \\
        &=  \frac{1}{n}\sum_{i=1}^n{\|nq_{i}\pow{0} \nabla r_{i}(w\pow 0) - nq_{i}^*\nabla r_{i}(w^\star)\|_2^2} \\
        &\leq \frac{2}{n} \sum_{i=1}^n \|nq_{i}\pow{0} \nabla (r_{i}(w\pow 0) - \nabla r_{i}(w^\star))\|_2^2 + \frac{2}{n} \sum_{i=1}^n \|n(q_{i}\pow{0} - q_i^\star) \nabla r_{i}(w^\star)\|_2^2\\
        &\leq 2n \sum_{i=1}^n (q_{i}\pow{0})^2 M^2\|w\pow 0 - w^\star\|_2^2 + 8nG^2 \|q\pow{0} - q^\star\|_2^2\\
        &\leq \sbr{2n \norm{\sigma}_2^2 M^2 + \frac{8n^2G^2}{\bar{\nu}^2}} \|w\pow 0 - w^\star\|_2^2 \\
        &\leq \sbr{2n\norm{\sigma}_2^2 M^2 + \mu^2 / 8} \|w\pow 0 - w^\star\|_2^2 \leq 3nM^2\|w\pow 0 - w^\star\|_2^2.
    \end{align*}
    Continuing with $\beta_1 \leq 1$ and \eqref{eq:lr_bound},
    \begin{align*}
        c_1 S\pow{0} &= \frac{n\eta \beta_1}{4(M+\mu)\kappa_\sigma} S\pow{0}\\
        &\leq \frac{2}{\mu(M+\mu)\kappa_\sigma} \cdot 3nM^2\|w\pow 0 - w^\star\|_2^2\\
        &\leq \frac{6n\kappa^2}{(1+\kappa)\kappa_\sigma} \|w\pow 0 - w^\star\|_2^2\\
        &\leq \frac{6\kappa^2}{\sigma_n} \|w\pow 0 - w^\star\|_2^2.
    \end{align*}
    Next, we have $T\pow{0} = n \norm{w\pow{0} - w^\star}_2^2$ and by \eqref{eq:lr_bound},
    \begin{align*}
        c_2 T\pow{0} &= \frac{\eta \mu}{2} \cdot n \norm{w\pow{0} - w^\star}_2^2\\
        &\leq 4 \norm{w\pow{0} - w^\star}_2^2.
    \end{align*}
    Because $U\pow{0} = 0$, it is bounded trivially. For $R\pow{0}$, with $c_4 = 1$ we have
    \begin{align*}
        R\pow{0} &= 2n\eta (\q(\ell(w\pow{0})) - \q(\ell(w^\star)))^\top (\ell(w\pow{0}) - \ell(w^\star))\\
        &\leq \frac{2n\eta}{\bar{\nu}} \norm{\ell(w\pow{0}) - \ell(w^\star)}_2^2\\
        &\leq \frac{2n^2\eta G^2}{\bar{\nu}} \norm{w\pow{0} - w^\star}_2^2\\
        &\leq \frac{16n G^2}{\mu\bar{\nu}} \norm{w\pow{0} - w^\star}_2^2\\
        &= 16\delta \norm{w\pow{0} - w^\star}_2^2.
    \end{align*}
    Combining each of these terms together, we have that
    \begin{align*}
        V\pow{0} \leq \p{5 + 16\delta + \frac{6\kappa^2}{\sigma_n}} \norm{w\pow{0} - w^\star}_2^2,
    \end{align*}
    completing the proof.
\end{proof}

\subsection{Proof of Main Result}
The objective is once again 
\begin{align*}
    \primobj(w) &= \max_{q \in \Pcal(\sigma)} q^\top \ell(w) - \nu D_f(q \Vert \ones_n/n) + \frac{\mu}{2}\norm{w}_2^2\\
    &= \max_{q \in \Pcal(\sigma)} q^\top \ell(w) - n\alpha_n\nu \frac{1}{n\alpha_n}D_f(q \Vert \ones_n/n) + \frac{\mu}{2}\norm{w}_2^2\\
    &= \max_{q \in \Pcal(\sigma)} q^\top \ell(w) - n\alpha_n\nu \Omega(q) + \frac{\mu}{2}\norm{w}_2^2\\
    &= \max_{q \in \Pcal(\sigma)} q^\top \ell(w) - \bar{\nu} \Omega(q) + \frac{\mu}{2}\norm{w}_2^2,
\end{align*}
where $\Omega(q) = D_f(q \Vert \ones_n/n) / n\alpha_n$ is the penalty scaled to be $1$-strongly convex and we simply notate $\bar{\nu} = n\alpha_n\nu$. The previous subsections give the convergence analysis in the cases of large and small values of $\bar{\nu}$.
They are combined below.
\label{sec:a:main_result}
\lsagamain*
\begin{proof}
    Combine \Cref{thm:large} (the analysis for $\bar{\nu} \geq 8nG^2/\mu$) and \Cref{thm:small} (the analysis for $\bar{\nu} \leq 8nG^2/\mu$) to achieve convergence for any value of $\bar{\nu}$. Substitute $\bar{\nu} = n\alpha_n \nu$ so that the condition $\bar{\nu} \geq 8nG^2/\mu$ reads as $\nu \geq G^2/(\mu \alpha_n)$. 
\end{proof}

\clearpage

\section{Improving \lsaga with Moreau Envelopes}\label{sec:a:prox_saga}
As mentioned in \Cref{sec:lsaga_algo}, we may want to generalize \lsaga to non-smooth settings which arise either when the shift cost $\nu = 0$ or when the underlying losses $(\ell_i)$ are non-smooth. The former case is already addressed by \Cref{prop:no-shift-appendix} in \Cref{sec:a:objective}. The latter case can be handled by considering a variant of \lsaga applied to the \emph{Moreau envelope} of the losses, as defined below. Not only does this extend the algorithm to non-smooth losses, it also allows even in the smooth setting for a less stringent lower bound on $\nu$ required for the $O((n + \kappa_\sigma \kappa)\ln(1/\epsilon))$ rate. The rest of this section contains necessary background, implementation details, and the adjustments to the convergence analysis.

\paragraph{Notation}
The Moreau envelope and the proximal (prox) operator of a convex function $f: \reals^d \to \reals$ are respectively defined for a constant $\eta > 0$ as 
\begin{align}
    \Mcal_\eta[f](w) &= \min_{z \in \reals^d} \left\{ f(z) + \frac{1}{2\eta} \norm{w-z}_2^2  \right\}\,, \\
    \prox_{\eta f}(w) &= \argmin_{z \in \reals^d} \left\{ f(z) + \frac{1}{2\eta} \norm{w-z}_2^2  \right\} \,.
\end{align}
A fundamental property is that the gradient of the Moreau envelope is related to the prox operator:
\begin{align}
    \grad \Mcal_{\eta}[f](w) = \frac{1}{\eta}(w - \prox_{\eta f}(w)) \,.
\end{align}
For simplicity, we denote $\bar \nu = 2 n \nu$.

\paragraph{Algorithm Description}
The algorithm is nearly equivalent to \Cref{algo:lsaga_conceptual}, but makes the following changes. We sample $i_t \sim q\pow{t}$ non-uniformly in the sense that $\P{}{i_t = i} = q\pow{t}_i$. We do not store an additional vector of weights $\rho\pow{t}$, and use $q\pow{t}$ in all associated steps. This does not change the expectation of the update direction or the control variate, but creates minor changes in the analysis of the variance term. In the iterate update, we replace the gradient descent-like update with
\begin{align*}
    u\pow t &= w\pow{t} + \eta(g_{i_t} \pow t - \bar g \pow t)\\
    w \pow {t+1} &= \prox_{\eta r_{i_t}}(u\pow{t}).
\end{align*}
The vector $u\pow{t}$ adds the control variate to $w\pow{t}$ before passing it to the proximal operator. The second change is that the elements of the gradient table are updated using $j_t$ and the gradients of the Moreau envelope. That is,
\begin{align*}
    g_{j_t} \pow {t+1} &= \grad \Mcal_\eta[r_{j_t}]\big(w\pow{t} + \eta(g_{j_t}\pow{t} - \bar g\pow{t})\big),\\
    g_{j} \pow {t+1} &= g_j \pow {t} \text{ for }j \neq j_t.
\end{align*}
Plugging these changes into \Cref{algo:lsaga_conceptual} produces the \lsagaprox variant.

\paragraph{Implementation Details} \label{sec:a:prox-impl}
The proximal operators can be computed in closed form or algorithmically for common losses.
We list here the implementations for some losses of interest. The proximal operators for the binary or multiclass logistic losses cannot be obtained in closed form, we approximate them by one Newton step.

{\it Squared loss.}
For the squared loss, defined as $\ell(w) = \frac{1}{2}(w\T x - y)^2$ for $x \in \reals^d, y \in \reals$, then 
\[
    \prox_{\eta \ell}(w) = w - \frac{\eta x}{1 + \eta \norm{x}^2}\left(x\T  w - y\right) \,.
\]

{\it Binary logistic loss.}
For the binary logistic loss defined for $x \in \R^d$, $y \in \{0, 1\}$, $w \in \R^d$ as
$  \ell(w) = -y \ln(\sigma(x^\top w)) - (1-y)\ln(1-\sigma(x^\top w)) 
  = - y x^\top w + \ln(1+ e^{x^\top w}),
$
we approximate the proximal operator by one Newton step, whose formulation reduces to
\[
\prox_{\eta \ell}(w) \approx w - \frac{\eta g }{1 + \eta q\|x\|_2^2} x
\]

{\it Multinomial logistic loss.}
For the multinomial logistic loss of a linear model
defined by $W$ on a sample $(x, y)$ as
$
\ell(W) = - y^\top W x + \ln(\exp(Wx)^\top \ones).
$
for $x \in \R^d$, $y \in \{0, 1\}^k$, $y^\top \ones = 1$,
$W \in \R^{k \times d}$, we consider approximating the proximal operator by one Newton-step,  whose formulation reduces to
\begin{align*}
    \prox_{\eta \ell}(W) & \approx W - \eta z^* x^\top \\
    z^* & = z_1 - \lambda^* z_2, \\
    z_1 &  = - y\oslash z_3 + z_2, \
    z_2 = \sigma(Wx) \oslash z_3, \
    z_3 = (\ones + \eta \|x\|_2^2 \sigma(Wx)), \
    \lambda^* = \frac{z_1^\top \ones}{z_2^\top \ones}. 
\end{align*}

{\it Regularized losses.}
For a convex $\ell: \reals^d \to \reals$, define $r(w) = \ell(w) + (\mu/2)\norm{w}^2$.
Then, we have,
\[
    \prox_{\eta r}(w) = \prox_{\frac{\eta \ell}{1+\eta \mu}}\left(\frac{w}{1 + \eta \mu}\right) \,.
\]

\subsection{Convergence Analysis}
\lsagaprox satisfies the following convergence bound. Recall that $\gamma_\star = \norm{\grad \ell(w^\star)}_2$.

\begin{theorem} \label{thm:acc:lsaga}
    Suppose the smoothing parameter $\bar \nu$ is set large enough as 
    \[
        \bar \nu \ge \frac{\gamma_* G}{M} \min \left\{ \sqrt{\frac{2n\kappa}{4\kappa_\sigma^* - 1}}, 2 \kappa \right\} \,,
    \]
    and define a constant
    \[
        \tau = 2 + \max\{2(n-1),  \,\, \kappa(4 \kappa_\sigma^* - 1) \} \,,   
    \]
    for $\kappa_\sigma^* = \sigma_n/\sigma_1$.
    Then, the sequence of iterates $(w\pow{t})$ generated by \lsagaprox
    with learning rate $\eta = M^{-1} \min\left\{1/(4 \kappa_\sigma^* - 1), \kappa/(n-1)   \right\}$ satisfies 
    \[
    \expect\norm{w\pow{t} - w^\star}^2_2
    \le (n + 3/2) \exp(-t/\tau) \norm{w\pow{0} - w^\star}^2_2 \,.
    \]
\end{theorem}

We now prove \Cref{thm:acc:lsaga}.

\myparagraph{Notation for the Proof}
We denote $\expect_t[\cdot]$ denote the expectation conditioned on the randomness until time $t$;
more precisely, on the sigma-algebra generated by $w\pow{t}$.
Further, we define $w_i^* = w^* + \eta \grad r_i(w^*)$. By analyzing the first-order conditions of the prox,
it is easy to see that 
\begin{align}
    \prox_{\eta r_i}(w_i^*) = w^* \,.
\end{align}

We will use the Lyapunov function
\begin{align}
    V\pow{t} = \norm{w\pow{t} - w^*}^2  + c_1 \sum_{i=1}^n \norm{z_i\pow{t} - w^*}^2 
        + \frac{c_2}{M^2} \sum_{i=1}^n \norm{g_i\pow{t} - \grad r_i(w^*)}^2 \,.
\end{align}

The first step is to analyze the effect of the update on $w\pow{t}$ as the first term of the Lyapunov function.
\begin{proposition} \label{prop:acc-lsaga:one-step}
    The iterates of \lsagaprox satisfy
    \begin{align*}
       (1+\mu \eta) & \expect_t \norm{w\pow{t+1} - w^*}^2
        \le \,
        \norm{w\pow{t} - w^*}^2
        + 2\eta^2 \sigma_n \sum_{i=1}^n \norm{g_{i}\pow{t} - \grad r_{i}(w^*)}^2  \\
        &\, 
        + \frac{2 \eta^2 \gamma_*^2 G^2}{\bar \nu^2}  \sum_{i=1}^n \norm{z_i\pow{t} - w^*}^2  \\
        & \, - \eta^2 \left( 1 + \frac{1}{M \eta} \right) \sigma_1 
         \sum_{i=1}^n \norm{\grad \Mcal_\eta[r_{i}]\big(w\pow{t} - \eta(g_i\pow{t} - \bar g\pow{t}) \big) - \grad r_{i}(w^*)}^2 \,.
    \end{align*}
\end{proposition}
\begin{proof}
We use the co-coercivity of the prox operator (\Cref{thm:prox:cocoercive}) to get 
\begin{align} \label{eq:lsaga-prox:main}
\begin{aligned}
    (1 + \mu\eta) \, \expect_t\norm{w\pow{t+1} - w^*}^2 
    &= (1 + \mu \eta) \, \expect_t \norm{\prox_{\eta r_{i_t}}(u\pow{t}) - \prox_{\eta r_{i_t}}(w_{i_t}^*)}^2 \\
    &\le \expect_t\inp{u\pow{t} - w_{i_t}^*}{\prox_{\eta r_{i_t}}(u\pow{t}) - \prox_{\eta r_{i_t}}(w_{i_t}^*)} \\
    &= \expect_t\inp{u\pow{t} - w_{i_t}^*}{w\pow{t+1} - w^*} \\
    &= \underbrace{\expect_t\inp{u\pow{t} - w_{i_t}\pow{t}}{w\pow{t} - w^*}}_{=: \Tcal_1}
        + \underbrace{\expect_t\inp{u\pow{t} - w_{i_t}^*}{w\pow{t+1} - w\pow{t}}}_{=: \Tcal_2} \,,
\end{aligned}
\end{align}
where we added and subtracted $w\pow{t}$ in the last step.

For the first term, we observe that $\expect_t[u\pow{t}] = w\pow{t}$ and 
$\expect_t[w_{i_t}^*] = w^* + \eta \, \expect_t [ \grad r_{i_t}(w^*)]$ so that 
\begin{align}
    \Tcal_1 &= \inp*{\expect_t[u\pow{t} - w_{i_t}^*]}{w\pow{t} - w^*} 
        = \norm{w\pow{t} - w^*}^2 + \eta \inp*{\expect_t[\grad r_{i_t}(w^*)]}{w\pow{t} - w^*} \,.
\label{eq:lsaga-prox:t1}
\end{align}
For $\Tcal_2$, note that 
\[
    w\pow{t+1} - w\pow{t} = -\eta\left( \grad \Mcal_\eta[r_{i_t}](u\pow{t}) - g_{i_t}\pow{t} + \bar g\pow{t}  \right) \,.
\]
We manipulate $\Tcal_2$ to set ourselves up to apply co-coercivity of prox-gradient by adding and subtracting $\grad \Mcal_\eta[r_{i_t}](w_{i_t}^*)$ as follows:
\begin{align*}
    \Tcal_2 =& \, 
    -\eta \, \expect_t\inp{u\pow{t} - w_{i_t}\pow{t}}{\grad \Mcal_\eta[r_{i_t}](u\pow{t}) - g_{i_t}\pow{t} + \bar g\pow{t}} \\
    =& 
    \underbrace{-\eta \, \expect_t \inp{u\pow{t} - w_{i_t}^*}{\grad \Mcal_\eta[r_{i_t}](u\pow{t})  - \grad \Mcal_\eta[r_{i_t}](w_{i_t}^*)}}_{=: \Tcal_2'}
    \\ & \quad
    \underbrace{-\eta\, \expect_t \inp{u\pow{t} - w_{i_t}^*}{\grad \Mcal_\eta[r_{i_t}](w_{i_t}^*) - g_{i_t}\pow{t} + \bar g\pow{t}}}_{ =: \Tcal_2''} \,.
\end{align*}
Now, co-coercivity of the prox-gradient (\Cref{thm:prox-grad:cocoercive}) of the $M$-smooth function $r_{i_t}$ gives
\begin{align} 
\label{eq:lsaga-prox:t2p}
    \Tcal_2' \le -\eta^2\left( 1 + \frac{1}{M \eta}\right) 
        \expect_t \norm{\grad \Mcal_\eta[r_{i_t}](u\pow{t})  - \grad \Mcal_\eta[r_{i_t}](w_{i_t}^*)}^2 \,.
\end{align}
Next, we use $u\pow{t} = w\pow{t} + \eta(g_{i_t}\pow{t} - \bar g\pow{t})$,
and $w_i^* = w^* + \eta \grad r_i(w^*)$ and 
$\grad \Mcal_\eta[r_i](w_i^*) = \grad r_i(w^*)$ to get 
\begin{align*}
    \Tcal_2'' &=
    -\eta \, \expect_t \inp{w\pow{t} - w^* - \eta(\grad r_{i_t}(w^*) - g_{i_t}\pow{t} + \bar g\pow{t})}{\grad r_{i_t}(w^*) - g_{i_t}\pow{t} + \bar g\pow{t}} \\
    &= -\eta \, \inp{w\pow{t} - w^*}{\expect_t[\grad r_{i_t}(w^*)]} 
        + \eta^2 \, \expect_t \norm{g_{i_t}\pow{t} - \bar g\pow{t} - \grad r_{i_t}(w^*)}^2 \,,
\end{align*}
where we used that $\expect_t[g_{i_t}\pow{t}] = \bar g\pow{t}$.
Next, we use $\norm{x+y}^2 \le 2\norm{x}^2 + 2 \norm{y}^2$ for any vectors $x, y$ and
$\expect\|X - \expect [X]\|^2 \le \expect\|X\|^2$ for any random vector $X$ to get 
\begin{align}
\label{eq:lsaga-prox:t2pp}
    \Tcal_2'' &\le 
    -\eta \, \inp{w\pow{t} - w^*}{\expect_t[\grad r_{i_t}(w^*)]} 
    + 2\eta^2 \, \expect_t \norm{g_{i_t}\pow{t} - \grad r_{i_t}(w^*)}^2 
    + 2 \eta^2 \, \norm{\expect_t[\grad r_{i_t}(w^*)]}^2 \,.
\end{align}
Plugging \eqref{eq:lsaga-prox:t2pp}, \eqref{eq:lsaga-prox:t2p}, and \eqref{eq:lsaga-prox:t2pp}
into \eqref{eq:lsaga-prox:main} gives us
\begin{align}
\label{eq:lsaga-prox:main2}
\begin{aligned}
        (1 + \mu \eta) \expect_t \norm{w\pow{t+1} - w^*}^2
        \le& \,
        \norm{w\pow{t} - w^*}^2 
        + 2\eta^2 \, \expect_t\norm{g_{i_t}\pow{t} - \grad r_{i_t}(w^*)}^2
        + 2 \eta^2 \, \norm{\expect_t \left[\grad r_{i_t}(w^*)\right]}^2 \\
        &\, -\eta^2 \left( 1 + \frac{1}{M \eta} \right) \expect_t\norm{
        \grad \Mcal_\eta[r_{i_t}](u\pow{t}) - \grad r_{i_t}(w^*)}^2 \,.
\end{aligned}
\end{align}
Next, we note that $\Pcal(\sigma) \subset [\sigma_1, \sigma_n]^n$ to get,
\begin{gather*}
    \expect_t \norm{g_{i_t} - \grad r_{i_t}(w^*)}^2
    = \sum_{i=1}^n q_i\pow{t} \norm{g_{i} - \grad r_{i}(w^*)}^2 
    \le \sigma_n \sum_{i=1}^n \norm{g_{i} - \grad r_{i}(w^*)}^2  \,,
    \quad\text{and}\\
    \begin{aligned}
    \expect_t\norm{\grad \Mcal_\eta[r_{i_t}](u\pow{t}) - \grad r_{i_t}(w^*)}^2
    &= \sum_{i=1}^n q_i\pow{t} \norm{\grad \Mcal_\eta[r_{i}]\big(w\pow{t} - \eta(g_i\pow{t} - \bar g\pow{t}) \big) - \grad r_{i}(w^*)}^2  \\
    &\ge \sigma_1 \sum_{i=1}^n \norm{\grad \Mcal_\eta[r_{i}]\big(w\pow{t} - \eta(g_i\pow{t} - \bar g\pow{t}) \big) - \grad r_{i}(w^*)}^2 \,.
    \end{aligned}
\end{gather*}
Moreover, we also have that 
\begin{align*}
    \norm{\expect_t[\grad r_{i_t}(w^*)]}^2 &= \norm{\grad \ell(w^\star)^\top(\q(l\pow{t}) - \q(\ell(w^\star)))]}^2\\
    & \gamma^2_* \norm{\q(l\pow{t}) - \q(\ell(w^\star)))}_2^2\\
    &\le
    \frac{\gamma_*^2 G^2}{\bar \nu^2} \sum_{i=1}^n \norm{z_i\pow{t} - w^*}^2 \,.
\end{align*}
Plugging these back into \eqref{eq:lsaga-prox:main2} completes the proof.
\end{proof}

Next, we analyze the other two terms of the Lyapunov function. The proof is trivial, so  we omit it.
\begin{proposition} \label{prop:acc:lsaga:prop2}
    We have,
    \begin{align*}
        \expect_t\left[ \sum_{i=1}^n \norm{z_i\pow{t+1} - w^*}^2 \right] =&\, 
        (1 - n^{-1}) \sum_{i=1}^n \norm{z_i\pow{t} - w^*}^2 + \norm{w\pow{t} - w^*}^2 \,, \\
        \expect_t\left[ \sum_{i=1}^n \norm{g_i\pow{t+1} - \grad r_i(w^*)}^2 \right] = & \,
        (1 - n^{-1}) \sum_{i=1}^n \norm{g_i\pow{t} - \grad r_i(w^*)}^2 \\
        &\,+ \frac{1}{n} \sum_{i=1}^n \norm{\grad \Mcal_\eta[r_{i}]\big(w\pow{t} - \eta(g_i\pow{t} - \bar g\pow{t}) \big) - \grad r_{i}(w^*)}^2 \,.
    \end{align*}
\end{proposition}

We are now ready to prove \Cref{thm:acc:lsaga}.

\begin{proof}[Proof of \Cref{thm:acc:lsaga}]
    Let $\tau > 1$ be a constant to be determined later
    and let $\Gamma := \gamma_*^2 G^2 / (M^2 \bar \nu^2)$ denote the effect of the smoothing.
    Combining \Cref{prop:acc-lsaga:one-step,prop:acc:lsaga:prop2}, we can write
    \begin{align} \label{eq:acc-saga-lya-dec}
    \begin{aligned}
        \expect_t[V\pow{t}]  &- (1-\tau^{-1}) V\pow{t} 
        \le 
        -\norm{w\pow{t} - w^*}^2 \left( \frac{\mu\eta}{1+\mu \eta} - c_1 - \tau^{-1}\right) \\
        &\, 
        - \sigma_1 \sum_{i=1}^n \norm{\grad \Mcal_\eta[r_{i}]\big(w\pow{t} - \eta(g_i\pow{t} - \bar g\pow{t}) \big) - \grad r_{i}(w^*)}^2
            \left( \frac{\eta^2(1 + (M\eta)^{-1})}{1+\mu\eta} - \frac{c_2}{n \sigma_1 M^2} \right) \\
        &\,- \sum_{i=1}^n \norm{z_i\pow{t} - w^*}^2 \left( c_1(n^{-1} - \tau^{-1}) - \frac{2 \eta^2 \gamma_*^2 G^2}{(1 + \mu\eta) \bar \nu^2} \right) \\
        &\, -\sum_{i=1}^n \norm{g_i\pow{t} - \grad r_i(w^*)}^2 \left(\frac{c_2}{M^2}(n^{-1} - \tau^{-1}) - \frac{2 \eta^2 \sigma_n}{1+ \mu\eta}\right) \,.
    \end{aligned}
    \end{align}
Let $\eta = b / M$.
Our goal is to set the constants $b, c_1, c_2, \tau > 0$ so that the right side above is non-positive and $\tau$ is as small as possible.
We will require $\tau \ge 2n$ so that $n^{-1} - \tau^{-1} \ge (2n)^{-1}$. Thus, we can have the right side nonpositive with
\begin{subequations}
\begin{gather}
\label{eq:coef:1}
    \frac{b}{b+\kappa} - c_1 - \tau^{-1} \ge 0 \\
\label{eq:coef:2}
    b(b+1) \ge \frac{c_2}{n \sigma_1} \left( 1 + \frac{b}{\kappa}\right) \\
\label{eq:coef:3}
    \frac{c_1}{2n} - \frac{2 b^2 \Gamma}{1 + b/\kappa} \ge 0 \\
\label{eq:coef:4}
    \frac{c_2}{2n} - \frac{2b^2 \sigma_n}{1 + b / \kappa} \ge 0 \,.
\end{gather}
\end{subequations}

Let us set $c_1 = \tau^{-1}$. By setting $c_2 = 4 \kappa n\sigma_n b^2 / (b+\kappa)$, we ensure that 
\eqref{eq:coef:4} is satisfied.
Next, we satisfy \eqref{eq:coef:1} with
\[
     \frac{b}{b+\kappa} = 2\tau^{-1} \quad\iff \quad
     b = \frac{2\kappa}{\tau-2} \,.
\]
Now, \eqref{eq:coef:2} is an inequality only in $\tau$. It is satisfied with
\[
    \tau \ge \tau_* := 2 + 2\kappa(4 \kappa_\sigma^* - 1) \,.
\]
This lets us fix $\tau = \max\{2n, \tau_*\}$ throughout, which leads to the value of $\eta$ as claimed in the theorem statement.
Finally, \eqref{eq:coef:3} requires 
\[
    \frac{4 n \kappa^2 \Gamma}{\tau-2} \le 1 
    \quad \iff \quad 
    \bar \nu \ge \frac{\sqrt{n} \kappa \gamma_* G}{M} \, \min \left\{ \sqrt{\frac{2}{\kappa(4\kappa_\sigma^* -1)}}, \, \frac{2}{\sqrt{n}} \right\} \,.
\]
Thus, under these conditions, the right-hand side of \eqref{eq:acc-saga-lya-dec} is non-negative. 
Iterating \eqref{eq:acc-saga-lya-dec} over $t$ updates, we get 
\[
    \expect[V\pow{t}] = (1- \tau^{-1})^t V\pow{0} \le \exp(-t / \tau) V\pow{0} \,.
\]
To complete the proof, we note that 
$c_1 \le 1/(2n)$ and 
\[
    c_2 = \frac{4\kappa n\sigma_n b^2}{b+\kappa}
    = 8 \frac{\kappa \kappa_\sigma}{\tau} b 
    \le 8 \frac{\kappa\kappa_\sigma}{\kappa(4 \kappa_\sigma^* -1)} \, \frac{1}{\kappa_\sigma^* - 1}
    \le \frac{8}{9} \,.
\]
This lets us use the fact that $\grad r_i$ is $M$-Lipschitz to bound 
\begin{align*}
    V\pow{0} &= \norm{w\pow{0} - w^*} + c_1 \sum_{i=1}^n \norm{w\pow{0} - w^*}^2
        + \frac{c_2}{M^2} \sum_{i=1}^n \norm{\grad r_i(w\pow{0}) - \grad r_i(w^*)}^2 
        \\ &\le (n + 3/2) \norm{w\pow{0} - w^*}^2 \,.
\end{align*}
\end{proof}

\clearpage

\section{Improving the Direct Saddle-Point Baseline}\label{sec:a:saddle_saga}
In \Cref{sec:experiments} we compared against the existing method of \citet{palaniappan2016stochastic}, which views our objective \eqref{eq:lsaga_obj} in its min-max form directly and applies variance reduction techniques to both the primal and dual sequences. In order to make the comparison more convincing, we also improve this method both theoretically and empirically by utilizing different learning rates for the primal and dual sequences. In this section, we provide a new convergence analysis for this improved two-hyperparameter variant, which we dub \emph{SaddleSAGA}, under the $\chi^2$-divergence penalty. 

\myparagraph{Notation}
For simplicity, we denote $\bar \nu  = 2 n  \nu$, 
and consider directly the min-max problem
\begin{equation}\label{eq:new_algo_pb}
\min_{w \in \R^d} \max_{q \in \mathcal{P}(\sigma)} \left[
\Psi(w, q) := 
q^\top \ell(w) + \frac{\mu}{2}\|w\|_2^2 - \frac{\bar \nu}{2}\|q- \ones_n/n\|_2^2
\right]
.
\end{equation}
Note that the function $\Psi$ is strongly convex in its first argument and strongly concave in its second argument.
A pair $(w^\star, q^\star)$ is called a saddle point of the convex-concave function $\Psi$ if
\[
    \max_{q \in \Pcal(\sigma)} \Psi(w^\star, q)
    \le \Psi(w^\star, q^\star)
    \le 
    \min_{w \in \reals^d} \Psi(w, q^\star) \,.
\]
In our setting, we can verify that the pair $w^\star = \argmin \primobj$ and $q^\star = \q(\ell(w^\star))$ is the unique saddle point of $\Psi$.

\myparagraph{Algorithm}
The algorithm makes use of proximal operators (as described in \Cref{sec:a:prox_saga} in addition), which is defined for a convex function $f:\R^d \rightarrow \R$, and $x\in \R^d$ as 
\[
    \prox_{f}(x) = \argmin_{y\in \R^d} \ f(y) + \frac{1}{2}\|x-y\|_2^2.
\]
The method is nearly equivalent to \Cref{algo:lsaga_conceptual}, but applies the update
\begin{align*}
    v\pow t &= n q_{i_t} \pow t \nabla \ell_{i_t}(w\pow t) 
        - (n \rho_{i_t} \pow t g_{i_t}\pow t) - \bar g \pow t)\\
    w\pow{t+1} &=  \prox_{\eta \mu \|\cdot\|_2^2} (w\pow t - \eta v \pow t)
\end{align*}
to the primal iterates and
\begin{align*}
    \pi \pow t &= n \ell_{i_t}(w\pow t) e_{i_t} -(nl_{i_t}\pow t e_{i_t} - l \pow t)\\
    q \pow {t+1} &= \prox_{\iota_{\mathcal{P}(\sigma)} + \delta \bar \nu \|\cdot - \ones_n/n\|_2^2/2}(q\pow t -\delta \pi\pow t)
\end{align*}
to the dual iterates, where $\delta > 0$ is the dual learning rate. The vector $\pi\pow{t}$ plays the role of an update direction, and the proximal update on $q\pow{t+1}$ can be solved with the PAV algorithm, as seen in \Cref{sec:a:efficient}.
Overall, the time and space complexity is identical to that of \lsaga. 

\myparagraph{Rate of Convergence}
We prove the following rate of convergence for SaddleSAGA.
\begin{theorem}\label{thm:cvg_lsaga_v2}
The iterates $(w\pow{t}, q\pow{t})$ of
\saddlesaga with learning rates 
\[
\eta  = \min\left\{\frac{1}{\mu}, \frac{1}{6(L\kappa_\sigma + 2G^2 n/\bar \nu)} \right\}, \quad 
\delta = \min\left\{\frac{1}{\bar \nu}, \frac{\mu}{8n^2 G^2}\right\}
\]
converge linearly to the saddle point of \eqref{eq:new_algo_pb}. In particular, for non-trivial regularization $\mu\bar \nu \leq 8n^2 G^2$ and $\mu \leq  6(L\kappa_\sigma + 2G^2 n/\bar \nu)$, the number of iterations $t$ to get $ \|w\pow{t} - w^\star\|_2^2 + c \|q\pow{t} - q^\star\|_2^2  \le \varepsilon$ (for some constant $c$) is at most
\[
O\left(\left(n + \kappa \kappa_\sigma + 
 \frac{n^2 G^2}{\mu \bar \nu}
\right)
\ln\frac1\varepsilon
\right) \,.
\]
\end{theorem}
The proof of this statement is given as \cref{cor:rate} later in this section.
To compare to the original variant, the rate obtained by \cite{palaniappan2016stochastic} in terms of our problem's constants is 
\[
    O\left(\left(n + \frac{n G^2}{\mu \bar \nu} + n \kappa^2 \right) \ln \frac1\varepsilon\right) \,.
\]
Compared to this, the rate we prove for SaddleSAGA improves $\kappa^2$ to $\kappa\kappa_\sigma$ while suffering an additional factor of $n$ in the $n^2 G^2 / (\mu \bar \nu)$ term. Empirical comparisons between SaddleSAGA and the original algorithm are given in \Cref{sec:a:additional}. As compared to \lsaga, SaddleSAGA matches the rate of \Cref{thm:lsaga:main} only when the shift cost $\bar \nu$ is large enough.

\subsection{Convergence Analysis}

In the following, we denote by $\E{t}{\cdot}$ the expecation of a quantity  according to the randomness of $i_t$ conditioned on $w \pow t, q \pow t$. Throughout the proof, we consider that the losses are $L$-smooth and $G$-Lipschitz continuous. 

\paragraph{Evolution of the distances to the optimum}
We start by using the contraction properties of the proximal operator to bound the evolution of the distances to the saddle point $(w^\star, q^\star)$. 
\begin{lemma}\label{lem:descent}
We have,
\begin{align*}
    \E{t}{\|w\pow {t+1} - w^\star\|_2^2} & \leq 
    \frac{1}{(1+\eta \mu)^2}
    \Big(
    \|w\pow t - w^\star\|_2^2  \\
    & \hspace{65pt} 
    - 2\eta ( \nabla ({q \pow t}^\top \ell)(w\pow t) -  \nabla ({q ^*}^\top \ell)(w^\star))^\top (w\pow t - w^\star)\\
    & \hspace{65pt}
    + \eta^2 \E{t}{\|v\pow t -  \nabla ({q ^*}^\top \ell)(w^\star) \|_2^2}
    \Big) \\
    \E{t}{\|q \pow {t+1} - q^\star\|_2^2} & \leq 
    \frac{1}{(1+ \delta \bar \nu)^2} 
    \Big(
    \|q \pow t - q^\star\|_2^2 \\
    & \hspace{65pt}
    + 2\delta (\ell(w\pow t) -\ell(w^\star))^\top (q \pow t - q^\star)  \\
    & \hspace{65pt}
    + \delta^2 \E{t}{\|\pi \pow t - \ell(w^\star)\|_2^2}
    \Big).
\end{align*}
\end{lemma}
\begin{proof}
By considering the first-order optimality conditions of the problem one verifies that $w^\star, q^\star$ satisfy for any $\eta, \delta$,
\begin{align*}
    w^\star = \prox_{\eta \mu \|\cdot\|_2^2/2}(w^\star - \eta \nabla ({q^\star}^\top \ell)(w^\star)), \quad
    q^\star = \prox_{\iota_{\mathcal{P}(\sigma)} + \delta \bar \nu \|\cdot - \ones_n/n\|_2^2/2}(q^\star + \delta \ell(w^\star)).
\end{align*}
Recall that the proximal operator of a $c$-strongly convex function $h$ is contractive such that $\|\prox_h(z) - \prox_h(z')\|_2 \leq \frac{1}{1+c}\|z-z'\|_2$. In our case, it means that 
\begin{align*}
    \|w\pow {t+1} - w^\star\|_2 & \leq 
    \frac{1}{1 + \eta \mu}
    \|w\pow t - \eta v\pow t - (w^\star - \eta \nabla ({q^\star}^\top \ell)(w^\star))\|_2, \\
    \|q \pow {t+1} - q^\star\|_2 & \leq 
    \frac{1}{1 + \delta \bar \nu} 
    \|q \pow t + \delta \pi \pow t - (q^\star + \delta \ell(w^\star))\|_2.
\end{align*}
By taking the squared norm, the expectation, expanding the squared norms and using that $\E{t}{v\pow t} = \nabla ({q \pow t}^\top \ell)(w\pow t)$, $\E{t}{\pi \pow t} = \ell(w \pow t)$, we get the result.
\end{proof}

\paragraph{Variance term evolutions}
We consider the evolution of the additional variance term added to the dual variables. 
\begin{lemma} \label{lem:noise_bound_saddle}
We have for any $\beta_2 >0$,
\begin{align*}
    \E{t}{\|\pi \pow t - \ell(w^\star)\|_2^2} & \leq
    (n + (n-1)\beta_2) nG^2 \|w\pow t - w^\star\|_2^2 \\
    & \quad + (n-1)(1+ \beta_2^{-1})\|\ell(w^\star) - l \pow t\|_2^2.
\end{align*}
\end{lemma}
\begin{proof}
As in the proof of \cref{lem:noise_bound_generic}, we have for some $\beta_2 >0$,
\begin{align*}
    \E{t}{\|\pi \pow t - \ell(w^\star)\|_2^2}
    &  = \mathbb{E}_{i_t} \Big[\|(n \ell_{i_t}(w \pow t) - n\ell_{i_t}(w^\star))e_{i_t}  \\
    & \hspace{30pt}
    + (n\ell_{i_t}(w^\star) - n \ell_{i_t}(z_{i_t}\pow t))e_{i_t} - (\ell(w^\star) - l \pow t)\|_2^2 \Big] \\
    & \leq - \beta_2 \|\ell(w\pow t) - \ell(w^\star)\|_2^2 \\
    & \quad + (1+\beta_2)\E{t}{\|(n \ell_{i_t}(w \pow t) - n\ell_{i_t}(w^\star))e_{i_t}\|_2^2} \\
    & \quad + (1+ \beta_2^{-1}) \E{t}{\|(n\ell_{i_t}(w^\star) - n \ell_{i_t}(z_{i_t}\pow t))e_{i_t}\|_2^2 } \\
    & \quad - (1+\beta_2^{-1}) \|\ell(w^\star) - l \pow t\|_2^2 \\
    & = (n +(n-1)\beta_2)\|\ell(w\pow t) - \ell(w^\star)\|_2^2 \\
    & \quad + (n-1)(1+ \beta_2^{-1})\|\ell(w^\star) - l \pow t\|_2^2 \\
    & \leq (n + (n-1)\beta_2) nG^2 \|w\pow t - w^\star\|_2^2\\
    & \quad + (n-1)(1+ \beta_2^{-1})\|\ell(w^\star) - l \pow t\|_2^2.
\end{align*}
\end{proof}

\paragraph{Incorporating smoothness and convexity of the losses}
The improved algorithm we developed here, for the purpose of a fair comparison to our own algorithm, differs from the original one from~\cite{palaniappan2016stochastic} by \cref{cor:super_descent} stemming from \cref{lem:mixed_coco}. We exploit the smoothness and convexity of the losses 
to get a negative term $-\E{t}{\|n q_{i_1} \nabla \ell_{i_t} (w\pow t) 
- n q_{i_t}^* \nabla \ell_{i_t}(w^\star)\|_2^2}$  used to temper the variance of the primal updates at the price of an additional positive term $\|q \pow t - q^\star\|_2^2$. The sum of both being positive we can dampen the effect of the additional positive term  $\|q \pow t - q^\star\|_2^2$ at the price of getting a less negative term $-\E{t}{\|n q_{i_1} \nabla \ell_{i_t} (w\pow t) 
- n q_{i_t}^* \nabla \ell_{i_t}(w^\star)\|_2^2}$. 
\begin{lemma}\label{lem:mixed_coco}
For any $q_1, q_2 \in \mathcal{P}(\sigma)$, $w_1, w_2 \in \R^d$, we have,
\begin{align*}
    & (q_1 - q_2)^\top(\ell(w_1) - \ell(w_2))
    - (\nabla (q_1^\top \ell)(w_1) - \nabla (q_2^\top \ell)(w_2))^\top (w_1 -w_2) \\
    & \leq -\frac{1}{2L n \sigma_{\max}} \left( 
    \E{i \sim \Unif[n]}{\| n q_{1, i} \nabla \ell_i(w_1) - nq_{2, i} \nabla \ell_i(w_2)\|_2^2 + \| n q_{1, i} \nabla \ell(w_2) - nq_{2, i} \nabla \ell(w_1)\|_2^2}
    \right) \\
    & \quad + \frac{G^2}{L \sigma_{\max}} \|q_1 -q_2\|_2^2.
\end{align*}
\end{lemma}
\begin{proof}
    For any $q \in \mathcal{P}(\sigma)$ and any $w, v \in \R^d$, we have by smoothness and convexity of $q_i \ell_i$, for $q_i >0 $
\begin{align}
    q_i \ell_i(v) &
    \geq q_i \ell_i(w) 
    + q_i \nabla \ell_i(w)^\top (v-w) 
    + \frac{1}{2L q_i} \|q_i \nabla \ell_i(w) - q_i \nabla \ell_i(v)\|_2^2 \\
    & \geq q_i \ell_i(w) 
    + q_i \nabla \ell_i(w)^\top (v-w) 
    + \frac{1}{2L n^2\sigma_{\max}} \|nq_i \nabla \ell_i(w) - nq_i \nabla \ell_i(v)\|_2^2.
\end{align}
Note that the second inequality is then true even if $q_i = 0$, since in that case all terms are 0. 
Therefore, for any $q_1, q_2 \in \mathcal{P}(\sigma)$, and any $w_1, w_2$, we have
\begin{align*}
    q_1^\top \ell(w_2) 
    & \geq q_1^\top \ell(w_1) 
    + \nabla (q_1^\top \ell)(w_1)^\top (w_2-w_1) 
    + \frac{1}{2Ln\sigma_{\max}}\E{i \sim \Unif[n]}{\| n q_{1, i} \nabla \ell_i(w_1) - nq_{1, i} \nabla \ell_i(w_2)\|_2^2}, \\
    q_2^\top \ell(w_1) 
    & \geq q_2^\top \ell(w_2) 
    + \nabla (q_2^\top \ell)(w_2)^\top (w_1-w_2) 
     + \frac{1}{2Ln\sigma_{\max}}\E{i \sim \Unif[n]}{\| n q_{2, i} \nabla \ell(w_1) - nq_{2, i} \nabla \ell(w_2)\|_2^2}.
\end{align*}
Combining these inequalities, we get 
\begin{align*}
    & -(q_1 - q_2)^\top(\ell(w_1) - \ell(w_2))
    + (\nabla (q_1^\top \ell)(w_1) - \nabla (q_2^\top \ell)(w_2))^\top (w_1 -w_2) \\
    & \geq \frac{1}{2L n \sigma_{\max}} \left(
    \E{i \sim \Unif[n]}{\| n q_{1, i} \nabla \ell_i(w_1) - nq_{1, i} \nabla \ell_i(w_2)\|_2^2 + \| n q_{2, i} \nabla \ell(w_1) - nq_{2, i} \nabla \ell(w_2)\|_2^2}
    \right).
\end{align*}
For any 4 vectors $a, b, c, d$, 
\[
\|a-b\|_2^2 + \|c-d\|_2^2 = \|a - c\|_2^2 + \|b-d\|_2^2 - 2(a-d)^\top(b-c).
\]
Applying this for 
$
a = q_{1, i} \nabla \ell_i(w_1), \ 
    b = q_{i, 1} \nabla \ell_i(w_2), \
    c = q_{2, i} \nabla \ell_i(w_2), \
    d = q_{2, i} \nabla \ell_i(w_1),
$
we get 
\begin{align*}
    & -(q_1 - q_2)^\top(\ell(w_1) - \ell(w_2))
    + (\nabla (q_1^\top \ell)(w_1) - \nabla (q_2^\top \ell)(w_2))^\top (w_1 -w_2) \\
    & \geq \frac{1}{2L n \sigma_{\max}} \Big(
    \E{i \sim \Unif[n]}{\| n q_{1, i} \nabla \ell_i(w_1) - nq_{2, i} \nabla \ell_i(w_2)\|_2^2 + \| n q_{1, i} \nabla \ell(w_2) - nq_{2, i} \nabla \ell(w_1)\|_2^2} \\
    & \quad  - 2n^2\E{i \sim \Unif[n]}{ (q_{1, i} - q_{2, i})^2 \nabla \ell_i(w_1)^\top \nabla \ell_i(w_2)}
    \Big).
\end{align*}
Reorganizing the terms and bounding $\nabla \ell_i(w_1)^\top \nabla \ell_i(w_2)$ by $G^2$ we get the result. 
\end{proof}

\begin{corollary}\label{cor:super_descent}
We have for any $\alpha  \in [0, 1]$
\begin{align*}
    & \E{t}{\frac{(1+ \eta \mu)^2}{\eta} \|w\pow {t+1} - w^\star\|_2^2 +
    \frac{(1+ \delta \bar \nu)^2}{\delta} \|q \pow {t+1} - q^\star\|_2^2} \\
    & \leq \eta^{-1} \|w \pow t - w^\star\|_2^2 
    + \left(\delta^{-1} +\frac{2 \alpha G^2}{L\sigma_{\max}}\right) \|q \pow t - q^\star\|_2^2 \\
    & \quad + \eta \E{t}{\|v\pow t -  \nabla ({q ^*}^\top \ell)(w^\star) \|_2^2} + \delta \E{t}{\|\pi \pow t -\ell(w^\star)\|_2^2}  \\
    & \quad 
    -\frac{\alpha}{Ln \sigma_{\max}} 
    \E{t}{\|n q_{i_1} \nabla \ell_{i_t} (w\pow t) 
    - n q_{i_t}^* \nabla \ell_{i_t}(w^\star)\|_2^2}.
\end{align*}
\end{corollary}
\begin{proof}
    Follows from \cref{lem:mixed_coco}
\end{proof}

\paragraph{Lyapunov function and overall convergence}
\cref{lem:lyap_dec} shows that an appropriately defined Lyapunov function incorporating the distances to the optima, decrease exponentially.

\begin{theorem}\label{lem:lyap_dec}
Define the Lyapunov function
\begin{align*}
    V\pow t & = \frac{(1+ \eta \mu)^2}{\eta} \|w\pow t - w^\star\|_2^2 + \frac{(1+ \delta\bar \nu)^2}{\delta} \|q \pow t - q^\star\|_2^2  \\ 
    & \quad + c_1 \sum_{i=1}^n \|n\rho_i\pow t\nabla \ell_i(z_i \pow t) - nq_i^* \nabla \ell_i (w^\star)\|_2^2 + \frac{c_2}{G^2}  \|l\pow t - \ell(w^\star)\|_2^2,
\end{align*}
with $c_1 = \frac{n}{2(L\kappa_\sigma + 2 G^2 n /\bar \nu)}$ and $c_2 = \frac{\mu}{2}$ with $\kappa_\sigma=n \sigma_{\max}$.
By taking 
\[
\eta  = \min\left\{\frac{1}{\mu}, \frac{1}{6(L\kappa_\sigma + 2G^2 n/\bar \nu)} \right\}, \quad 
\delta = \min\left\{\frac{1}{\bar \nu}, \frac{\mu}{8n^2 G^2}\right\},
\]
we have 
\[
\E{t}{V\pow {t+1}} \leq (1 - \tau^{-1}) V\pow t,
\]
for some $\tau > 1$. In particular, for small regularizations, i.e., $\mu\bar \nu \leq 8n^2 G^2$ and $\mu \leq  6(L\kappa_\sigma + 2G^2 n/\bar \nu)$, we have 
\[
\tau = \max\left\{2n, 4 + \frac{24 L\kappa_\sigma}{\mu} + \frac{48 G^2 n }{\mu \bar \nu}, 2 + \frac{16G^2n^2}{\bar \nu \mu}\right\}.
\]
\end{theorem}
\begin{proof}
Let us denote
\[
T \pow t = \frac{1}{n} \sum_{i=1}^n \|n\rho_i\pow t\nabla \ell_i(z_i \pow t) - nq_i^* \nabla \ell_i (w^\star)\|_2^2, \quad 
S \pow t = \|l\pow t - \ell(w^\star)\|_2^2, 
\]
we have,
\begin{align*}
    \E{t}{T \pow {t+1}}  & \leq  \frac{1}{n^2} \sum_{i=1}^n \|n q_i\pow t \nabla \ell_i(w\pow t) - n q_i^* \nabla \ell_i(w^\star)\|_2^2 
+  \left(1 - \frac{1}{n}\right)T\pow t, \\
\E{t}{S \pow {t+1}} & \leq G^2 \|w\pow t - w^\star\|_2^2 +  \left(1 - \frac{1}{n}\right)S\pow t.
\end{align*}
By combining \cref{cor:super_descent}, \cref{lem:noise_bound_generic}, \cref{lem:noise_bound_saddle} we have,  denoting $\kappa_\sigma = n\sigma_{\max}$,
\begin{align*}
    \E{t}{V\pow {t+1}} 
    & \leq \left(\eta^{-1} + \delta (n + (n-1)\beta_2) n G^2 + c_2 \right)
    \|w \pow t - w^\star\|_2^2  \\
    & \quad + \left(\delta^{-1} +\frac{2 \alpha n G^2}{L\kappa_\sigma}\right) 
    \|q \pow t - q^\star\|_2^2 \\
    & \quad + \left(\eta(1+\beta_1) + \frac{c_1}{n} - \frac{\alpha}{Ln \sigma_{\max}}\right)
    \E{i \sim \Unif[n]}{\|n q_{i_1} \nabla \ell_{i_t} (w\pow t) 
    - n q_{i_t}^* \nabla \ell_{i_t}(w^\star)\|_2^2} \\
    & \quad + \left(\eta (1+ \beta_1^{-1}) + c_1\left(1-\frac{1}{n}\right)\right)
    \frac{1}{n} \sum_{i=1}^n \|n\rho_i\pow t\nabla \ell_i(z_i \pow t) - nq_i^* \nabla \ell_i (w^\star)\|_2^2
    \\ 
    & \quad + \left(\delta(n-1)(1+\beta_2^{-1})+\frac{c_2}{G^2}\left(1-\frac{1}{n}\right)\right)
    \|\ell(w^\star) -l\pow t\|_2^2.
\end{align*}
Therefore for some $\tau > 1$, we have 
\begin{align*}
    \E{t}{V\pow {t+1}} - (1 - \tau^{-1}) V\pow t 
    & \leq K_1 \|w\pow t - w^\star\|_2^2 
    + K_2 \|q \pow t - q^\star\|_2^2 \\
    & \quad
    + K_3 \E{i \sim \Unif[n]}{\|n q_{i_1} \nabla \ell_{i_t} (w\pow t) 
    - n q_{i_t}^* \nabla \ell_{i_t}(w^\star)\|_2^2} \\
    & \quad 
    + K_4 \frac{1}{n} \sum_{i=1}^n \|n\rho_i\pow t\nabla \ell_i(z_i \pow t) - nq_i^* \nabla \ell_i (w^\star)\|_2^2 
    + K_5 \|\ell(w^\star) -l\pow t\|_2^2,
\end{align*}
with,
\begin{align*}
    K_1 & = \frac{(1+ \eta \mu)^2}{\eta}
    \left(\frac{1  + \eta\left((n + (n-1)\beta_2 )n G^2\delta + c_2\right) }{(1+\eta \mu)^2} -(1-\tau\inv)\right) \\
    K_2 & = \frac{(1+ \delta \bar \nu)^2}{\delta}
    \left( \frac{1 + 2\delta \alpha G^2n /(L \kappa_\sigma)}{(1+ \delta \bar \nu)^2}  - (1-\tau\inv)\right) \\
    K_3 & = \eta(1+ \beta_1) + \frac{c_1}{n} - \frac{\alpha}{L \kappa_\sigma} \\
    K_4 & = c_1\left(\eta(1+ \beta_1^{-1})\frac{1}{c_1} + \left(1- \frac{1}{n}\right) - (1-\tau\inv)\right) \\
    K_5 & =  \frac{c_2}{G^2} \left(\delta(n-1)(1+\beta_2^{-1})\frac{G^2}{c_2} + \left(1-\frac{1}{n}\right) - (1-\tau\inv)\right).
\end{align*}
Fix $\beta_1 = 2, \beta_2 = 1$. Denote also $\bar \eta = \frac{\eta \mu }{1 + \eta \mu} \in (0, 1)$ and $\bar \delta  = \frac{\delta \bar \nu}{1 +\delta \bar \nu} \in (0, 1)$ with e.g. $\eta = \frac{\bar \eta}{\mu (1+ \bar\eta)} $.
We have then for $c_1/n = \alpha/(2L\kappa_\sigma)$ and $c_2 = \mu/2$,
\begin{align*}
    K_1 & \leq  \eta \mu^2 \bar \eta \left( \bar \eta^2 - \left(1-  \frac{2n^2 G^2 \delta}{\mu}\right) \bar \eta  + \tau\inv \right) \\
    K_2 & \leq  \delta \bar \nu^2 \bar \delta \left( \bar \delta^2 - 2 \left(1 - \frac{\alpha G^2 n}{L \kappa_\sigma \bar \nu}\right) \bar \delta + \tau\inv \right) \\
    K_3 & = 3 \eta - \frac{\alpha}{ 2 L \kappa_\sigma} \\
    K_4 & = c_1\left(3\eta \frac{L\kappa_\sigma}{n \alpha}  - \frac{1}{n} + \tau\inv \right) \\
    K_5 & \leq  \frac{c_2}{G^2} \left(\delta \frac{4n G^2}{\mu}  - \frac{1}{n} + \tau\inv \right).
\end{align*}
We can further take $3\eta \leq \alpha/(2L\kappa_\sigma)$ and $\delta \leq \mu/(8n^2 G^2)$. 
By imposing the constraint  $\tau \ge {2n}$, we can simplify
\begin{align*}
    K_1 & \leq  \eta \mu^2 \bar \eta \left( \bar \eta^2 - \frac{3}{4} \bar \eta    + \tau\inv \right) \\
    K_2 & \leq  \delta \bar \nu^2 \bar \delta \left( \bar \delta^2 - 2 \left(1 - \frac{\alpha G^2 n}{L \kappa_\sigma \bar \nu}\right) \bar \delta  + \tau\inv \right) \\
    K_3 & \leq 0, K_4 \leq 0, K_5 \leq 0.
\end{align*}
Recall that $\alpha$ must be chosen in $[0, 1]$. Taking then 
\[
\alpha = \frac{L\kappa_\sigma}{L \kappa_\sigma  + 2G^2n /\bar \nu} \leq \frac{L\kappa_\sigma \bar \nu}{2G^2 n},
\]
we get 
\begin{align*}
    K_1 \leq  \eta \mu^2 \bar \eta \left( \bar \eta^2 - \frac{3}{4} \bar \eta    + \tau\inv \right), \quad
    K_2 \leq  \delta \bar \nu^2 \bar \delta \left( \bar \delta^2 - \bar \delta  + \tau\inv \right).
\end{align*}
By taking $\eta \leq 1/\mu$, $\delta \leq 1/\bar \nu$, we get $\bar \eta \leq 1/2$, $\bar \delta \leq 1/2$ and so $\bar \eta^2 -\frac{3}{4}\bar \eta \leq -\frac{1}{4}\bar \eta$ and $\bar \delta^2 - \bar \delta \leq -\frac{1}{2} \bar \delta$. Therefore taking 
\[
\eta  = \min\left\{\frac{1}{\mu}, \frac{1}{6(L\kappa_\sigma + 2G^2 n/\bar \nu)} \right\}, \quad 
\delta = \min\left\{\frac{1}{\bar \nu}, \frac{\mu}{8n^2 G^2}\right\},
\]
we get $K_i \leq 0$ for all $i$ as long as $\tau \ge \max\{2n, 4 / \bar \eta,  2/ \bar \delta \}$. In our case, 
\begin{align*}
    \frac{4}{\bar \eta} &  = \begin{cases}
4\left(1 + \frac{6 L\kappa_\sigma}{\mu} + \frac{12 G^2 n }{\mu \bar \nu}\right)
& \mbox{if} \ \mu \leq  6(L\kappa_\sigma + 2G^2 n/\bar \nu) ,\\
8
& \mbox{otherwise} ,
\end{cases} \\
\frac{2}{\bar \delta} &  = \begin{cases}
2 \left(1 + \frac{8G^2n^2}{\bar \nu \mu}\right)
& \mbox{if} \ \mu\bar \nu \leq 8n^2 G^2,  \\
4 & \mbox{otherwise} .
\end{cases}
\end{align*}
The result follows.
\end{proof}

\begin{corollary}\label{cor:rate}
    Under the setting of \cref{lem:lyap_dec}, after $t$ iterations of \saddlesaga, we have
    \begin{align*}
    & \E{}{\frac{(1 + \eta \mu)^2}{\eta} \|w\pow t - w^\star\|_2^2 
    + \frac{(1+ \delta \bar \nu)^2}{\delta} \|q \pow t - q^\star\|_2^2} \\
    & \leq \exp(- t / \tau) \Bigg(
    \frac{(1 + \eta \mu)^2}{\eta} \|w\pow 0 - w^\star\|_2^2 
    + \frac{(1+ \delta \bar \nu)^2}{\delta} \|q \pow 0 - q^\star\|_2^2  \\
    & \hspace{60pt} + c_1 n^2 \sum_{i=1}^n \| n q_i\pow 0 \nabla \ell_i(w\pow 0) -q_i^* \nabla \ell_i(w^\star)\|_2^2 
    + \frac{c_2}{G^2} \|\ell(w\pow 0) - \ell(w^\star)\|_2^2 \Bigg).
    \end{align*}
\end{corollary}

\clearpage

\section{Technical Results from Convex Analysis}
\label{sec:a:convex}
In this section, we collect several results, mostly from \citet{Nesterov2018Lectures}, that are used throughout the manuscript. In the following, let $\norm{\cdot}$ denote an arbitrary norm on $\R^d$ and let $\norm{\cdot}_*$ denote its associated dual norm.

The first concerns $L$-smooth function, or those with $L$-Lipschitz continuous gradient.
\begin{theorem}{\citep[Theorem 2.1.5]{Nesterov2018Lectures}}
    \label{thm:smooth_cvx}
    The conditions below are considered for any $x, y \in \R^d$ and $\alpha \in [0, 1]$. The following are equivalent for a differentiable function $f: \R^d \rightarrow \R$.
    \begin{enumerate}
        \item $f$ is convex and $L$-smooth with respect to $\norm{\cdot}$.
        \item $0\leq f(y) - f(x) - \ip{\grad f(x), y - x} \leq \frac{L}{2}\normsq{x-y}$.
        \item $f(x) + \ip{\grad f(x), y-x} + \frac{1}{2L} \normsq{\grad f(x) - \grad f(y)}_* \leq f(y)$.
        \item $\frac{1}{L}\normsq{\grad f(x) - \grad f(y)}_* \leq \ip{\grad f(x) - \grad f(y), x - y}$.
        \item $0 \leq \ip{\grad f(x) - \grad f(y), x - y} \leq L\normsq{x-y}$.
    \end{enumerate}
\end{theorem}

Next, we detail the properties of strongly convex functions.
\begin{theorem}{\citep[Theorem 2.1.10]{Nesterov2018Lectures}}
    If $f: \R^d \rightarrow \R$ is $\mu$-strongly convex and differentiable, then for any $x, y \in \R^d$,
    \begin{itemize}
        \item $f(y) \leq f(x) + \ip{f(x), y - x} + \frac{1}{2\mu} \normsq{\grad f(x) - \grad f(y)}_*$.
        \item $\ip{\grad f(x) - \grad f(y), x - y} \leq \frac{1}{\mu} \normsq{\grad f(x) - \grad f(y)}_*$.
        \item $\mu \norm{x - y} \leq \norm{\grad f(x) - \grad f(y)}_*$.
    \end{itemize}
\end{theorem}

Finally, functions that are both smooth and strongly convex enjoy a number of relevant primal-dual properties.
\begin{theorem}{\citep[Theorem 2.1.12]{Nesterov2018Lectures}}
    \label{thm:smooth_strongly_cvx_inequality}
    If $f$ is both $L$-smooth and $\mu$-strongly convex, then for any $x, y \in \R^d$,
    \begin{align}
         -\ip{\grad f(x), x - y} &= -\tfrac{\mu L}{\mu + L} \norm{x - y}^2 - \tfrac{1}{\mu + L} \norm{\grad f(x) - \grad f(y)}^2 - \ip{\grad f(y), x - y}.  \label{eqn:nesterov_descent}
    \end{align}
\end{theorem}

\begin{lemma}\label{lem:smooth_strg_cvx} 
Let $f: \R^d \rightarrow \R$
be 
    $\mu$-strongly convex and $M$-smooth. Then, we have for any $w, v \in \R^d$,
    \begin{align*}
        f(v) & \geq f(w) 
        + \nabla f(w)^\top (v-w) 
        + \frac{1}{2(M+\mu)}\|\nabla f(w) - \nabla f(v)\|_2^2
        + \frac{\mu}{4}\|w-v\|_2^2.
    \end{align*}
\end{lemma}
\begin{proof}
    The function $g = f - \mu \|\cdot\|_2^2/2$ is convex and $M-\mu$ smooth.
    Hence, we have by line 3 of \Cref{thm:smooth_cvx} for any $w, v \in \R^d$,
    \[
        g(v) \geq g(w) 
        + \nabla g(w)^\top (v - w) 
        + \frac{1}{2(M-\mu)}\|\grad g(v)-\grad g(w)\|_2^2.
    \]
    Expanding $g$ and $\nabla g$, we get 
    \begin{align*}
        f(v) & \geq f(w) 
        + \nabla f(w)^\top (v-w) 
        + \frac{1}{2(M-\mu)}\|\nabla f(w) - \nabla f(v)\|_2^2 \\
        & \quad + \frac{\mu M}{2(M-\mu)}\|w-v\|_2^2 
        - \frac{\mu}{M-\mu}(\nabla f(w) - \nabla f(v))^\top(w-v).
    \end{align*}
    Using Young's inequality, that is, $a^\top b \leq \frac{\alpha}{2} \|a\|_2^2
    + \frac{\alpha^{-1}}{2}\|b\|_2^2$, we have 
    \begin{align*}
        f(v) & \geq f(w) 
        + \nabla f(w)^\top (v-w) 
        + \frac{1 - \alpha \mu}{2(M-\mu)}\|\nabla f(w) - \nabla f(v)\|_2^2 \\
        & \quad + \frac{\mu (M-\alpha^{-1})}{2(M-\mu)}\|w-v\|_2^2.
    \end{align*}
    Taking $\alpha = \frac{2}{\mu + M}$ gives the claim. 
\end{proof}
We state a few properties of the prox operator.
\begin{theorem}[Co-coercivity of the prox]
\label{thm:prox:cocoercive}
    If $f : \reals^d \to \reals$ is $\mu$-strongly convex, then we have
    for any constant $\eta > 0$ that
    \[
    \inp{x - y}{\prox_{\eta f}(x) - \prox_{\eta f}(y)}
    \ge (1 + \eta \mu) \norm{\prox_{\eta f}(x) - \prox_{\eta f}(y)}^2 \,.
    \]
\end{theorem}

The same result applied to the convex conjugate $f^\star$ of $f$ and noting that 
$\grad \Mcal_\eta[f](x) = \prox_{f^\star / \eta}(x / \eta)$ 
gives the following result:
\begin{theorem}[Co-coercivity of the prox]
\label{thm:prox-grad:cocoercive}
    If $f : \reals^d \to \reals$ is $L$-smooth, then we have
    for any constant $\eta > 0$ that
    \[
    \inp{x - y}{\grad \Mcal_\eta[f](x) - \grad \Mcal_\eta[f](y)}
    \ge \eta \left(1 + \frac{1}{L\eta}\right) \norm{\grad \Mcal_\eta[f](x) - \grad \Mcal_\eta[f](y)}^2 \,.
    \]
\end{theorem}

\begin{lemma}[{\citep[Lemma 4]{blondel2020fast}}]\label{lem:cvx_ordering}
For a convex function $f: \R \rightarrow\R$, if $x_1 \geq  x_2$ and $y_2 \geq y_1$, then 
\[
f(y_1 - x_1) + f(y_2 - x_2) \geq f(y_2 - x_1) + f(y_1 - x_2).
\]
\end{lemma}

\begin{lemma}\label{lem:smooth_weights} Define for $l \in \R^n$,
    \[
        h(l) = \max_{q \in \mathcal{P}(\sigma)} l^\top q - \frac{\bar \nu}{2}\|q - \ones_n/n\|_2^2.
    \]
    The function $h$ is $1/\bar \nu$-smooth
    and convex such that for any $l, l' \in \R^n$,
    \[
        \bar \nu \|\nabla h(l)-\nabla h(l')\|_2^2  
        \leq (\nabla h(l) - \nabla h(l'))^\top (l -l') 
        \leq \frac{1}{\bar \nu}\|l - l'\|_2^2.
    \]
\end{lemma}

\clearpage

\section{Experimental Details}\label{sec:a:experiments}
\begin{table}[t]
\renewcommand{\arraystretch}{1.3}
    \centering
    \begin{tabular}{cccccc}
    \toprule
        {\bf Dataset} & $d$ & $n_{\text{train}}$ & $n_{\text{test}}$ & {\bf Task} & {\bf Source}\\
        \hline
        \texttt{yacht} & 6 & 244 & 62 & Regression & UCI\\
        \texttt{energy} & 8 & 614 & 154 & Regression& UCI\\
        \texttt{concrete} & 8 & 824 & 206 & Regression& UCI\\
        \texttt{kin8nm} & 8 & 6,553 & 1,639 & Regression& OpenML\\
        \texttt{power} & 4 & 7,654 & 1,914 & Regression& UCI\\
        \hline
        \texttt{diabetes} & 33 & 4,000 & 1,000 & Binary Classification & Fairlearn\\
        \texttt{acsincome} & 202 & 4,000 & 1,000 & Regression & Fairlearn\\
        \hline
        \texttt{amazon} & 535 & 10,000 & 10,000 & Multiclass Classification & WILDS\\
        \texttt{iwildcam} & 9420 & 20,000 & 5,000 & Multiclass Classification & WILDS\\
    \bottomrule
    \end{tabular}
    \vspace{1em}
    \caption{Dataset attributes and dimensionality $d$, train sample size $n_{\text{train}}$, and test sample size $n_{\text{test}}$.}
    \label{tab:dataset}
\end{table}

\subsection{Tasks \& Objectives}
\label{sec:a:task}

In all settings, we consider supervised learning tasks specified by losses of the form
\begin{align*}
    \ell_i(w) = h(y_i, w^\top \varphi(x_i)),
\end{align*}
where we consider an input $x_i \in \msc{X}$, a feature map $\varphi: \msc{X} \to \reals^d$, and a label $y_i \in \msc{Y}$. The function $h: \msc{Y} \times \R \rightarrow \R$ measures the error between the true label and another value which is the prediction in regression and the logit probabilities of the associated classes in classification. In the regression tasks, $\msc{Y} = \R$ and we used the squared loss
\[
    \ell_i(w) = \frac{1}{2} (y_i - w\T \phi(x_i))^2 \,.
\]
For binary classification, we have $\msc{Y} = \{-1, 1\}$, denoting a negative and positive class. We used the binary logistic loss
\[
    \ell_i(w) = - y_i x_i^\top w + \ln(1+ e^{x_i^\top w}) \,.
\]
For multiclass classification, $\msc{Y} = \{1, \ldots, C\}$ where $C$ is the number of classes. We used the multinomial logistic loss:
\[
    \ell_i(w) = -\ln p_{y_i}(x_i; w), \text{ where } p_{y_i}(x_i; w) := \frac{\exp\p{w_{\cdot y}^\top x_i}}{\sum_{y' = 1}^C \exp\p{w_{\cdot y'}^\top x_i}}, \ w \in \R^{d \times C}
\]
The design matrix $(\varphi(x_1), \ldots, \varphi(x_n)) \in \R^{n \times d}$ is standardized to have columns with zero mean and unit variance, and the estimated mean and variance from the training set is used to standardize the test sets as well.
Our final objectives are of the form 
\begin{align*}
    \primobj(w) = \max_{q \in \msc{P}(\sigma)} \sum_{i=1}^n q_i \ell_i(w) - \nu n\norm{q - \ones_n/ n}_2^2 + \frac{\mu}{2} \norm{w}_2^2
\end{align*}
for shift cost $\nu \geq 0$ and regularization constant $\mu \geq 0$.

\subsection{Datasets}
\label{sec:a:datasets}

We detail the datasets used in the experiments. If not specified below, the input space $\msc{X} = \R^d$ and $\varphi$ is the identity map. The sample sizes, dimensions, and source of the datasets are summarized in \Cref{tab:dataset}, where $d$ refers to the dimension of each $\varphi(x_i)$.
\begin{enumerate}[nosep, label=(\alph*), leftmargin=\widthof{ (a) }]
    \item \yacht:
prediction of the residuary resistance of a sailing yacht based on its physical attributes \cite{Tsanas2012AccurateQE}.
    \item \energy:
prediction of the cooling load of a building based on its physical attributes \cite{Segota2020Artificial}.
    \item \concrete:
prediction of the compressive strength of a concrete type based on its physical and chemical attributes \cite{Yeh2006Analysis}. 
    \item \kinnm:
prediction of the distance of an 8 link all-revolute robot arm to a spatial endpoint \citep{Akujuobi2017Delve}. 
    \item \power:
prediction of net hourly electrical energy output of a power plant given environmental factors \citep{Tufekci2014Prediction}.
    \item \diabetes:
prediction of readmission for diabetes patients based on 10 years worth of clinical care data at 130 US hospitals \citep{Rizvi2014Impact}.
    \item \acsincome:
prediction of income of US adults given features compiled from the American Community Survey (ACS) Public Use Microdata Sample (PUMS) \citep{Ding2021Retiring}.
    \item \amazon:
prediction of the review score of a sentence taken from Amazon products. Each input $x \in \msc{X}$ is a sentence in natural language and the feature map $\varphi(x) \in \reals^d$ is generated by the following steps:
\begin{itemize}
    \item A BERT neural network \cite{Devlin2019BERTPO} (fine-tuned on $10,000$ held-out examples) is applied to the text $x_i$, resulting in vector $x'_i$.
    \item The $x'_1, \ldots, x'_n$ are normalized to have unit norm.
    \item Principle Components Analysis (PCA) is applied, resulting in $105$ components that explain $99\%$ of the variance, resulting in vectors $x''_i \in \R^{105}$. The $d$ in \Cref{tab:dataset} refers to the total dimension of the parameter vectors for all $5$ classes.
\end{itemize}
    \item \iwildcam:
prediction of an animal or flora in an image from wilderness camera traps, with heterogeneity in illumination, camera angle, background, vegetation, color, and relative animal frequencies \cite{beery2020iwildcam}. Each input $x \in \msc{X}$ is an image the feature map $\varphi(x) \in \reals^d$ is generated by the following steps:
\begin{itemize}
    \item A ResNet50 neural network \cite{He2016DeepResidual} that is pretrained on ImageNet \cite{deng2009imagenet} is applied to the image $x_i$, resulting in vector $x'_i$.
    \item The $x'_1, \ldots, x'_n$ are normalized to have unit norm.
    \item Principle Components Analysis (PCA) is applied, resulting in $d = 157$ components that explain $99\%$ of the variance. The $d$ in \Cref{tab:dataset} refers to the total dimension of the parameter vectors for all $60$ classes.
\end{itemize}

\end{enumerate}

\subsection{Hyperparameter Selection}
\label{sec:a:hyperparam}

We fix a minibatch size of $64$ SGD and SRDA and an epoch length of $N = n$ for LSVRG. For SaddleSAGA we consider three schemes for selecting the primal and dual learning rates that reduce to searching for a single parameter $\eta > 0$, as described in \Cref{sec:a:additional}. In practice, the regularization parameter $\mu$ and shift cost $\nu$ are tuned by a statistical metric, i.e. generalization error as measured on a validation set. We study the optimization performance of the methods for multiple values of each in \Cref{sec:a:additional}.

For the tuned hyperparameters, we use the following method. Let $k \in \{1, \ldots, K\}$ be a seed that determines algorithmic randomness. This corresponds to sampling a minibatch without replacement for SGD and SRDA and a single sampled index for SaddleSAGA, LSVRG, and \lsaga. Letting $\mc{L}_k(\eta)$ denote the average value of the training loss of the last ten passes using learning rate $\eta$ and seed $k$, the quantity
$   \mc{L}(\eta) = \frac{1}{K} \sum_{k=1}^K \mc{L}_k(\eta)$ was minimized to select $\eta$.
The learning rate $\eta$ is chosen in the set $\{1\times 10^{-4}, 3\times 10^{-4}, 1\times 10^{-3}, 3\times 10^{-3}, 1\times 10^{-2}, 3\times 10^{-2}, 1\times 10^{-1}, 3\times 10^{-1}, 1\times 10^{0}, 3\times 10^{0}\}$, with two orders of magnitude lower numbers used in \acsincome due to its sparsity. We discard any learning rates that cause the optimizer to diverge for any seed.

\subsection{Compute Environment}
\label{sec:a:code}

No GPUs were used in the study; Experiments were run on a CPU workstation with an Intel i9 processor, a clock speed of 2.80GHz, 32 virtual cores, and 126G of memory. The code used in this project was written in Python 3 using the PyTorch and Numba packages for automatic differentiation and just-in-time compilation, respectively.

\clearpage

\section{Additional Experiments}\label{sec:a:additional}
\myparagraph{Varying Risk Parameters} We study the effect of varying the risk parameters, that is $(p, b, \gamma)$ for the $p$-superquantile, $b$-extremile, $\gamma$-ESRM, choosing spectral to increase the condition number $\kappa_\sigma = n\sigma_n$ compared to the experiments in the main text. We use $p = 0.25$, $b = 2.5$, and $\gamma = 1/e^{-2}$ to generate ``hard'' version of the superquantile, extremile, and ESRM. \Cref{fig:hard} plots the corresponding training curves for four datasets of varying sample sizes: \yacht, \energy, \concrete, and \iwildcam. We see that the comparison of methods is the same as the original methods, that is that \lsaga performs the best or close to best in terms of optimization trajectories. Except on \concrete, SaddleSAGA generally matches the performance of \lsaga. The trajectory of LSVRG is noticeably noisier than on the original settings; we hypothesize that the bias accrued by this epoch-based algorithm is exacerbated by the skewness in the spectrum, as mentioned in \citet[Proposition 1]{Mehta2022Stochastic}.

\begin{figure}[t]
    \centering
    \includegraphics[width=0.8\linewidth]{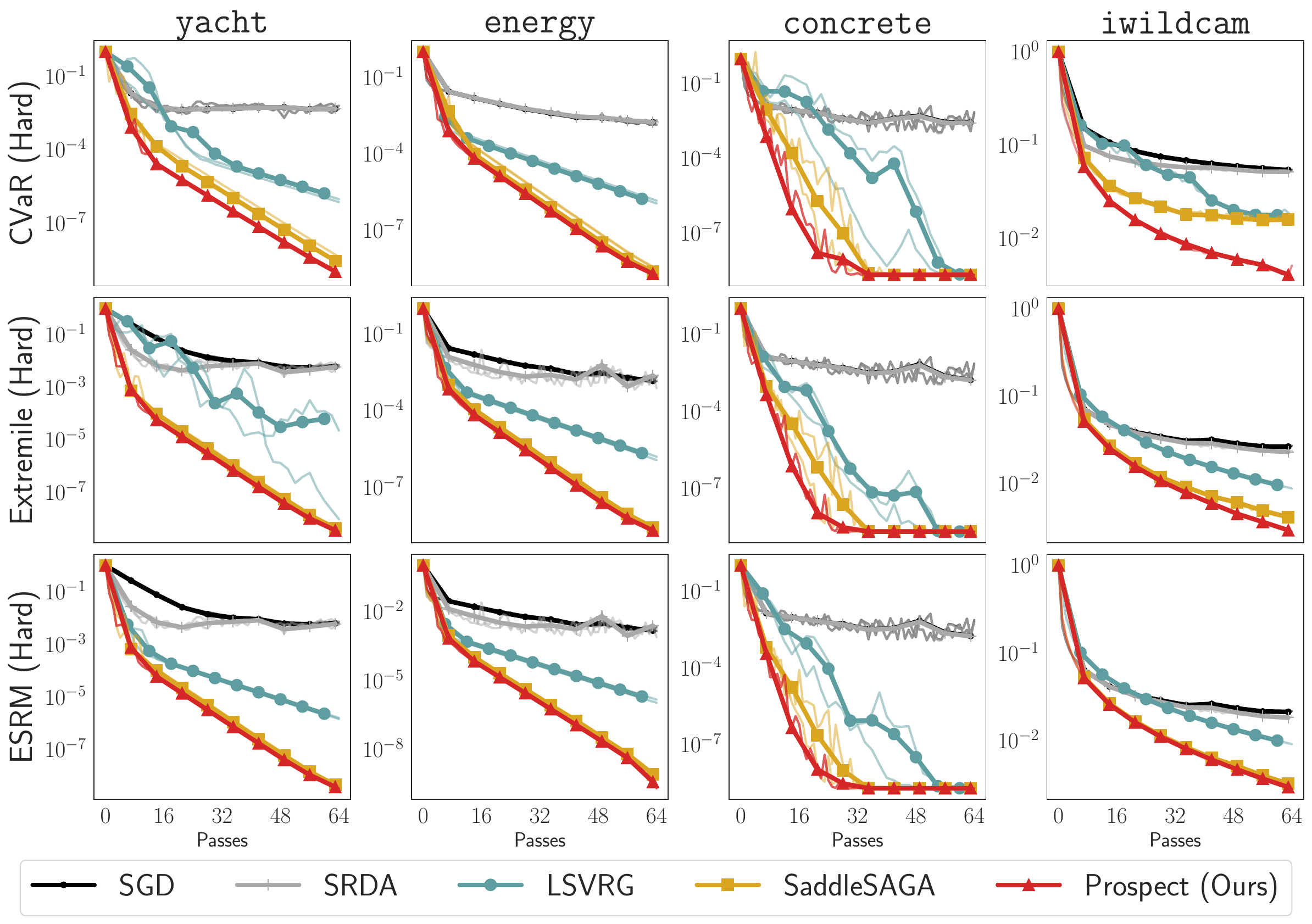}
    \caption{{\bf Harder risk parameter settings.} Each row represents a different ``hard'' variant of the superquantile, extremile, and ESRM spectra. Columns represent different datasets. Suboptimality \eqref{eqn:subopt} is measured on the $y$-axis while the $x$-axis measures the total number of gradient evaluations made divided by $n$, i.e. the number of passes through the training set.}
    \label{fig:hard}
\end{figure}

\myparagraph{Lowering or Removing Shift Cost}
A relevant setting is the low or no shift cost regime, as this allows the adversary to make arbitrary distribution shifts (while still constrained to $\Pcal(\sigma)$). These settings correspond to $\nu = 10^{-3}$ and $\nu = 0$, respectively. The low-cost experiment is displayed in \Cref{fig:low_smoothing} while \Cref{fig:no_smoothing} displays these curves for the no-cost experiment. When $\nu = 0$, the optimization problem can equivalently be written as
\begin{align*}
    \min_{w \in \R^d} \sbr{\max_{q \in \Pcal(\sigma)} q^\top \ell(w) + \frac{\mu}{2}\norm{w}_2^2 = \sum_{i=1}^n \sigma_i \ell_{(i)}(w) + \frac{\mu}{2}\norm{w}_2^2}.
\end{align*}
In this case, we always have that $\q(l) = (\sigma_{\pi^{-1}(1)}, \ldots, \sigma_{\pi^{-1}(n)})$, where $\pi$ sorts $l$.
Here, $w$ is chosen to optimize a linear combination of order statistics of the losses. In the low shift cost settings, performance trends are qualitatively similar to those seen from $\nu = 1$. Interestingly, for the no-cost setting, LSVRG, SaddleSAGA, and \lsaga seem to converge linearly empirically even without smoothness of the objective. 

\begin{figure}[t]
    \centering
    \includegraphics[width=\linewidth]{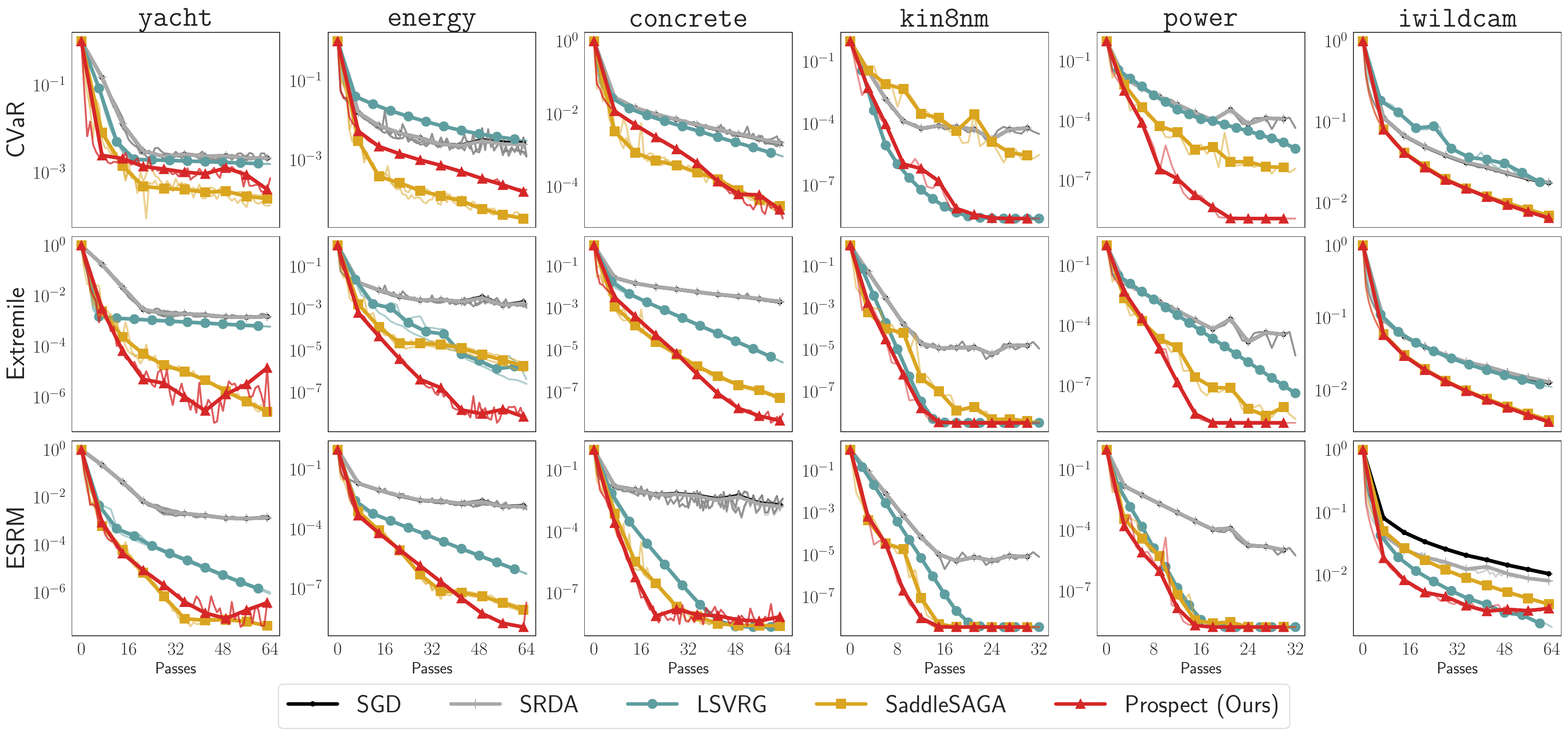}
    \caption{{\bf Low shift cost settings.} Each row represents a different spectral risk objective with $\nu = 10^{-3}$ (instead of $\nu=1$) while each column represents a different datasets. Suboptimality \eqref{eqn:subopt} is measured on the $y$-axis while the $x$-axis measures the total number of gradient evaluations made divided by $n$, i.e. the number of passes through the training set.}
    \label{fig:low_smoothing}
\end{figure}
\begin{figure}[t]
    \centering
    \includegraphics[width=\linewidth]{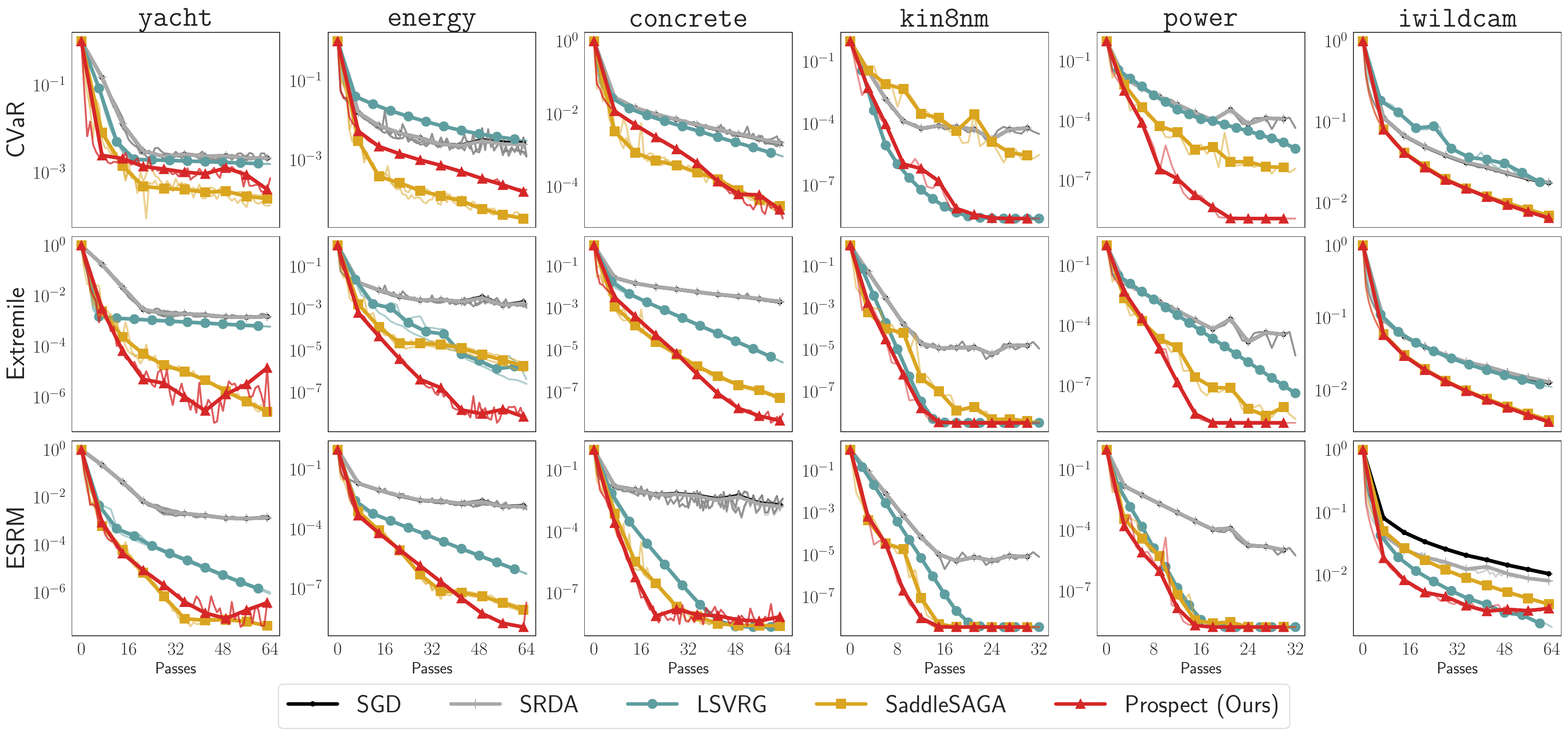}
    \caption{{\bf No shift cost settings.} Each row represents a different spectral risk objective with $\nu = 0$ (instead of $\nu=1$) while each column represents a different datasets. Suboptimality \eqref{eqn:subopt} is measured on the $y$-axis while the $x$-axis measures the total number of gradient evaluations made divided by $n$, i.e. the number of passes through the training set.}
    \label{fig:no_smoothing}
\end{figure}

\myparagraph{Lowering Regularization}
Next, we decrease the $\ell_2$-regularization from $\mu = 1/n$ to $\mu = 1/(10n)$ and $\mu =1/(100n)$. These settings are plotted in \Cref{fig:low_reg} and \Cref{fig:vlow_reg}, respectively. Performance rankings among methods reflect those of the original parameters. For five of the six datasets, that is \yacht, \energy, \concrete, \kinnm, and \power, the regression tasks involve optimizing the squared error. This function is already strongly convex, with constant depending on the smallest eigenvalue of the empirical second moment matrix. When assuming that the input data vectors are bounded, this function is also $G$-Lipschitz. Thus, if the problem is already well-conditioned, we may observe similar behavior even at negligible regularization ($\mu = 5\cdot 10^{-7}$ for \iwildcam, for example).

\begin{figure}[t]
    \centering
    \includegraphics[width=\linewidth]{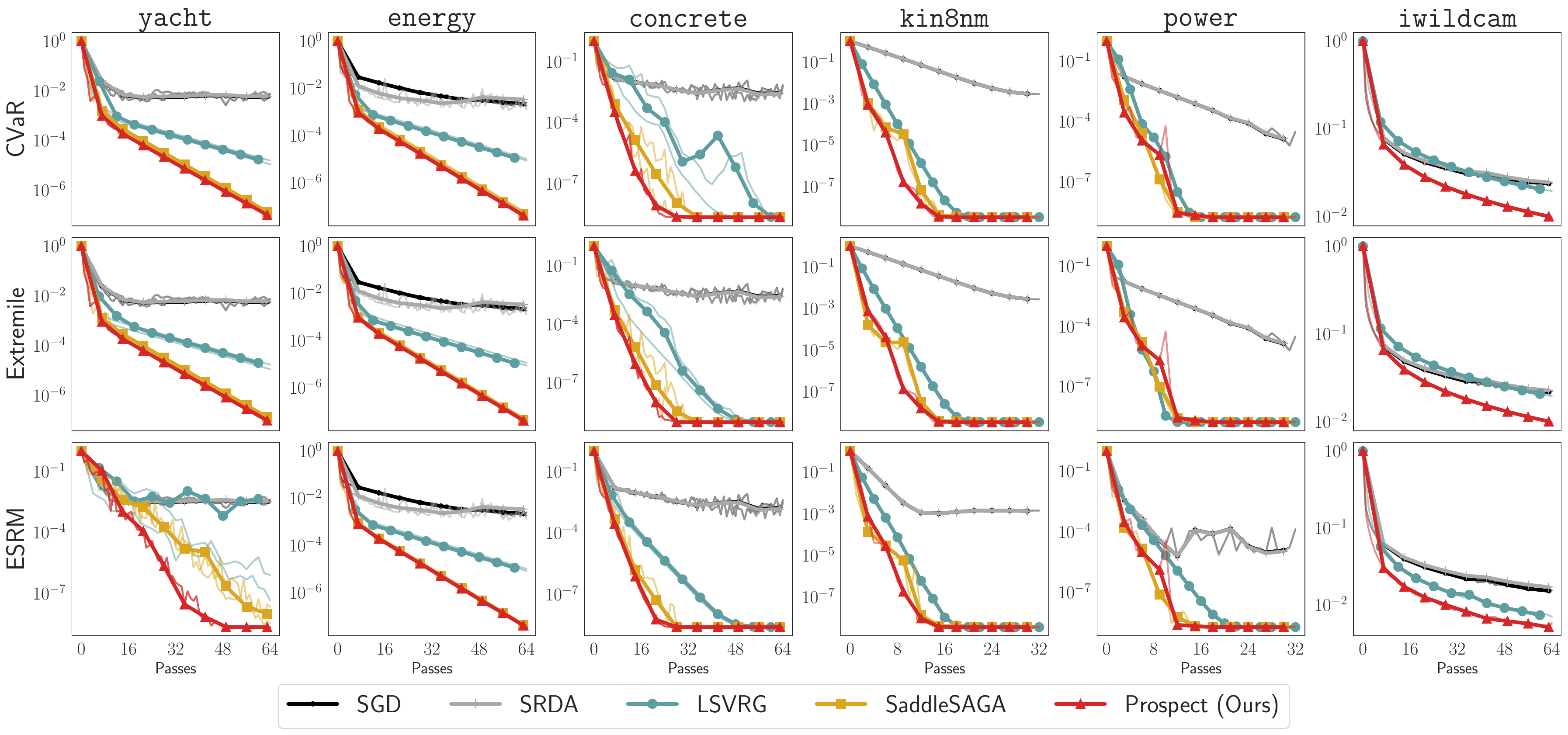}
    \caption{{\bf Reduced $\ell_2$-regularization settings ($\mu = 1/(10n)$.} Each row represents a different spectral risk objective with $\mu = 1/(10n)$ (instead of $\mu=1/n$) while each column represents a different dataset. Suboptimality \eqref{eqn:subopt} is measured on the $y$-axis while the $x$-axis measures the total number of gradient evaluations made divided by $n$, i.e. the number of passes through the training set.}
    \label{fig:low_reg}
\end{figure}

\begin{figure}[t]
    \centering
    \includegraphics[width=\linewidth]{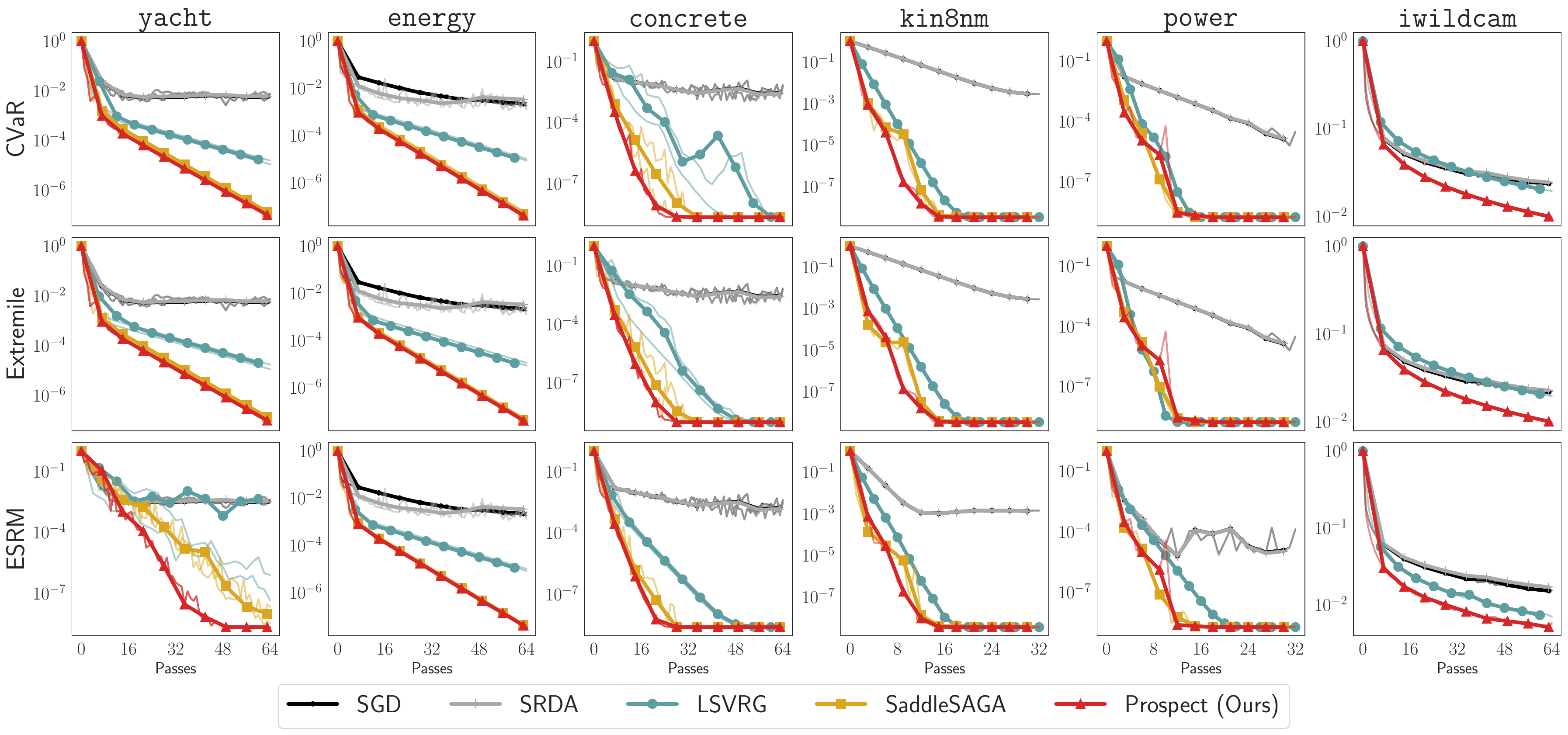}
    \caption{{\bf Low $\ell_2$-regularization settings ($\mu = 1/(100n)$.} Each row represents a different spectral risk objective with $\mu = 1/(100n)$ (instead of $\mu=1/n$) while each column represents a different dataset. Suboptimality \eqref{eqn:subopt} is measured on the $y$-axis while the $x$-axis measures the total number of gradient evaluations made divided by $n$, i.e. the number of passes through the training set.}
    \label{fig:vlow_reg}
\end{figure}

\myparagraph{Comparison of Saddle-Point and Moreau Variants}
Finally, observe in \Cref{fig:saddle} the comparison of SaddleSAGA variants (\Cref{sec:a:saddle_saga}), as well as the Moreau version of \lsaga using Moreau envelope-based oracles (\Cref{sec:a:prox_saga}). There are variants shown.
\begin{itemize}
    \item {\bf Primal LR = Dual LR:} The original variant of \citet{palaniappan2016stochastic}, in which the primal and dual learning rates are set to be equal and searched as a single hyperparameter.
    \item {\bf Search Dual LR:} Here, the primal learning rate is fixed as the optimal one for \lsaga, and the dual learning rate is searched as a single hyperparameter.
    \item {\bf Primal-Dual Heuristic:} In this version, used as the ``SaddleSAGA'' baseline in the main text, the dual learning rate is set to be $10n$ times smaller than the primal learning rate. 
    \item {\bf \lsaga-Moreau:} The Moreau-envelope version of \lsaga using proximal oracles.
\end{itemize}
We find that all methods besides the original variant (primal LR = dual LR) perform comparably on \yacht, \energy, \concrete, \kinnm, and \power. Notably, the ProxSAGA method performs similarly to \lsaga and the saddle point-based baselines. While using the Moreau envelope results in accelerated rates in the ERM setting~\cite{defazio2016simple}, we find  that the convergence rate is the same empirically. This phenomenon is in agreement with \Cref{thm:acc:lsaga}, which states that ProxSAGA will achieve the same linear convergence rate as \lsaga, but will require a much less stringent condition on the shift cost $\nu$ than in the case of \lsaga.

\begin{figure}[t]
    \centering
    \includegraphics[width=\linewidth]{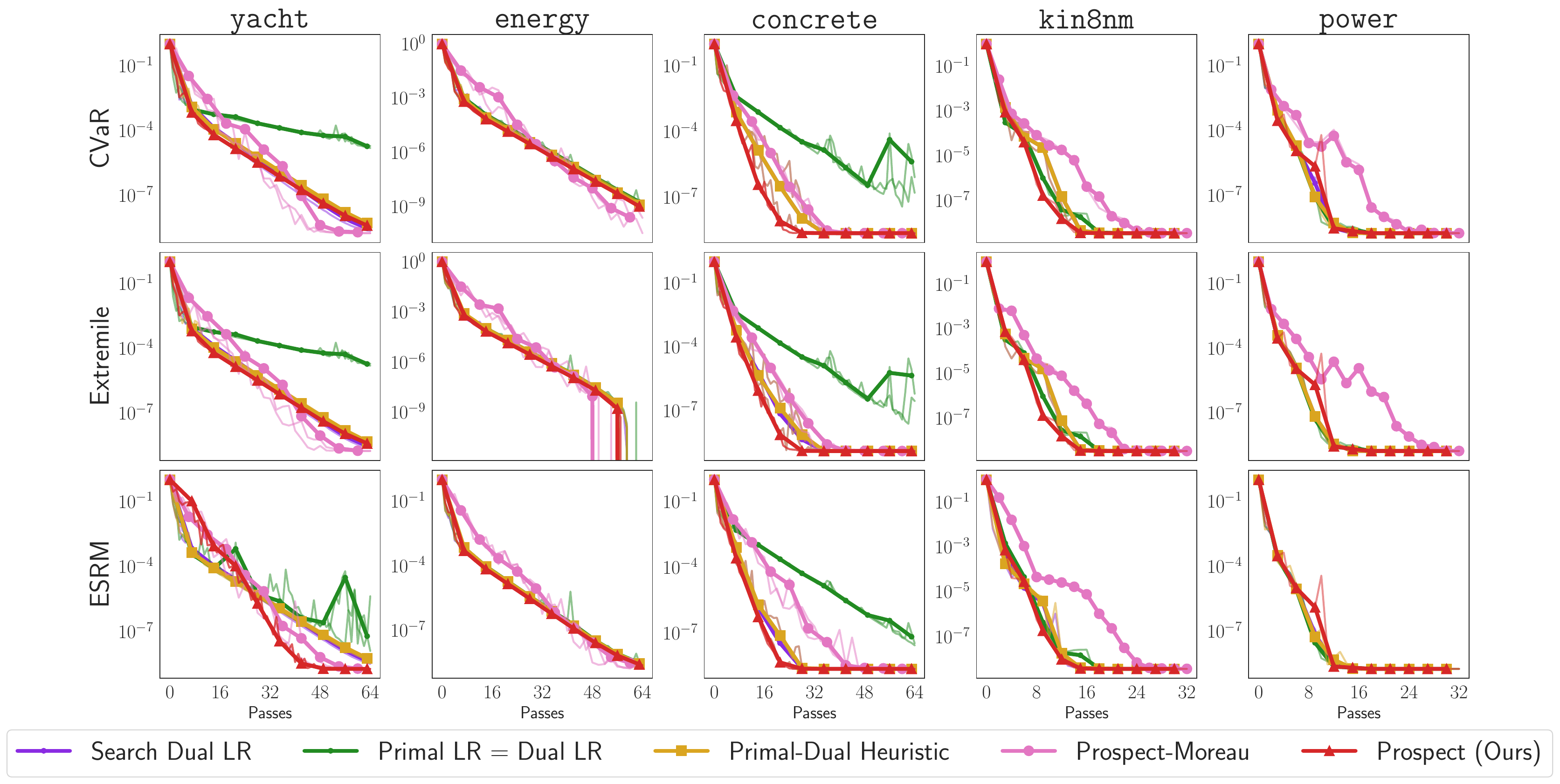}
    \caption{{\bf SaddleSAGA and \lsaga-Moreau method comparisons.} Each row represents a different spectral risk objective while each column represents a different dataset. Suboptimality \eqref{eqn:subopt} is measured on the $y$-axis while the $x$-axis measures the total number of gradient evaluations made divided by $n$, i.e. the number of passes through the training set.}
    \label{fig:saddle}
\end{figure}

\end{document}